\documentclass[twoside,11pt]{article}

\usepackage{enumerate}
\usepackage{jmlr2e}
\usepackage{bm}
\usepackage{relsize}
\usepackage{multirow}
\usepackage{mathtools}
\usepackage{xcolor}
\usepackage{amsfonts,amssymb} 

\usepackage[bbgreekl]{mathbbol}
\usepackage{amsfonts}

\usepackage{amsmath}
\newcommand\code[1]{\texttt{#1}}

\def\Corr{\mathrm{Corr}}
\def\card{\mathrm{card}}

\newcommand{\K}{Q}
\usepackage{amsmath}
\newcommand{\vertiii}[1]{{\left\vert\kern-0.25ex\left\vert\kern-0.25ex\left\vert #1 
    \right\vert\kern-0.25ex\right\vert\kern-0.25ex\right\vert}}
\usepackage{bbm}

\usepackage{algorithm}
\newcommand\mycolon{{\hspace{0.6mm}:\hspace{0.6mm}}}
\newcommand\asp{\hspace{4mm}}
\def\le{\widehat{l}_e}
\def\lp{\widehat{l}_p}
\def\lb{\widehat{l}_b}
\newcommand\lik{\ell}
\newcommand\altUnit{v'}
\renewcommand\time{n}

\newcommand\given{{\, | \,}}
\newcommand\data[1]{#1}
\newcommand\normal{\mathrm{Normal}}
\newcommand\mystretch{\rule[-2mm]{0mm}{5mm} }

\newtheorem{assumption}[theorem]{Assumption}
\newtheorem{thm}{Theorem}

\newcommand*{\ooverline}[1]{\overline{\overline{#1}}}
\usepackage{scalerel}[2016-12-29]

\def\scaleint#1{\vcenter{\hbox{\scaleto[3ex]{\displaystyle\int}{#1}}}}
\allowdisplaybreaks

\jmlrheading{1}{2023}{1-48}{4/00}{10/00}{meila00a}{Ning Ning and Edward L. Ionides}

\ShortHeadings{IBPF for High-dimensional Parameter Learning}{Ning and Ionides}
\firstpageno{1}

\begin{document}

\title{Iterated Block Particle Filter for High-dimensional Parameter Learning: Beating the Curse of Dimensionality}

\author{\name Ning Ning \email patning@tamu.edu \\
       \addr Department of Statistics\\
       Texas A\&M University\\
       College Station, TX 77843, USA
        \AND
       \name Edward L.\ Ionides \email ionides@umich.edu \\
       \addr Department of Statistics\\
       University of Michigan\\
       Ann Arbor, MI 48109, USA}
       
\editor{}

\maketitle

\begin{abstract}
Parameter learning for high-dimensional, partially observed, and nonlinear stochastic processes is a methodological challenge.
Spatiotemporal disease transmission systems provide examples of such processes giving rise to open inference problems.
We propose the iterated block particle filter (IBPF) algorithm for learning high-dimensional parameters over graphical state space models with general state spaces, measures, transition densities and graph structure.
Theoretical performance guarantees are obtained on beating the curse of dimensionality (COD), algorithm convergence, and likelihood maximization.
Experiments on a highly nonlinear and non-Gaussian spatiotemporal model for measles transmission reveal that the iterated ensemble Kalman filter algorithm \citep{li2020substantial} is ineffective and the iterated filtering algorithm \citep{ionides2015inference} suffers from the COD, while our IBPF algorithm beats COD consistently across various experiments with different metrics. 
\end{abstract}



\section{Introduction}
\label{sec:Introduction}
We firstly give the background and motivation in Section \ref{sec:Background_and_motivation} and then state our contributions in Section \ref{sec:Our_contributions}, followed by the organization of the paper in Section \ref{sec:organization}.

\subsection{Background and motivation}
\label{sec:Background_and_motivation}

Spatiotemporal data arises when measurements are made through time at a collection of spatial locations.
Spatiotemporal inference for epidemiological and ecological systems is arguably the last remaining open problem from the six challenges in time series analysis of nonlinear systems posed by \cite{bjornstad2001noisy}. 
A disease transmission system is stochastic and imperfectly observable, thus it is commonly modeled by a partially observed Markov process (POMP), otherwise known as state space model or hidden Markov model, which consists of a latent Markov process representing the time evolution of the system and a measurement process by which stochastic observations of this latent process are collected at specified time points. 
Particle filters (PFs), also known as sequential Monte Carlo (SMC) methods, are recursive algorithms that enable estimation of the likelihood of observed data and the conditional distribution of the latent process given data from a POMP model \citep{doucet2001introduction,cappe2007overview,doucet2009tutorial}. 

For the purpose of parameter learning, two iterated filtering (IF) approaches were developed, \cite{ionides2006inference} and its subsequent improvement \cite{ionides2015inference} referred to as 
IF$1$ and IF$2$ algorithms respectively, which coerce a particle filter into maximizing the likelihood function for unknown parameters.
PF methods and the parameter learning algorithms based on them (such as IF$2$) are capable of handling highly nonlinear latent processes \citep{king2008inapparent,ionides2011iterated}.
In epidemiological applications, IF$1$ and IF$2$ can considerably increase the accuracy of outbreak predictions while also allowing models whose structures reflect different underlying assumptions to be compared \citep{dobson2014mathematical}.
Unfortunately, PF suffers from rapid deterioration in performance as the model dimension increases \citep{bengtsson2008curse,snyder2008obstacles}. 
\cite{rebeschini2015can} rigorously showed that PF suffers the curse of dimensionality (COD) phenomenon, which says that the upper bound of the algorithmic filter error is exponential in the dimension of the state space of the underlying model. As expected, PF-based parameter learning algorithms suffer from the COD, limiting their applicability in high-dimensional problems.

The ensemble Kalman filter (EnKF) is a recursive filter suitable for problems with a large number of variables.
EnKF represents uncertainty in the latent state space using a finite collection of state values, and we refer to these ensemble members as particles by analogy with PF.
EnKF differs from PF by adopting a Gaussian approximation in the rule used to update the particles when filtering.
EnKF methods have been used for geophysical models in data assimilation due to their computational scalability to high dimensions \citep{houtekamer2001sequential, evensen1994sequential,katzfuss2020ensemble}.
For the parameter learning purpose, the iterated EnKF (IEnKF) algorithm extends the IF$2$ approach for parameter estimation by replacing a PF with an EnKF; it propagates the ensemble members by simulation from the dynamic model and then updates the ensemble to assimilate observations using a Gaussian-inspired rule \citep{li2020substantial}.
Given that EnKF relies on locally linear and Gaussian approximations, it can be ineffective for highly nonlinear and non-Gaussian systems \citep{ades2015equivalent,lei2010comparison,miller1999data}.
Unsurprisingly, the corresponding unsuitability carries to EnKF-based parameter learning algorithms (such as IEnKF). 

Block sampling strategies for PF were proposed with temporal blocks in \citep{doucet2006efficient}.  \cite{rebeschini2015can} investigated spatial blocks and proved that a block PF (BPF) beats the COD under certain conditions.
However, the beautiful work of \cite{rebeschini2015can} is theoretical in nature and was not anticipated to be applicable to real high-dimensional problems (page $2812$ therein).
In recent years, many efforts have been undertaken to develop practical methods for these problems by developing the  ``block'' concept, which include, but are not limited to, the following:
\cite{johansen2015blocks} proposed a method for systems identification based on both the block sampling idea and the annealed importance sampling approach;
\cite{singh2017blocking} applied the particle Gibbs algorithm inside a generic Gibbs sampler over temporal blocks to handle long time series; \cite{park2020inference} proposed a twisted particle filter model \citep{whiteley2014twisted} with iterated auxiliary PFs \citep{guarniero2017iterated} to infer on moderately high-dimensional spatiotemporal models where its particle filtering corresponds to an adapted version of the block sampling method; \cite{goldman2021spatiotemporal} proposed a blocked sampling scheme for latent state inference in high-dimensional state space models; \cite{ionides2020bagged} proposed the bagged filter for  partially observed interacting systems and showed that BPF can perform well on practical scientific models.

So far there is no high-dimensional parameter learning approach that is generically applicable over partially observed, highly nonlinear, stochastic spatiotemporal processes.
The goal of this paper is to develop such an  algorithm that is generically applicable and able to beat the COD.
Unlike the limited theoretical understanding of EnKF and hence IEnKF, the proposed algorithm has rigorous convergence analysis with a precise error bound.

\subsection{Our contributions}
\label{sec:Our_contributions}
In this paper, we propose the iterated BPF (IBPF) algorithm.
The contributions of the paper fall into four distinct categories:

\begin{enumerate}
	\item \textbf{General graphical model structure.} In this paper, we consider a general graphical POMP model having general state spaces and measures, general transition densities, and a general graph structure. Specifically, 
the latent state $(X_n)_{n\geq 0}$, the observation sequence $(Y_n)_{n\geq 1}$ that is conditionally independent given $(X_n)_{n\geq 1}$, and the auxiliary Markov chain $(\Theta_n)_{n\geq 0}$ for parameter learning purpose, have their own state spaces which are all Polish spaces endowed with their own general reference measures.
The transition densities of $X_n$, $Y_n$, and $\Theta_n$ are all time-inhomogeneous and in the general form where only standard conditions are required. 

	\vspace{-0.22 cm}
	
	\item \textbf{Innovative methodology.} 
		IBPF embeds BPF on an extended state space in an iterative scheme that constructs parameter values approaching the maximum of the likelihood function.
		It inherits from BPF the property that only observations in each block are used to update predictions, which is the key to scalability. When all vertices are in a single block, IBPF is in nature IF$2$; when there is inference on the latent process, all vertices are in a single block, and parameters are known, IBPF is PF; when there is inference on the latent process and parameters are known, IBPF is BPF; when all vertices are in a single block and there is no particle involved, IBPF is the iterated importance sampling. 
		
		\vspace{-0.22 cm}
		
	\item \textbf{Theoretical contribution.} Our IBPF algorithm has rigorous performance guarantees in terms of graph dimensions, time steps, algorithm convergence, and likelihood maximization.  In Theorem \ref{thm:main_theorem}, under standard assumptions, we rigorously show that the algorithmic error can be bounded using the dimension of a local block, uniformly both in time and in the model dimension.  Our result generalizes that of \cite{rebeschini2015can} to the time-inhomogeneous setting.  They introduced the mathematical machinery of a local particle filtering algorithm in high dimension that had not previously been applied in the study of nonlinear filtering, however being time-homogeneous is a limitation in practical applications. Furthermore, with our precise bound, we provide exact sufficient conditions needed and reveal the influences of crucial  quantities (such as the range of interacting neighborhoods and the maximal number of blocks of interaction) on the error bound. At last, we rigorously show that by iterating a Bayes map together with perturbations of the parameter variable, the algorithm converges to a point mass at the maximum likelihood estimate.
	
			
	\item \textbf{Excellent performance.}
	For spatiotemporal modeling, it is appropriate and sometimes necessary to have some parameters vary across locations, for instance, for measles transmission modeling the basic reproduction number regarding the epidemic transmission speed. To demonstrate how to use IBPF and compare its performance with IF$2$ and IEnKF, we generalize the spatiotemporal model for measles transmission covered in \cite{park2020inference} and \cite{ionides2020bagged}, by allowing location-specific parameters. Extensive experiments reveal that  IF$2$ does not scale well which confirms the phenomenon that PF does not scale well with dimensions (e.g. \cite{bengtsson2008curse}), and IEnKF performs very badly for highly nonlinear and non-Gaussian problems confirming the same phenomenon of EnKF  (e.g. \cite{ades2015equivalent}). In all experiments, IBPF is able to 
	find parameter values with a likelihood on or better than that of the true parameters consistently. The performances of IBPF with respect to iterations and block sizes are further examined, and confirm our theoretical findings.

	%
\end{enumerate}

\subsection{Organization of the paper}
\label{sec:organization}
The rest of the paper proceeds as follows.
In Section~\ref{sec:setups}, we set up our general model and provide necessary definitions.
In Section \ref{sec:Main_results}, we provide the main results of this paper, by firstly describing IBPF in Section \ref{sec:Algorithm}, conducting preliminary algorithmic analyses in Section \ref{sec:Preliminary_analysis},  providing theoretical results in Section \ref{sec:beat_COD}, and then investigating likelihood convergence in Section \ref{sec:Block_MLE}.
In Section \ref{sec:Performance}, we conduct performance analysis through a generalized spatiotemperal model for measles covered in Section \ref{sec:measles_model}, over the dataset covered in Section \ref{sec:Spatiotemporal_illustration}, and evaluate the performance of IBPF, IF$2$ and IEnKF in Section \ref{sec:Performance_analysis}. 
We conclude with discussion and extensions in Section \ref{sec:Discussion_and_extensions}.
In Appendix \ref{sec:IF2}, we provide the algorithms of IF$2$ and IEnKF for comparison. In Appendix \ref{sec:Existing_results}, existing technical results are provided which are needed for rigorous proofs following. 
In Appendix \ref{sec:Preparation}, we prepare for mathematical derivations by properly defining filtering and correlation measurement quantities. 
We defer all the lemmas and propositions in bounding the bias and variance of the algorithmic error, 
to Appendix \ref{sec:bounding_bias_lemmas} and Appendix \ref{sec:bounding_variance_lemmas}, respectively. 
In Appendix \ref{sec:proof_main_theorem}, we provide a rigorous proof of Theorem \ref{thm:main_theorem} whose detailed technical discussions are provided in Appendix \ref{appendix:discussion}. 
In Appendix \ref{appendix:block_MLE}, we provide a rigorous proof of Theorem \ref{thm:ionidesA2_modified}.
Original parameter learning results without rescaling to account for spatial and time dimensions, are provided in
Appendix \ref{Sec:training_result_data}.

\section{Model and analysis setups}
\label{sec:setups}
In this section, we firstly describe the extended POMP model $(X_n, Y_n \;;\Theta_n)$ on graph $G$ in Section \ref{sec:model_structure}, and then the partition $\mathcal{K}$ that separates $G$ into nonoverlapping blocks in Section \ref{sec:Partition}, followed by the global and local metrics necessary to conduct analysis  in Section \ref{sec:metric}.

\subsection{Extended POMP model on graph}
\label{sec:model_structure}
A general POMP model is a Markov chain $(X_n, Y_n)$, where $(X_n)_{n\geq 0}$ is a Markov chain in a Polish state space $\mathbb{X}$, while $(Y_n)_{n\geq 1}$ is conditionally independent given $(X_n)_{n\geq 1}$ in a Polish state space $\mathbb{Y}$. Here, $X_n$ is not directly observable while $Y_n$ serves as its partial and noisy observations made at time $n$. 
Define the reference measure of $X_n$ (resp. $Y_n$) on its state space $\mathbb{X}$ (resp. $\mathbb{Y}$) as $\psi$ (resp. $\phi$).
Suppose that there is an unknown auxiliary Markov chain $\Theta_n$ for parameter learning, which has its own Polish state space $\mathbb{\Theta}$ and reference measure $\lambda$.
 For $n\geq 1$, with respect to $\psi$ we define the transition density of $X_n$ as $f_{X_n|X_{n-1}}(x_n|x_{n-1}\;; \theta_n)$,  with respect to $\phi$ we define the emission density (or measurement density) of $Y_n$ as $f_{Y_n|X_n}(y_n|x_n\;; \theta_n)$, and with respect to $\lambda$ we define the transition density of $\Theta_n$ as $f_{\Theta_n|\Theta_{n-1}}(\theta_n|\theta_{n-1}\,;\sigma)$ where $\sigma$ is a nonnegative constant. That is, with our extended Markov chain model $(X_n, Y_n\;;\Theta_n)$, its transition probability is given by
\begin{align*}
&P(A|(x_{n-1},y_{n-1})\;;\theta_n)\\
&=\int \mathbbm{1}_A(x_n,y_n) f_{X_n|X_{n-1}}(x_n|x_{n-1}\;; \theta_n) f_{Y_n|X_n}(y_n|x_n\;; \theta_n)\psi(dx_n)\phi(dy_n).
\end{align*}

The state of $(X_n,Y_n,\Theta_n)$ at each time $n$ is a random field $(X_n^v,Y_n^v,\Theta_n^v)_{v\in V}$ indexed by a finite undirected graph $G=(V,E)$, where $V$ stands for the set of vertices and $E$ stands for the set of edges. The graph describes the location relationship of data and the spatial degrees of freedom of the model. Based on the network structure, the state spaces $\mathbb{X}$, $\mathbb{Y}$, and $\mathbb{\Theta}$ can be written as the product forms
$$\mathbb{X}=\prod_{v\in V}\mathbb{X}^v, \quad\quad \mathbb{Y}=\prod_{v\in V}\mathbb{Y}^v, \quad\text{and}\quad \mathbb{\Theta}=\prod_{v\in V}\mathbb{\Theta}^v.$$
Define the reference measure of $X_n^v$ on its state space $\mathbb{X}^v$ as $\psi^v$. Define the reference measure of $Y_n^v$ on its state space $\mathbb{Y}^v$ as $\phi^v$. Define the reference measure of $\Theta_n^v$ on its state space $\mathbb{\Theta}^v$ as $\lambda^v$. Similarly, based on the network structure, we have the following product-formed expressions:
$$\psi=\prod_{v\in V}\psi^v, \quad\quad \phi=\prod_{v\in V}\phi^v,  \quad\quad \lambda=\prod_{v\in V}\lambda^v.$$
With respect to $\psi^v$ we define the transition density of $X_n^v$ as $f_{X_n^v|X_{n-1}}$, with respect to $\phi^v$ we define the transition density of $Y_n^v$ as $f_{Y_n^v|X_n^v}$, and with respect to $\lambda^v$ we define the transition density of $\Theta_n^v$ as $f_{\Theta_n^v|\Theta_{n-1}}$. Similarly, based on the network structure, we have the following product-formed expressions:
\begin{align}
\label{eqn:density_product_form}
f_{X_n|X_{n-1}}(x_n|x_{n-1}\;; \theta_n)&=\prod_{v\in V}f_{X_n^v|X_{n-1}}(x_n^v|x_{n-1}\;; \theta_n^v),\nonumber\\
f_{Y_n|X_n}(y_n|x_n\;; \theta_n)&=\prod_{v\in V}f_{Y_n^v|X_n^v}(y_n^v|x_n^v\;; \theta_n^v),\\
f_{\Theta_n|\Theta_{n-1}}(\theta_n|\theta_{n-1}\,;\sigma)&=\prod_{v\in V}f_{\Theta_n^v|\Theta_{n-1}}(\theta_n^v|\theta_{n-1}\,;\sigma).\nonumber
\end{align}

\subsection{Partition of the graph}
\label{sec:Partition}
We consider a partition $\mathcal{K}$ that partitions $V$ into nonoverlapping blocks, i.e.,
\[V=\bigcup_{K\in \mathcal{K}}K, \quad\quad K\cap K'=\emptyset\; \text{for}\; K\neq K',\; K,K'\in \mathcal{K}.\]
Based on the partition, we can write 
$$\Xi=(\Xi^K)_{K\in \mathcal{K}}=(\Xi^v)_{v\in V}, \quad\quad \Xi^{W}:=(\Xi^v)_{v\in W}\;\text{for}\; \forall W\subseteq V,$$
where $\Xi$ can be $X_n$, $Y_n$ or $\Theta_n$,
as well as the associated state space $\mathbb{X}$, $\mathbb{Y}$, or $\mathbb{\Theta}$.
For any set $W\subseteq V$, we use
$(\mathbb{X}\times\mathbb{\Theta})^W$ and $\mathbb{X}^W\times\mathbb{\Theta}^W$ interchangeably and
define
 \begin{equation}
\label{eqn:block_measure}
\psi^W(dx_n^W):=\prod_{v\in W}\psi^v (dx_n^v)\quad\text{and}\quad \lambda^W(d\theta_n^W):=\prod_{v\in W}\lambda^v (d\theta_n^v).
\end{equation}

We define the distance $d$ as the length of the shortest path in the graph $G$ connecting two vertices, based on which we define for each vertex $v\in V$ the $r$-neighborhood $N(v)$
as 
$$N(v):=\{v'\in V: d(v,v')\leq r\}.$$ 
For integers $0\leq m\leq n$, denote $$\Xi_{m:n}:=\{\Xi_m,\Xi_{m+1},\cdots,\Xi_n\}, $$
where $\Xi$ can be $X$, $Y$, or $\Theta$.
Suppose that, for $n\geq 1$, the conditional distribution of $X_n^v$ given $X_{0:n-1}$ depends on $X_{n-1}^{N(v)}$ only and then we have
$$f_{X_n^v|X_{n-1}}(x_n^v|x_{n-1}\;; \theta_n^v)=f_{X_n^v|X_{n-1}}(x_n^v|\overline{x}_{n-1}\;; \theta_n^v),$$
whenever $x_{n-1}^{N(v)}=\overline{x}_{n-1}^{N(v)}$ where $x,\overline{x}\in \mathbb{X}$ and $x\neq \overline{x}$. 
That is, if $x_{n-1}$ and $\overline{x}_{n-1}$ coincide on the neighbouring vertices of $v$, then their associated transition densities are the same.
Similarly suppose that, for $n\geq 1$,  the conditional distribution of $\Theta_n^v$ given $\Theta_{0:n-1}$ depends on  $\Theta_{n-1}^{N(v)}$ only and then we have
$$f_{\Theta_n^v|\Theta_{n-1}}(\theta_n^v|\theta_{n-1}\,;\sigma)=f_{\Theta_n^v|\Theta_{n-1}}(\theta_n^v|\overline{\theta}_{n-1}\,;\sigma),$$
whenever $\theta_{n-1}^{N(v)}=\overline{\theta}_{n-1}^{N(v)}$  where $\theta,\overline{\theta}\in \mathbb{\Theta}$ and $\theta\neq \overline{\theta}$. 
An illustration of the dependence with $r=1$ is provided in Figure \ref{f:Dependency_graph}.

\begin{figure}[htb!]
\centering
\includegraphics[width=0.68\textwidth,height=8.5cm]{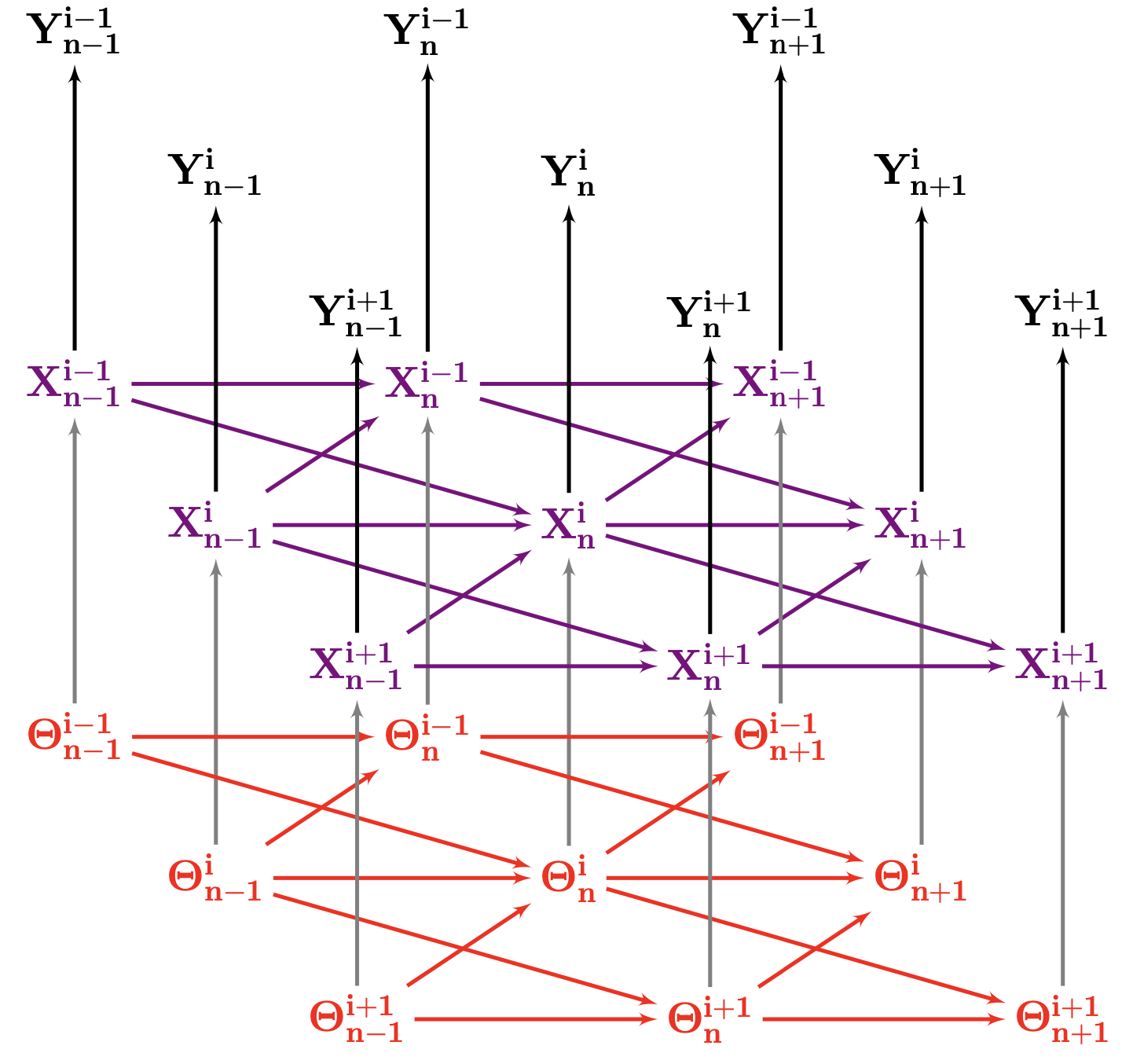}
\caption{\footnotesize Illustration of $1$-neighborhood dependence.}
\label{f:Dependency_graph}
\end{figure}

For any sets $W,W'\subseteq V$, define
\begin{equation}
\label{eqn:distance_blocks_def}
d(W,W'):=\min_{v\in W}\min_{v'\in W'} d(v,v'),
\end{equation}
based on which we can define the collection of blocks that interact with any block $K\in \mathcal{K}$ in one step of the dynamics as 
\begin{equation}
\label{eqn:block_neighborhood}
N(K):=\{K'\in \mathcal{K}: d(K,K')\leq r\}.
\end{equation}
Given a set $W\subseteq V$, denote the $r$-inner boundary of $W$ as the subset of vertices in $W$ that can interact with vertices outside $W$,
\begin{equation}
\label{eqn:partial_definition}
\partial W:=\{v\in W: N(v)\nsubseteq W \},
\end{equation}
and denote the interior of $W$ as
\begin{equation}
\label{eqn:interior_definition}
\operatorname{int}(W):=W\backslash\partial W.
\end{equation}
We now denote some quantities which will be used frequently throughout the paper:
the maximal size of one single block in the partition $\mathcal{K}$ 
\begin{equation}
\label{eqn:maxsize_block}
|\mathcal{K}|_{\infty}:=\max_{K\in\mathcal{K}}\card(K),
\end{equation}
where $\card(K)$ denotes the cardinality of $K$;
the maximal number of vertices that interact with one single vertex in its $r$-neighborhood  in one step of the dynamics
\begin{equation}
\label{eqn:maxn_inter_vertices}
\Delta:=\max_{v\in V}\card\{v'\in V:d(v,v')\leq r\};
\end{equation}
the maximal number of blocks that interact with one single block  in one step of the dynamics
\begin{equation}
\label{eqn:maxn_inter_blocks}
\Delta_\mathcal{K}:=\max_{K\in\mathcal{K}}\card\{K'\in\mathcal{K}:d(K,K')\leq r\}.
\end{equation}


\subsection{Metrics}
\label{sec:metric}
We assume that the process $(X,Y;\;\Theta)$ is realized on its
canonical probability space; denote $\mathbf{P}$ and $\mathbf{E}$ as the  probability measure and expectation on that space, respectively.
We use the functional analytic notation $\rho(g)$ for the integral of a measurable function $g$ with respect to the measure $\rho$ (provided this integral exists),
$$\rho(g):=\int g d\rho=\int g(x) d\rho(x)=\int g(x) \rho(dx).$$
Between two random measures $\rho$ and $\rho'$ on space $\mathbb{S}$, we define the distance
\begin{equation}
\label{eqn:vertiii_definition}
\vertiii{\rho- \rho'}:=\sup_{g\in \mathcal{S}:|g|\leq 1}\left[\mathbf{E}|\rho(g)-\rho'(g)|^2\right]^{1/2},
\end{equation}
where $\mathcal{S}$ denotes the class of measurable
functions $g: \mathbb{S}\rightarrow \mathbb{R}$,
and define the local distance, for $\K\subseteq V$, 
\begin{equation}
\label{eqn:vertiii_K_definition}
\vertiii{\rho- \rho'}_{\K}:=\sup_{g\in \mathcal{S}^\K:|g|\leq 1}\left[\mathbf{E}|\rho(g)-\rho'(g)|^2\right]^{1/2}.
\end{equation}
Here, $\mathcal{S}^\K$ denotes the class of measurable
functions $g: \mathbb{S}\rightarrow \mathbb{R}$ such that $g(x)=g(\overline{x})$ whenever $x^{\K}=\overline{x}^{\K}$.
That is, $\mathcal{S}^\K$ is the class of measurable
functions such that when $x$ and $\overline{x}$ coincide on the set $\K$ then their associated function values are the same.
Similarly, we define the total variation distance between two probability measures $\rho$ and $\rho'$ on $\mathbb{S}$
\begin{equation}
\label{eqn:tvd}
\|\rho-\rho'\|:=\sup_{g\in \mathcal{S}:|g|\leq 1}|\rho(g)-\rho'(g)|,
\end{equation}
and define the local total
variation distance, for $\K\subseteq V$, 
\begin{equation}
\label{eqn:ltvd}
\|\rho-\rho'\|_{\K}:=\sup_{g\in \mathcal{S}^{\K}:|g|\leq 1}|\rho(g)-\rho'(g)|.
\end{equation}

\section{Main results}
\label{sec:Main_results}
In this section, we describe our IBPF algorithm for parameter learning over general graphical POMP models in Section \ref{sec:Algorithm}, conduct its preliminary algorithmic analyses in Section \ref{sec:Preliminary_analysis},  establish its theoretical guarantees on algorithm performances and convergences in Section \ref{sec:beat_COD}, and then investigate maximum
likelihood estimates (MLEs) in Section \ref{sec:Block_MLE}.

\subsection{Algorithm}
\label{sec:Algorithm} 
We propose the IBPF algorithm in Algorithm \ref{Algorithm_IBPF}.
For notational convenience, we set 
$$1{:}N:=\{1,2,\ldots,N\} \quad\quad\mbox{ for $N\in \mathbb{N}$}$$
throughout the paper.
In Algorithm \ref{Algorithm_IBPF}, $\Theta^{F,m}_{n,j}$ (resp. $X^{F,m}_{n,j}$) is the $j$-th particle in the Monte Carlo
representation of the $m$-th iteration of a filtering recursion at time $n$, where this filtering recursion is coupled with a prediction recursion represented by $\Theta^{P,m}_{n,j}$ (resp. $X^{P,m}_{n,j}$). 
The IBPF algorithm allows users to infer initial values of latent states by incorporating initial values into the parameter set.
That is, let $\theta=(\theta_{\text{init}}, \theta_{\text{dynamic}})$, where $\theta_{\text{init}}$  stands for  initial values of latent states and
 $\theta_{\text{dynamic}}$ stands for the parameters affecting $f_{X_n|X_{n-1}}$ and $f_{Y_n|X_n}$ for some or all $n\in 1:N$.
In this case, the initial density  $f_{X_0}(x_0;\theta)$ is a Dirac mass function at $x_0=\theta_{\text{init}}$.
This case is common for scientific modeling and also is the situation addressed by our theory.
Numerical experiments on learning initial values can be seen in Section \ref{sec:Performance_analysis}. 

\begin{algorithm}[!htb]
\noindent\begin{tabular}{l}
Initial value function $f_{X_0}(\theta)$ \\
Simulator for $f_{X_n|X_{n-1}}(x_n \mid x_{n-1}\,; \theta)$, $n\in 1{:}N$ \\
Evaluator for $f_{Y_n|X_n}(y_n\mid x_n\,;\theta)$, $n\in 1{:}N$ \\
Data, $y_{1:N}$ \\
Number of iterations, $M$ \\
Number of particles, $J$ \\
Partition, $\mathcal{K}$ \\
Initial parameter swarm, $\{\Theta^0_j, \mbox{ $j\in 1{:}J$}\}$ \\
Perturbation density, $f_{\Theta_n|\Theta_{n-1}}(\theta\mid \vartheta\,;\sigma)$, $n\in 1{:}N$\\
Perturbation sequence, $\sigma_{1:M}$ \\
{\bf Output:}\rule[-1.5mm]{0mm}{6mm} 
Final parameter swarm, $\{\Theta^M_j, \mbox{ $j\in 1{:}J$}\}$ and log-likelihood $\lb$\\
1. \, For $m$ in $1\mycolon M$\rule[0mm]{0mm}{5mm}\\
2. \asp     Set $\Theta^{F,m}_{0,j}=\Theta^{m-1}_{j}$ for $j\in 1{:}J$\mystretch\\
3. \asp Set $X_{0,j}^{F,m}=f_{X_0}(\Theta^{F,m}_{0,j})$ for $j\in 1{:}J$\mystretch\\
4. \asp For $n$ in $1\mycolon N$\\
5. \asp\asp Draw $\Theta^{P,m}_{n,j}\sim f_{\Theta_n|\Theta_{n-1}}(\theta_n\mid \Theta^{F,m}_{n-1,j}\,; \sigma_m)$ for $j\in 1{:}J$\mystretch\\
6. \asp\asp   Draw $X_{n,j}^{P,m}\sim f_{X_n|X_{n-1}}(x_n\mid X^{F,m}_{n-1,j}\;; \Theta^{P,m}_{n,j})$ for $j\in 1{:}J$
\mystretch\\
7. \asp\asp For $K\in\mathcal{K}$\\
8. \asp\asp\asp\asp   Compute $w_{n,j}^{K,m} = \prod_{v\in K}f_{Y_n^v|X_n^v}(y^{v}_n \mid X_{n,j}^{v,P,m}\; ; \Theta^{v,P,m}_{n,j})$ for $j\in 1{:}J$  \mystretch\\
9. \asp\asp\asp\asp   Draw $s_{1:J}^{K,m}$ with $\operatorname{Prob}(s_j^{K,m}=i)=  w_{n,i}^{K,m}\Big/\sum_{j=1}^J w_{n,j}^{K,m}$  \\
10. \asp\asp End For\\ %
11. \asp\asp  Set $X^{F,m}_{n,j}=(X^{K,F,m}_{n,j})_{K\in \mathcal{K}}$ where $X^{K,F,m}_{n,j}=X^{K,P,m}_{n,s_j^{K,m}}$ for $j\in 1{:}J$\\  
12. \asp\asp  Set $\Theta^{F,m}_{n,j}=(\Theta^{K,F,m}_{n,j})_{K\in \mathcal{K}}$ where $\Theta^{K,F,m}_{n,j}=\Theta^{K,P,m}_{n,s_j^{K,m}}$ for $j\in 1{:}J$ \mystretch\\
13. \asp End For\\ %
14. \asp   Set $\Theta^{m}_{j}=(\Theta^{F,m}_{n,j})$ for $j\in 1{:}J$\\
15. End For\\
16. Set $\lb=\sum_{n=1}^N \sum_{K\in\mathcal{K}}\log (\frac{1}{J}\sum_{j=1}^J w_{n,j}^{K,M})$
\end{tabular}
\caption{(The IBPF algorithm)}
\label{Algorithm_IBPF}
\end{algorithm}

\subsection{Preliminary analysis}
\label{sec:Preliminary_analysis}
Inspecting the IBPF pseudocode (Algorithm~\ref{Algorithm_IBPF}), we can see that  the same set of observations $Y_1,\ldots,Y_n$ is used in each of the $M$ iterations. Let us first focus on one of the $M$ iterations, say  $m=1$, and ignore the $m$ superscript/subscript.

Given the observations $Y_1,\ldots,Y_n$, we aim to approximate the joint nonlinear filter, for $n\geq 1$,
\begin{align*}
\pi_n(A)=\pi_n(A_x \times A_{\theta})=\mathbf{P}[X_n \in A_x, \Theta_n\in A_{\theta} \mid Y_1,\ldots,Y_n].
\end{align*}
The filter $\pi_n$ for $n\geq 1$ that we also call the true filter to differentiate with the IBPF approximated filter, can be expressed in a recursive way
\begin{equation}
\label{eqn:pi_recursion}
\pi_n= \mathsf{F}_n \pi_{n-1}, \quad\quad \pi_0=\delta(x,\theta)=\delta_x\delta_{\theta},
\end{equation}
where $\delta_x$ stands for the point mass at $x$,
with 
\begin{equation}
\label{eqn:pi_def}
\mathsf{F}_n=\mathsf{C}_n \mathsf{P}_n
\end{equation}
 evolving as follows:
\begin{align*}
\pi_{n-1} \xrightarrow[]{\text{prediction}} \pi_{n|n-1}=\mathsf{P}_n \pi_{n-1}
\xrightarrow[]{\text{correction}} \pi_n=\mathsf{C}_n\pi_{n|n-1}
\end{align*}
where $\mathsf{P}_n$ is defined as the prediction operator 
\begin{equation}
\begin{split}
\label{eqn:prediction_operator_definition}
(\mathsf{P}_n\rho)(g)=&\int  g(x_n,\theta_n)f_{X_n|X_{n-1}}(x_n\mid x_{n-1}\,;\theta_n)f_{\Theta_n|\Theta_{n-1}}(\theta_n\mid \theta_{n-1}\,; \sigma)\\
&\quad\quad\quad\quad\quad\quad\quad\quad\times\psi(dx_n)\lambda(d\theta_n)\rho(dx_{n-1},d\theta_{n-1}),
\end{split}
\end{equation}
and $\mathsf{C}_n$ is defined as the correction operator
\begin{equation}
\begin{split}
\label{eqn:correction_operator_definition}
(\mathsf{C}_n\rho)(g)=&\dfrac{
\int g(x_n,\theta_n)f_{Y_n|X_n}(Y_n\mid x_n\,;\theta_n)\rho(dx_n,d\theta_n)}
{\int f_{Y_n|X_n}(Y_n\mid x_n\,;\theta_n)\rho(dx_n,d\theta_n)},
\end{split}
\end{equation}
for any probability measure $\rho$ on $\mathbb{X}\times \mathbb{\Theta}$.

To facilitate analysis, we define an intermediate filter $\widetilde{\pi}_n$ which can be expressed in a recursive way
\begin{equation}
\label{eqn:widetilde_pi_recursion}
\widetilde{\pi}_n=\widetilde{\mathsf{F}}_n \widetilde{\pi}_{n-1},\quad\quad \widetilde{\pi}_0=\delta_x\delta_{\theta},
\end{equation}
with
\begin{equation}
\label{eqn:widetilde_pi_def} 
\widetilde{\mathsf{F}}_n=\mathsf{C}_n \mathsf{B}\mathsf{P}_n
\end{equation}
 evolving as follows:
\begin{align*}
\widetilde{\pi}_{n-1} \xrightarrow[]{\text{prediction}} \widetilde{\pi}_{n|n-1}=\mathsf{P}_n\widetilde{\pi}_{n-1}
\xrightarrow[\text{correction}]{\text{blocking}}  \widetilde{\pi}_n=\mathsf{C}_n \mathsf{B} \widetilde{\pi}_{n|n-1},
\end{align*}
where, for any measure $\rho$ on $\mathbb{X}\times \mathbb{\Theta}$, $\mathsf{B}$ is defined as the blocking operator 
\begin{equation}
\label{eqn:blockingoperator}
\mathsf{B}\rho:=\bigotimes_{K\in \mathcal{K}}\mathsf{B}^K \rho,
\end{equation}
with $\mathsf{B}^K \rho$ being the marginal of $\rho$ on $(\mathbb{X}\times \mathbb{\Theta})^K$. 
Before we can explicitly show the effect of $\mathsf{B}^K$ on $\widetilde{\mathsf{F}}_n\rho$ for any $n\geq 1$ and any measure $\rho$ on $\mathbb{X}\times \mathbb{\Theta}$, we need to first define the block versions of the prediction operator $\mathsf{P}_n$ given in \eqref{eqn:prediction_operator_definition} and correction operator $\mathsf{C}_n$ given in \eqref{eqn:correction_operator_definition}, as follows: 
define $\mathsf{P}_n^K$ as the 
prediction operator specific to block $K$ 
\begin{equation}
\begin{split}
\label{eqn:block_prediction_operator_definition}
(\mathsf{P}_n^K\rho_1)(g)=&\int  g(x_n^K,\theta_n^K)\prod_{v\in K}f_{X_n^v|X_{n-1}}(x_n^v\mid x_{n-1}\,;\theta_n^v)f_{\Theta_n^v|\Theta_{n-1}}(\theta_n^v\mid \theta_{n-1}\,; \sigma)\\
&\quad\quad\quad\quad\quad\quad\quad\quad\times\psi^v(dx_n^v)\lambda^v(d\theta_n^v)\rho_1(dx_{n-1},d\theta_{n-1}),
\end{split}
\end{equation}
for any measure $\rho_1$ on $(\mathbb{X}\times \mathbb{\Theta})^{\cup_{K'\in N(K)}K'}$;
define $\mathsf{C}_n^K$ as the correction operator specific to block $K$
\begin{equation}
\begin{split}
\label{eqn:block_correction_operator_definition}
(\mathsf{C}_n^K\rho_2)(g)=&\dfrac{
\int g(x_n^K,\theta_n^K)\prod_{v\in K}f_{Y_n^v|X_n^v}(Y_n^v\mid x_n^v\,;\theta_n^v)\rho_2(dx_n^K,d\theta_n^K)}
{\int \prod_{v\in K}f_{Y_n^v|X_n^v}(Y_n^v\mid x_n^v\,;\theta_n^v)\rho_2(dx_n^K,d\theta_n^K)},
\end{split}
\end{equation}
for any measure $\rho_2$ on $(\mathbb{X}\times \mathbb{\Theta})^K$. Then we can write 
\begin{equation}
\label{eqn:B_K_effect}
\mathsf{B}^K\widetilde{\mathsf{F}}_n\rho=\mathsf{C}_n^K \mathsf{P}_n^K \bigotimes_{K'\in N(K)}\rho^{K'},
\end{equation}
for any measure $\rho$ on $\mathbb{X}\times \mathbb{\Theta}$.

The IBPF approximated filter denoted as $\widehat{\pi}_n$ can be expressed in a recursive way
\begin{equation}
\label{eqn:widehat_pi_recursion}
\widehat{\pi}_n=\widehat{\mathsf{F}}_n \widehat{\pi}_{n-1},
\quad\quad \widehat{\pi}_0=\delta_x\delta_{\theta},
\end{equation} 
with 
\begin{equation}
\label{eqn:widehat_pi_def} 
\widehat{\mathsf{F}}_n=\mathsf{C}_n \mathsf{B}\mathsf{S}^{J}\mathsf{P}_n
\end{equation}
evolving as follows:
\begin{align*}
\widehat{\pi}_{n-1} \xrightarrow[\text{sampling}]{\text{prediction}} \widehat{\pi}_{n|n-1}=\mathsf{S}^{J}\mathsf{P}_n\widehat{\pi}_{n-1}
\xrightarrow[\text{correction}]{\text{blocking}} \widehat{\pi}_n=\mathsf{C}_n \mathsf{B} \widehat{\pi}_{n|n-1},
\end{align*}
where $\mathsf{S}^{J}$ is defined as the sampling operator 
\begin{equation}
\label{eqn:samplingoperator}
\mathsf{S}^{J}\rho=\frac{1}{J}\sum_{j=1}^J\delta_{x_j}.
\end{equation}
for any probability measure $\rho$ and $\{x_j\}_{\{j=1,\cdots,J\}}$ being i.i.d. samples distributed according to $\rho$. We note that $\mathsf{S}^{J}\mathsf{P}_n$ corresponds to lines $5-6$ in Algorithm \ref{Algorithm_IBPF}, and $\mathsf{C}_n \mathsf{B}$  corresponds to lines $7-12$ in Algorithm \ref{Algorithm_IBPF}.

\subsection{Beating the curse of dimensionality}
\label{sec:beat_COD}

The following assumption is enforced in obtaining our main theoretical result (Theorem \ref{thm:main_theorem}): 
\begin{assumption} 
	\label{assumption}
	For any $v\in V$, $x_{n-1},x_n\in \mathbb{X}$, $y_n\in \mathbb{Y}$, $\theta_{n-1},\theta_n \in \mathbb{\Theta}$, and $n\geq 1$, we impose the following conditions:
	\begin{enumerate}[(1)]
		\item Suppose there exists $\epsilon_x>0$ such that
		$$\epsilon_x\leq f_{X_n^v|X_{n-1}}(x_n^v|x_{n-1}\;; \theta_n^v)\leq \epsilon_x^{-1}.$$
		\item Suppose there exists $\epsilon_y>0$ such that
		$$\epsilon_y \leq f_{Y_n^v|X_n^v}(y_n^v|x_n^v\;; \theta_n^v)\leq \epsilon_y^{-1}.$$
		\item Suppose there exist $\epsilon_{\theta}(\sigma)>0$ and $\sigma\geq 0$ such that 
		$$\epsilon_{\theta}(\sigma)\leq f_{\Theta_n^v|\Theta_{n-1}}(\theta_n^v\mid \theta_{n-1}\,; \sigma)\leq [\epsilon_{\theta}(\sigma)]^{-1}.$$
	\end{enumerate}
\end{assumption}
In Assumption \ref{assumption}, $(1)$ and $(2)$ are the same as the conditions enforced in Theorem $2.1$ of \cite{rebeschini2015can} which are localized versions of standard assumptions that are routinely employed in the analysis of PFs, and $(3)$ is the same condition on the transition density of $\Theta$. Similar to the global mixing assumption implying that the underlying Markov chain is strongly ergodic, a local transition density being bounded above and below as a local counterpart of the global mixing assumption, could be viewed as a local ergodicity assumption on the model.

Recall that $|\mathcal{K}|_{\infty}$ defined in 
\eqref{eqn:maxsize_block} is the maximal size of one single block in the partition $\mathcal{K}$, $\Delta$ defined in 
\eqref{eqn:maxn_inter_vertices} is the maximal number of vertices that interact with one single vertex in its $r$-neighborhood, and $\Delta_\mathcal{K}$ defined in \eqref{eqn:maxn_inter_blocks} is the maximal number of blocks that interact with a single block (including itself). 
In the following theorem, we bound the error generated by our IBPF algorithmic filter $\widehat{\pi}_n$ defined in \eqref{eqn:widehat_pi_recursion}, to the unalgorithmic true filter $\pi_n$ defined in \eqref{eqn:pi_recursion}, uniformly both in time $n$ and
in the model dimension $\card(V)$:
\begin{thm}
\label{thm:main_theorem}
With $\epsilon_x$, $\epsilon_{\theta}(\sigma)$, and $\epsilon_y$ satisfying Assumption \ref{assumption}, when 
\begin{align}
\label{eqn:main_thm_condition}
\epsilon_x\epsilon_{\theta}(\sigma)>\left(1-\frac{1}{16\Delta_{\mathcal{K}}\Delta^2} \right)^{\frac{1}{2\Delta}},
\end{align}
 for every $n\geq 0$, $K\in \mathcal{K}$ and $\K \subseteq K$, we have
\begin{align*}
\vertiii{\widehat{\pi}_n-\pi_n}_{\K}
\leq & \frac{\card(\K)}{1-e^{-\beta}}\bigg[ 7e^{-\beta}(1-\epsilon_x^{2\Delta}[\epsilon_{\theta}(\sigma)]^{2\Delta})e^{-\beta d(\K,\partial{K})}\\
&\hspace*{2cm}+\frac{40}{\sqrt{J}}[\epsilon_{\theta}(\sigma)]^{-4|\mathcal{K}|_{\infty}}\epsilon_x^{-4|\mathcal{K}|_{\infty}}\epsilon_y^{-2|\mathcal{K}|_{\infty}(\Delta_{\mathcal{K}}+1)}\Delta_{\mathcal{K}}\bigg],
\end{align*}
where $\card(\,\cdot\,)$ stands for cardinality and
\begin{align}
\label{eqn:beta_definition}
\beta=\frac{1}{2r}\log\left(\frac{1}{16\Delta_{\mathcal{K}}\Delta^2(1-\epsilon_x^{2\Delta}[\epsilon_{\theta}(\sigma)]^{2\Delta})} \right).
\end{align}
\end{thm}

In Theorem \ref{thm:main_theorem}, we provide a tighter bound than the important result of \cite{rebeschini2015can} (Theorem $2.1$ therein),  generalizes their results to a time-inhomogeneous setting, and made their proofs clearer. The rigorous proof of Theorem \ref{thm:main_theorem} is postponed to Appendix \ref{sec:proof_main_theorem}. 
Here, we merely interpret the upper bound in Theorem \ref{thm:main_theorem} in terms of the graph dimension and the time dimension, further discussions are provided in Appendix \ref{appendix:discussion}. For the standard PF algorithm where all observations are used to update the filtering distribution, the algorithmic error is exponential in the dimension of the model under the global metric $\vertiii{\;\cdot\;}$ defined in \eqref{eqn:vertiii_definition}.  
For the IBPF algorithm, only observations in a block, say $K$, are used to update the filtering distribution in that block. From Theorem \ref{thm:main_theorem}, we can see that under the local metric $\vertiii{\;\cdot\;}_{\K}$ defined in \eqref{eqn:vertiii_K_definition}, the algorithmic error is  merely exponential in the dimension of a set $\card(\K)$ instead of the dimension of the graph $\card(V)$. That is, our IBPF algorithm has a  rigorous performance guarantee in terms of the graph dimension, thus beating the COD.
Next,  since the upper bound in Theorem \ref{thm:main_theorem} is uniform on all the time steps $n$, our IBPF algorithm has a rigorous performance guarantee in terms of the time dimension. The second term of the bound quantifies the error due to the variance of the Monte Carlo sampling of the IBPF algorithm. As in the standard PF analysis, Monte Carlo sampling provides the $\frac{1}{\sqrt{J}}$ factor under the local metric in this local update setting. Given that each block interacts with at most $\Delta_\mathcal{K}$ neighbors in the previous time step, the $\Delta_\mathcal{K}$ factor in the second term is expected.

\subsection{MLEs}
\label{sec:Block_MLE}
The IBPF algorithm generalizes the data cloning method \citep{lele2007data,lele2010estimability}, which is based on the observation that iterating a Bayes map converges to a point mass at the maximum likelihood estimate. Combining such iterations with perturbations of model parameters improves the numerical stability of data cloning and provides a foundation for stable algorithms \citep{ionides2015inference}. To be specific, the same set of data $Y_1,\ldots,Y_N$ is used in any one of the $M$ iterations of the IBPF algorithm, given the result of the $m$-th iteration for $m\in 1{:}(M-1)$ is simply the initial value of the $(m+1)$-th iteration, we can see that  
all $M$ iterations together can be represented as
a filtering problem on $M$ replications of the data as follows:
$$\bigg\{ \underbrace{\{ Y_1,\ldots,Y_N\}}_{1},\underbrace{\{ Y_1,\ldots,Y_N\}}_{2}, \ldots, \underbrace{\{ Y_1,\ldots,Y_N\}}_{M} \bigg\}.$$
As in the previous subsections, our strategy is to analyze the original theoretical quantity and then explore its algorithmic approximation. 

The joint density of the classical POMP model can be written as
\begin{align*}
f_{X_{0:N},Y_{1:N}}(x_{0:N},y_{1:N}\;;\theta)
=f_{X_0}(x_0\;; \theta)\prod_{n=1}^N f_{X_n|X_{n-1}}(x_n|x_{n-1}\;; \theta) f_{Y_n|X_n}(y_n|x_n\;; \theta).
\end{align*}
We write $f_{Y_{1:N}}(y_{1:N}\;;\theta)$ for the marginal density of $Y_{1:N}$. Then the likelihood function is defined to be $\lik(\theta)=f_{Y_{1:N}}(y_{1:N}\;;\theta)$, where the data is a sequence of observations $y_{1:N}$. A MLE is a value $\widehat{\theta}$ that maximizes $\lik(\theta)$. We define an extended likelihood function on $\mathbb{\Theta}^{N+1}$ by
$$\breve\lik(\theta_{0:N})=\!\! \int\!\!\dots\!\!\int  dx_0\dots dx_N \bigg\{ f_{X_0}(x_0\;;\theta_0)
\prod_{n=1}^N f_{X_n|X_{n-1}}(x_n\mid x_{n-1}\;;\theta_{n})f_{Y_n|X_n}(y_n\mid x_{n}\;;\theta_{n})\bigg\}.$$
Each $m$ iteration of data cloning corresponds to an operator $T_\sigma$, which is a composition of a parameter perturbation with a Bayes map that multiplies the likelihood and renormalizes (page $2$ of \cite{ionides2015inference}), i.e.,
\begin{equation}\label{recursion:if} 
T_{\sigma} g(\theta_N)=
\frac{\int \breve\lik(\theta_{0:N})f_{\Theta_{0:N}}(\theta_{0:N}|\vartheta\,;\sigma)g(\vartheta)\,d\vartheta \, d\theta_{0:N-1}}{\int \breve\lik(\theta_{0:N})f_{\Theta_{0:N}}(\theta_{0:N}|\vartheta\,;\sigma)g(\vartheta)\,d\vartheta \, d\theta_{0:N}},
\end{equation} 
with $g$ and $T_\sigma g$ approximating the initial and final density of the
parameter swarm, where $$f_{\Theta_{0:N}}(\theta_{0:N}\given\vartheta\,;\sigma)=f_{\Theta_{0}}(\theta_0 \given\vartheta\,;\sigma)\prod_{n=1}^Nf_{\Theta_n|\Theta_{n-1}}(\theta_n\given\theta_{n-1}\,;\sigma)$$ 
and $\vartheta$ is the mean of the distribution of $\Theta_{0}$.
Iteration of the Bayes map alone has a central limit theorem that forms the theoretical basis for the data cloning methodology \citep{lele2007data,lele2010estimability}. 

IBPF is the approximation to $T_\sigma^M$ which is the $M$-th iterate of $T_\sigma$.
Following \citet{ionides2015inference}, we first show that $\lim_{m\to\infty}T_\sigma^m g=g_\sigma$ exists for every fixed $\sigma>0$, and 
as the noise intensity becomes small $\lim_{\sigma\to 0}g_\sigma$ approaches a point mass at the MLE if it exists. 
Then we show that when the number of particles $J$ and the number of iterations $M$ become large, the IBPF algorithm numerically approximates $g_\sigma$.
Let $\{\breve{\Theta}_{0:N}^m,\, m=1,2,\ldots\}$ be a Markov chain such that $\breve{\Theta}_{0:N}^1$
has density $\int f_{\Theta_{0:N}}(\theta_{0:N}\given\vartheta\,;\sigma)g(\vartheta)\,d\vartheta$
and $\breve{\Theta}_{0:N}^m$ has conditional density 
$f_{\Theta_{0:N}}(\theta_{0:N}\given\vartheta_N\,;\sigma)$ given $\breve{\Theta}_{0:N}^{m-1}=\vartheta_{0:N}$
for $m\geq 2$. Suppose that $\{\breve{\Theta}_{0:N}^m,\, m=1,2,\ldots\}$ is constructed on the canonical probability space $\Omega=\{({\theta}_{0:N}^1, {\theta}_{0:N}^2,\ldots)\}$ with ${\theta}_{0:N}^m=\breve{\Theta}_{0:N}^m(\tilde{\theta})$ for $\tilde{\theta}=({\theta}_{0:N}^1, {\theta}_{0:N}^2,\ldots)\in \Omega$. 
To consider a time-rescaled limit of $\{\breve{\Theta}_{0:N}^m,\, m=1,2,\ldots\}$
as $\sigma\rightarrow 0$, let $\{W_\sigma(t),0 \leq t\}$ be a continuous-time, right-continuous, piecewise 
constant process defined at its points of discontinuity by $W_\sigma(k\sigma^2)=\breve{\Theta}_N^{k+1}$ when $k$
is a nonnegative integer.
Consider the following assumptions, following \citet{ionides2015inference}:
\begin{enumerate}
	\item [(B$1$)] $\{W_\sigma(t),0 \leq t\leq 1\}$ converges weakly as $\sigma\to 0$ to a diffusion $\{W(t),0\leq t\leq 1\}$, in the space of right-continuous functions with left limits equipped with the uniform convergence topology.
	For any open set $A\subset \Theta$ with positive Lebesgue measure and $\epsilon>0$, there is a $\delta(A,\epsilon)>0$ such that $\mathbf{P}\big[W(t)\in A \mbox{ for all } \epsilon\leq t\leq 1\given W(0)\big]>\delta$.
	
	\item [(B$2$)] For some $t_0(\sigma)$ and $\sigma_0>0$, $W_\sigma(t)$ has a positive density on $\mathbb{\Theta}$, uniformly over the distribution of $W(0)$ for all $t>t_0$ and $\sigma<\sigma_0$.
	
	\item [(B$3$)] $\lik(\theta)$ is continuous in a neighborhood $\{\theta:\lik(\theta)>\lambda_1\}$ for some $\lambda_1<\sup_\vartheta\lik(\vartheta)$.
	\item [(B$4$)] There is an $\epsilon>0$ with $\epsilon^{-1}>f_{Y_n|X_n}(y_n\given x_n\,;\theta)>\epsilon$ for all $1\leq n\leq N$, $x_n\in\mathbb{X}$ and $\theta\in\mathbb{\Theta}$.
	\item [(B$5$)] There is a $C_1$ such that $f_{\Theta_n|\Theta_{n-1}}(\theta_n|\theta_{n-1}\,;\sigma)=0$ when $|\theta_n-\theta_{n-1}|>C_1\sigma$, for all $\sigma$.
	\item [(B$6$)] There is a $C_2$ such that $\sup_{1\leq n \leq N}|\theta_n-\theta_{n-1}|<C_1\,\sigma$ implies
	$|\breve\lik(\theta_{0:N})-\lik(\theta_N)| <C_2\, \sigma$, for all $\sigma$ and all $n$. 
\end{enumerate}

For the theoretical operator $T_\sigma^m$, \cite{ionides2015inference} proved that as $m$ goes to infinity, $T_\sigma^m$ converges to $g_\sigma$, for every fixed $\sigma>0$. Since there is no algorithmic approximation with this result, it also applies here. 
\begin{thm} [Theorem $1$, \cite{ionides2015inference}]
	\label{thm:ionides2015inference1}
	Let $T_\sigma$ be the operator in \eqref{recursion:if} and suppose (B$2$) and (B$4$) hold. There is a unique probability density $g_\sigma$ such that for any probability density $g$ on $\mathbb{\Theta}$,
	$$\lim_{m\rightarrow \infty}\| T_\sigma^m g-g_\sigma\|_1=0,$$
	where $\|g\|_1$ is the $L^1$ normal of $g$. 
\end{thm}
\cite{ionides2015inference} also showed that as the noise intensity $\sigma$ goes to zero, $g_\sigma$ approaches a point mass at the MLE if it exists, under (B$1$) -  (B$6$). Since this is about a theoretical operator, it can be applied here. But a slight modification on (B$5$) is needed, because in Assumption \ref{assumption} we assume that there exist $\epsilon_{\theta}(\sigma)>0$ and $\sigma\geq 0$ such that for any $v\in V$, $\theta_{n-1},\theta_n \in \mathbb{\Theta}$, and $n\geq 1$,
$$\epsilon_{\theta}(\sigma)\leq f_{\Theta_n^v|\Theta_{n-1}}(\theta_n^v\mid \theta_{n-1}\,; \sigma)\leq [\epsilon_{\theta}(\sigma)]^{-1},$$
which is a clear violation of (B$5$).
\begin{thm} 
	\label{thm:ionidesA3_modified}
	Suppose (B$1$) -  (B$4$) and (B$6$) hold. Suppose the following (B$5'$) holds:
	\begin{enumerate}
		\item [(B$5'$)] There is a $C_1$ such that $f_{\Theta_n|\Theta_{n-1}}(\theta_n|\theta_{n-1}\,;\sigma)=o(\sigma)$ when $|\theta_n-\theta_{n-1}|>C_1\sigma$, for all $\sigma$.
	\end{enumerate}
	Then, for $\lambda_2<\sup_{\vartheta}\lik(\vartheta)$,
	$$\lim_{\sigma\rightarrow 0}\int g_\sigma(\theta)\mathbbm{1}_{\{\lik(\theta)<\lambda_2 \}}d\theta=0.$$
\end{thm}
\begin{proof}
	Note that under the limit that $\sigma$ goes to zero, (B$5$) and (B$5'$) are the same. Then the proof can be finished simply by following that of Theorem $2$ in \cite{ionides2015inference}. 
\end{proof}

At last, we aim to show that when the number of particles $J$ and the number of iterations $M$ become large, the IBPF algorithm numerically approximates $g_\sigma$. Our proof (provided in Appendix \ref{appendix:block_MLE}) is similar to that of the second part of Theorem $1$ in \citep{ionides2015inference}, where the difference is caused by replacing equation [S26] therein which is a simplified form of Theorem $2$ in \citep{crisan2002survey} with a simplified form of our Theorem \ref{thm:main_theorem}.
\begin{thm} 
	\label{thm:ionidesA2_modified}
	Suppose (B$2$) and  (B$4$) hold. Let $\{\Theta_j^{M},\, j=1,\ldots,J\}$ be the output of IBPF, with $\sigma_m=\sigma>0$. There are positive finite constants $\widetilde{C}_{\alpha}$, $C_{\beta_1}$, and $C_{\beta_2}$ such that under conditions imposed in Theorem \ref{thm:main_theorem}, for any function $\breve{g}:\mathbb{\Theta}\rightarrow \mathbb{R}$, all $M$, every $n\geq 0$, $K\in \mathcal{K}$ and $\K \subseteq K$,
	$$\limsup_{M\rightarrow \infty}\mathbf{E}\left| \frac{1}{J}\sum_{j=1}^{J}\breve{g}(\Theta_j^{M})-\int \breve{g}(\theta)g_{\sigma}(\theta)d\theta \right| \leq \widetilde{C}_{\alpha} \|\breve{g}\|_{\infty}\card(\K)\left[e^{-C_{\beta_1} d(\K,\partial{K})}+\frac{e^{C_{\beta_2} |\mathcal{K}|_{\infty}}}{\sqrt{J}}\right],$$
	where $\|\breve{g}\|_{\infty}=\sup_{\theta}|\breve{g}(\theta)|$.
\end{thm}

\section{Application and performance analysis}
\label{sec:Performance}
In this section, we illustrate how IBPF can be used and compare its performances with those of IF$2$ and IEnKF, using the spatiotemporal model covered in Section \ref{sec:measles_model}, over the dataset covered in Section \ref{sec:Spatiotemporal_illustration}.
We gradually increase the parameter learning difficulties in $4$ cases, with consistent fairness in all experiments, in Section \ref{sec:Performance_analysis}. We implement IF$2$ and IEnKF through the \textit{spatPomp} package \citep{asfaw2023partially}. 

\subsection{Generalized spatiotemporal modeling for measles}
\label{sec:measles_model}
Measles is a highly contagious infectious disease caused by the measles virus; it spreads easily from one person to the next through coughs and sneezes of infected people. 
 In this section, we consider a generalized spatiotemporal model for disease transmission dynamics of measles within and between multiple cities.  
 
A compartment modeling framework for spatiotemporal population dynamics divides the population at each spatial location into compartments. Specifically, measles transmission at each location is modeled according to the SEIR model with $6$ compartments:
($S$) represents susceptible individuals who have not been infected yet but may experience infection later, ($E$) represents individuals exposed and carrying a latent infection, 
($I$) represents infectious individuals that have been infected and are infectious to others, ($R$) represents recovered individuals that are no longer infectious and are immune,
($B$) represents
the birth of individuals, and ($D$) represents
the death of individuals. 
 \cite{park2020inference} generalized the compartment model presented by \cite{he2010plug} to the spatiotemporal modeling setting, which is analyzed in other literature such as \cite{ionides2020bagged}. We further generalize that spatiotemporal compartment model by allowing the dynamics in each spatial location to have their own specific parameters, including different disease transmission speeds across locations. 
By demonstrating a methodology that scales to vertex-specific parameters, we open up new possibilities for spatiotemporal inference, though we focus here on testing statistical tools and so we do not engage directly in the scientific debates.

Different to the discrete-time based modelling and algorithm having time index $n\in \mathbb{N}$ in the previous sections, here the spatiotemporal model under consideration is a continuous-time Markov chain model with time index $t\in [0,\infty)$. Time discretization is a common and natural practice to link these two kinds of time notations. Specifically, for the continuous-time latent process $X(t)=(S(t),E(t),I(t),R(t))$, the corresponding discrete-time latent process is given by 
	$X_n=X(t_n)$
with 
$t_{0:n}$ being observation time. 
The number of individuals in compartments $S$, $E$, $I$, and $R$ of city $v$ at time $t$ are denoted by integer-valued random variables $S^v(t)$, $E^v(t)$, $I^v(t)$, and $R^v(t)$ respectively.
Denote $N_{ij}^{v}(t)$ as the counting process 
enumerating cumulative transitions from compartment $i$ to compartment $j$ where $i,j\in \{B,S,E,I,R,D\}$ and $i\neq j$, in city $v$ up to time $t$. 
We model the $40$ largest cities in the UK, ordered in decreasing size with $v=1$ corresponding to London.
Our model is described by the following system of stochastic differential equations, for $v\in \{1,\ldots,40\}$,
\begin{equation}
\label{eq:measles:system}
\begin{array}{lllllll}
\displaystyle dS^v(t) &=& dN_{BS}^{v}(t) &-& dN_{SE}^{v}(t) &-& dN_{SD}^{v}(t), \\
\displaystyle dE^v(t) &=& dN_{SE}^{v}(t) &-& dN_{EI}^{v}(t) &-& dN_{ED}^{v}(t), \\
\displaystyle dI^v(t) &=& dN_{EI}^{v}(t) &-& dN_{IR}^{v}(t) &-& dN_{ID}^{v}(t).
\end{array}
\end{equation} 
The total population 
\begin{equation}
\label{eq:population_requirement}
P^v(t)=S^v(t)+E^v(t)+I^v(t)+R^v(t)
\end{equation} 
 is calculated by smoothing census data and is treated as known.
Hence, by \eqref{eq:measles:system}, the number of recovered individuals $R^v(t)$ in city $v$ is defined implicitly.

The birth process $N_{BS}^{v}(t)$ is a time-inhomogeneous Poisson process with rate $\mu_{BS}^{v}(t)$ given by interpolated census data.
The transition processes $N_{EI}^{v}(t)$, $N_{IR}^{v}(t)$, $N_{S D}^{v}(t)$, $N_{E D}^{v}(t)$, and $N_{I D}^{v}(t)$ are modeled as conditional Poisson processes with per-capita rates $\mu_{EI}$, $\mu_{IR}$, $\mu_{S D}$, $\mu_{E D}$, and $\mu_{I D}$, respectively. Throughout this section, we consistently use an overline to indicate average across time and use a tilde to indicate average across time across cities.
The transition process $N_{SE}^{v}(t)$ is modeled as a negative binomial death process according to \cite{breto2009time} and \cite{breto2011compound}
with rate given by
\begin{eqnarray}
\label{eq:dEdt}
\mu_{SE}^{v}(t)
= 
\beta^v(t)
\left[ 
\left( \frac{I^v(t)+\widetilde{\iota}}{P^v(t)} \right)^{\overline{\alpha}^{v}}
+ \sum_{\altUnit \neq v} \frac{\overline{\theta}_{v\altUnit}}{P^v(t)} 
\left\{ 
\left(
\frac{ I^{\altUnit} (t)}{ P^{\altUnit}(t) }
\right)^{\overline{\alpha}^{\altUnit}} - 
\left(
\frac{I^v(t)}{P^v(t) } 
\right)^{\overline{\alpha}^{v}}
\right\}
\right] \frac{d\Gamma_{SE}^{v}(t)}{dt},\nonumber
\end{eqnarray}
where $\Gamma_{SE}^{v}(t)$ is a gamma process, whose mean is $t$ and variance is $\overline{\sigma}_{SE}^{v}t$ with $\overline{\sigma}_{SE}^{v}$ being the
over-dispersion parameter.
Here, $\beta^v(t)$ models seasonality driven by high contact rates between children at school, described by
\begin{equation}
\beta^v(t)=\begin{cases}
\big(1+\widetilde{\theta}_{a}(1-\widetilde{p})\widetilde{p}^{-1} \big)\, \overline{R}_0^v \mu_{IR} & \mbox{ during school term},\\
\big( 1-\widetilde{\theta}_{a}\big) \, \overline{R}_0^v \mu_{IR}& \mbox{ during vacation},
\end{cases} \label{eq:term}
\end{equation}
where $\widetilde{p}$ is the proportion of the year taken up by the school terms, $\overline{R}_0^v$ is the annual average basic reproductive ratio,  and $\widetilde{\theta}_{a}$ measures the reduction of transmission during school holidays.
In \eqref{eq:dEdt}, $\overline{\alpha}^v$ is a mixing exponent modeling inhomogeneous contact rates within the city $v$, and $\widetilde{\iota}$ models immigration of infected individuals which is appropriate when analyzing a subset of cities that cannot be treated as a closed system.
In \eqref{eq:dEdt}, the number of travelers from city $v$ to $\altUnit$ is denoted by $\overline{\theta}_{v\altUnit}$, constructed as fixed through time and symmetric between any two arbitrary cities, using the gravity model of \cite{xia2004measles}, 
$$
\overline{\theta}_{v\altUnit} = \widetilde{G} \cdot \frac{\;\widetilde{d}\;}{\widetilde{P}^2} \cdot \frac{\overline{P}^v \cdot \overline{P}^{\altUnit}}{d(v,\altUnit)},
$$
where $\widetilde{G}$ is a coupling parameter, $d(v,\altUnit)$ denotes the distance between city $v$ and city $\altUnit$, $\overline{P}^v$ is the average population for city $v$ across time, $\widetilde{p}$ is the average population across time across cities, and $\widetilde{d}$ is the average distance between a randomly chosen pair of cities. 

To complete the model specification, the measurement process for modeling the partial observability is defined as follows:
for $t_{0:\time}$ being observation time and 
\begin{equation}
\label{eqn:X_dynamic_experiment}
Z_{\time}^{v}=N_{IR}^{v}(t_\time)-N_{IR}^{v}(t_{\time-1})
\end{equation}
 being the number of removed infected individuals in the $n$th reporting interval,
suppose that they are quarantined once they are identified, so that reported counts comprise a fraction $\widetilde{\varrho}$ of these removal events;
the case report $\data{y}_{\time}^{v}$ is modeled as a realization of a conditionally Gaussian random variable $Y_{\time}^{v}$ via
\begin{equation}
\begin{split}
	\label{eq:obs_experiments}
	&\mathbf{P}\big[Y_{\time}^{v}{=}y\mid Z_{\time}^{v}{=}z\big]\\
	&= \mathcal{N}\big(y+0.5; \widetilde{\varrho} z,\widetilde{\varrho}(1-\widetilde{\varrho})z+\widetilde{\psi}^2\widetilde{\varrho}^2z^2\big)
	- \mathcal{N}\big(y-0.5; \widetilde{\varrho} z,\widetilde{\varrho}(1-\widetilde{\varrho})z+\widetilde{\psi}^2\widetilde{\varrho}^2z^2\big),
\end{split}
\end{equation}
where $\mathcal{N}(\,\cdot\,;\mu,\sigma^2)$ is the  cumulative distribution function of $\normal(\mu,\sigma^2)$ and $\widetilde{\psi}$ models overdispersion relative to the binomial distribution.

\subsection{Spatiotemporal illustration}
\label{sec:Spatiotemporal_illustration}

Figure \ref{f:illustration} shows a simulation (Plot A) from our model covered in Section \ref{sec:measles_model} and the real measles data (Plot B). We note that the spatiotemporal model considered in \cite{park2020inference} and \cite{ionides2020bagged} is a special case of ours, by taking location-specific parameters $\overline{\alpha}^{v}=1$ (in equation \eqref{eq:dEdt}), $\overline{R}_0^v=30$ (in equation \eqref{eq:term}) and $\overline{\sigma}_{SE}^{v}=0.15 \mbox{ year}^{1/2}$ the same for all locations, and take all the initializations $\theta_{S_0^v}=0.032$, $\theta_{E_0^v}=0.00005$, $\theta_{I_0^v}=0.00004$, and $\theta_{R_0^v}=1-\theta_{S_0^v}-\theta_{E_0^v}-\theta_{I_0^v}$ the same for all locations. In our simulation, for each location $v$, 
we draw the corresponding variables according to uniform distributions of the $[0.99, 1.0355]$-scaled range of those in \cite{ionides2020bagged}. That is, our $\overline{\alpha}^{v} \sim \text{Unif} [0.99\times 1, 1.0355\times 1]$ for each $v$.
We take the other parameters as fixed values as those of \cite{ionides2020bagged}: $\mu_{S D}=\mu_{E D}=\mu_{I D}=0.02 \mbox{ year}^{-1}$, $\mu_{EI}=\mu_{IR}=52$, 
$\widetilde{p} = 0.759$, $\widetilde{\varrho}=0.5$, $\widetilde{\psi}=0.15$, $\widetilde{\theta}_a=0.5$, $\widetilde{\iota}=0$,  $\widetilde{G}=400$. From Figure \ref{f:illustration}, we can see that the simulation shares the biennial pattern with most cities locked in phase most of the time. 
In both plots, each row is associated with a city, each column is associated with a date, and each pixel in a row represents Log(reported counts $+ 1$) of the epidemics.
We note that although the simulated data and real data in Figure \ref{f:illustration} are for $40$ cities, in the next subsection we gradually increase the number of cities in modeling and stop when the number of cities involved in the spatiotemporal data analysis is sufficient to clearly reveal the performances of algorithms in terms of COD. 
For example, in Figure \ref{f:case_1_2}, city number being $2$ indicates that we infer $4$ parameters for each of these $2$ cities only ($8$ parameters in total), and city number being $20$ indicates that we infer $4$ parameters for each of these $20$ cities only  ($80$ parameters in total). 
We would stop the test if the performances of algorithms (measured by log-likelihood) are sufficiently clear using spatiotemporal data of $14$ cities (Case $2$ in Figure \ref{f:case_1_2}) instead of testing up to $40$ cities.

\begin{figure}[htb!]
\centering
\includegraphics[width=0.7\textwidth,height=8cm]{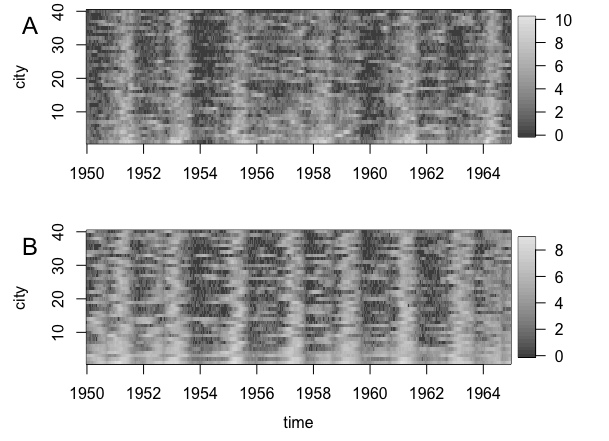}
\caption{\footnotesize Log(reported counts $+ 1$) for (A) the measles simulation from our spatiotemperal model and (B) the corresponding UK measles data.}
\label{f:illustration}
\end{figure}

\subsection{Performance analysis}
\label{sec:Performance_analysis}
We gradually increase the parameter learning difficulties in $4$ cases: in the first case, the goal is to learn initial value parameters for measles transmission dynamics for each location: $\theta_{S_0^v}$, $\theta_{E_0^v}$, $\theta_{I_0^v}$, and $\theta_{R_0^v}$; in the second case, in addition to these $4$ initial values the goal is to also learn the nonlinear parameter $\overline{R}_0^{v}$ for each location; in the third case, in addition to these $4$ initial value parameters the goal is to also learn the highly nonlinear parameter $\overline{\alpha}^{v}$ for each location; in the fourth case, the goal is to learn all the location-specific parameters, namely, $\theta_{S_0^v}$, $\theta_{E_0^v}$, $\theta_{I_0^v}$, $\theta_{R_0^v}$, $\overline{\alpha}^{v}$, $\overline{\sigma}_{SE}^{v}$, and $\overline{R}_0^{v}$. Table \ref{table:parameter_to_learn} provides an illustration of the parameters to infer for each location $v$ in four cases.
\begin{table}[t!]
	\centering
	\renewcommand{\arraystretch}{1.5}
	\begin{tabular}{|c|c|c|c|c|}
		\hline
		Parameters & Case $1$ &Case $2$  & Case $3$ & Case $4$ \\
		\hline
		$\theta_{S_0^v}$ & Yes & Yes & Yes &  Yes\\
		$\theta_{E_0^v}$ &  Yes & Yes & Yes &Yes\\
		$\theta_{I_0^v}$  &  Yes & Yes &  Yes& Yes\\
		$\theta_{R_0^v}$&  Yes &  Yes & Yes & Yes\\
		$\overline{R}_0^{v}$ &  & Yes &  & Yes \\
		$\overline{\alpha}^{v}$ &  &  & Yes &  Yes\\
		$\overline{\sigma}_{SE}^{v}$ &  &  &  &  Yes\\
		\hline
	\end{tabular}
	\caption{Parameters to infer for each location $v$ in four cases.}
	\label{table:parameter_to_learn}
\end{table}
Fairness is obtained over all experiments on all these three algorithms, as follows:
\begin{itemize}
\item Each algorithm uses the same number of iterations $M=100$ (see Algorithm \ref{Algorithm_IBPF} for notations) and the same number of particles $J=80000$.
\item We conduct $10$ replicates of all the parameter learning performance comparisons, in the way that in each replicate all algorithms start with the same initial search values drawn uniformly as follows:
$$\overline{\alpha}^{v} \sim \text{Unif}[0,2],\quad \overline{\sigma}_{SE}^{v} \sim \text{Unif}[0,1],\quad \overline{R}_0^{v} \sim \text{Unif}[25,35],$$
$$\theta_{S_0^v},\theta_{E_0^v},\theta_{I_0^v},\theta_{R_0^v}\sim \text{Unif}[0,1].$$
Here, we consider latent states of each city as portions of the population of that specific city such that $S^v(0)=\theta_{S_0^v}P^v(0)$, $E^v(0)=\theta_{E_0^v}P^v(0)$, $I^v(0)=\theta_{I_0^v}P^v(0)$, $R^v(0)=\theta_{R_0^v}P^v(0)$, and 
$$\theta_{S_0^v}+\theta_{E_0^v}+\theta_{I_0^v}+\theta_{R_0^v}=1.$$
\item Mathematically, there is only one likelihood function for the model and data in question, and different algorithms are making various approximations to estimate this quantity. The algorithms each compute the likelihood corresponding to an approximation to the exact filter distribution, and therefore they have a negative bias. The highest estimated likelihood among the available filters may therefore be anticipated to have the lowest bias. We note that this reasoning assumed that the model is correctly specified---substantial model misspecification could result in an approximate filter estimating a substantially higher likelihood than the exact filter. Thus, each algorithm has its own metric of log-likelihood: IEnKF uses the metric of EnKF on log-likelihood estimation, denoted as $\le$ whose algorithmic definition is provided in Algorithm \ref{Algorithm_IEnKF} in Appendix \ref{sec:IF2}; IF$2$ uses the metric of PF on log-likelihood estimation, denoted as $\lp$ whose algorithmic definition is provided in Algorithm \ref{Algorithm_IF2} in Appendix \ref{sec:IF2}; IBPF uses the metric of BPF on log-likelihood estimation, denoted as $\lb$ whose algorithmic definition is provided in Algorithm \ref{Algorithm_IBPF} in Section \ref{sec:Algorithm}. We evaluate the best parameters learned using each algorithm with the metrics of the other two algorithms, in all the experiments. For example, we evaluate the best parameters learned using the IEnKF algorithm through its $\le$ metric, with the other two metrics ($\lp$ and $\lb$).
\item One additional setup needed with IBPF is to set up the block sizes. In all the comparisons with IF$2$ and IEnKF, we simply allow each block in our IBPF algorithm to have exactly $2$ cities in all the experiments. That is, the first block is city $1$ and city $2$, the second block is city $3$ and city $4$, and so on. Hence, the number of blocks $\card(\mathcal{K})=\text{Number of cities}/2$. 
\end{itemize}
\begin{figure}[htb!]
	\includegraphics[width=\textwidth]{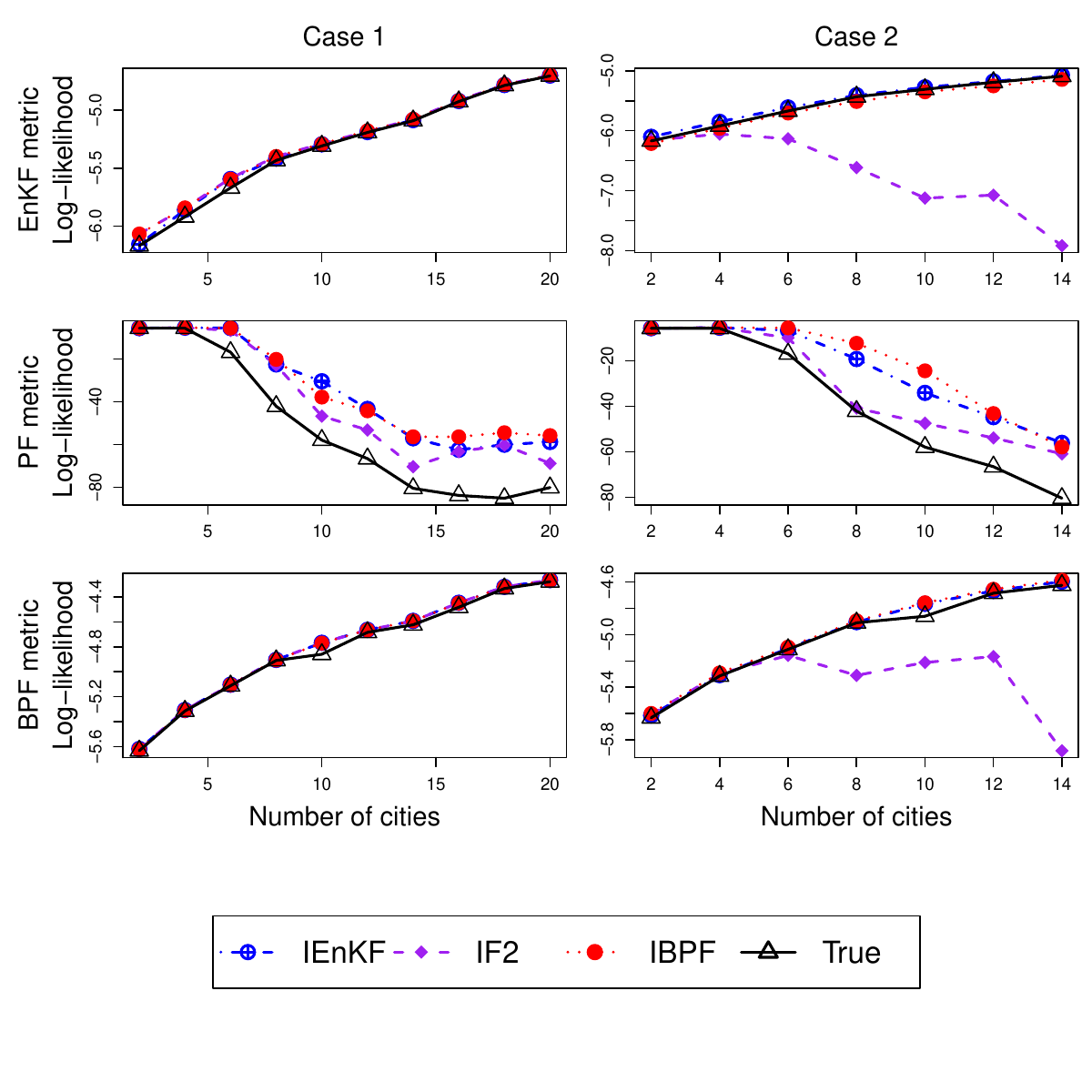}
	\caption{Log-likelihood estimates per city per time step of various dimensions for cases $1$ and $2$. Original parameter learning results are reported in Tables \ref{table:parameter_learning_results1} and \ref{table:parameter_learning_results2} in Appendix \ref{Sec:training_result_data}. Here, we divide the original results by the number of cities and then by $15\times 26$ (biweekly data in years $1950-1965$).}
	\label{f:case_1_2}
\end{figure}

\begin{figure}[htb!]
	\includegraphics[width=\textwidth]{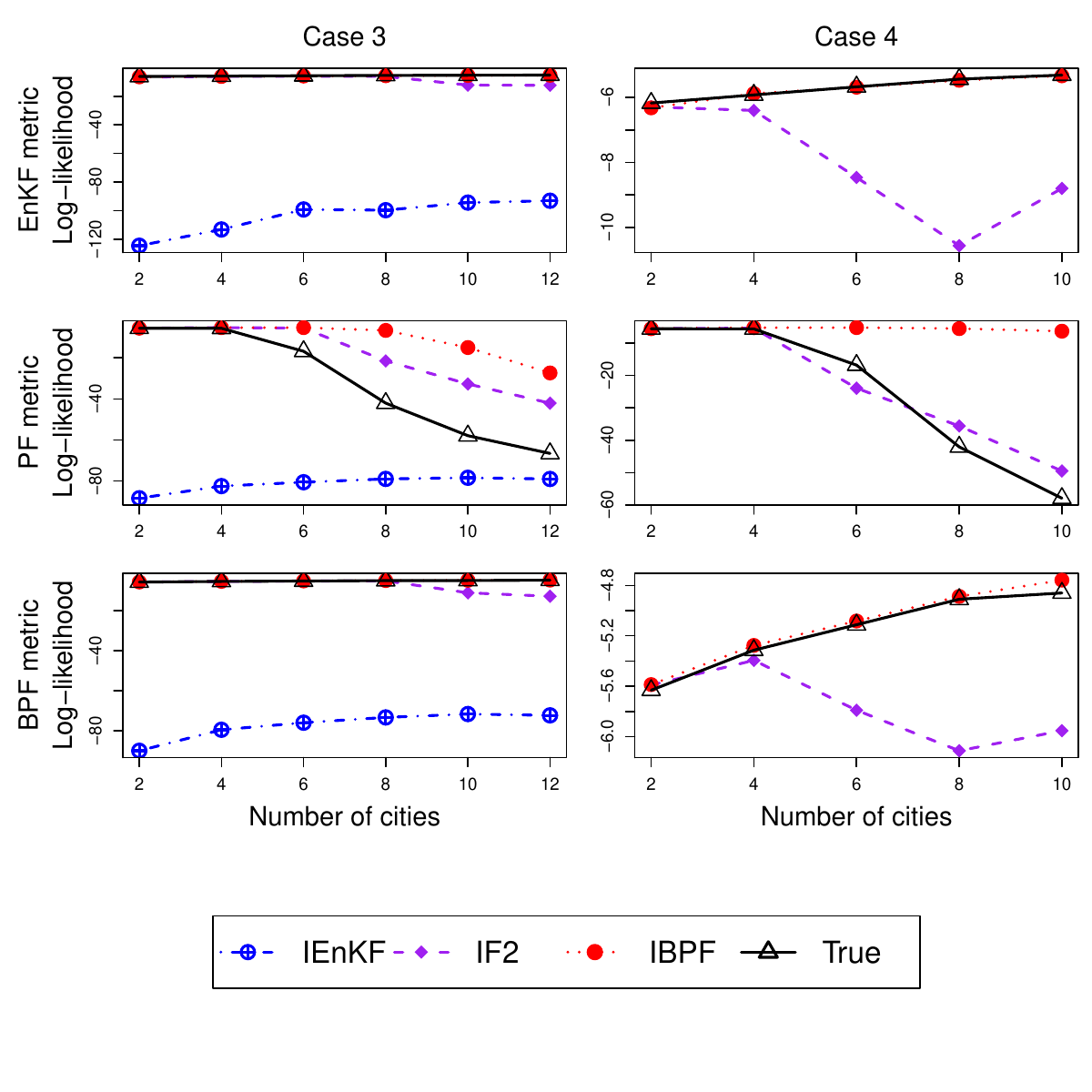}
	\caption{Log-likelihood estimates per city per time step of various dimensions for cases $3$ and $4$. Original parameter learning results are reported in Tables \ref{table:parameter_learning_results3} and \ref{table:parameter_learning_results4} in Appendix \ref{Sec:training_result_data}. Here, we divide the original results by the number of cities and then by $15\times 26$ (biweekly data in years $1950-1965$).}
	\label{f:case_3_4}
\end{figure}

Figure \ref{f:case_1_2} reports the log-likelihood estimates per city per time step of various dimensions for cases $1$ and $2$.  The corresponding original parameter learning results are reported in Tables \ref{table:parameter_learning_results1} and \ref{table:parameter_learning_results2} in Appendix \ref{Sec:training_result_data}. In Figure \ref{f:case_1_2}, the original results are divided by the number of cities and then by time steps which is
$15\times 26$ for biweekly data in years $1950-1965$. We can see in case $1$ that when we only learn the initial values ($\theta_{S_0^v}$, $\theta_{E_0^v}$, $\theta_{I_0^v}$, and $\theta_{R_0^v}$), the best parameter learning results from $10$ replicates of experiments of all three algorithms are as good as the true parameter using the EnKF metric and the BPF metric, while the best parameter learning results outperform the true parameters using the PF metric consistently. We can see that for nonlinear and non-Gaussian problems but not highly nonlinear (both case $1$ and case $2$), the IEnKF algorithm performs very well; IF$2$ scales well at least up to $20$ cities for simpler initial values learning problem in case $1$, while its performances start to drop from $4$ cities with all metrics in case $2$ which confirms the phenomenon that PF may not scale well with dimensions \citep{bengtsson2008curse,snyder2008obstacles}; our IBPF is as good as the true parameter consistently with the EnKF metric and the BPF metric, and much better than the true parameter with the PF metric consistently, in both case $1$ and case $2$. 

Figure \ref{f:case_3_4} reports the log-likelihood estimates per city per time step of various dimensions for cases $3$ and $4$, with the same transformation done as that of Figure \ref{f:case_1_2} upon the corresponding original parameter learning results that are reported in Tables \ref{table:parameter_learning_results3} and \ref{table:parameter_learning_results4} in Appendix \ref{Sec:training_result_data}. Recall that 
the goal of case $3$ is to learn the initial values ($\theta_{S_0^v}$, $\theta_{E_0^v}$, $\theta_{I_0^v}$, $\theta_{R_0^v}$) and the highly non-linear parameter $\overline{\alpha}^{v}$. In equation \eqref{eq:dEdt}, we can see that through $\overline{\alpha}^{v}$, the dynamic of one city has direct interactions with other cities. Case $4$ is the hardest of all these $4$ cases. Its goal is to learn all the location-specific parameters ($\theta_{S_0^v}$, $\theta_{E_0^v}$, $\theta_{I_0^v}$, $\theta_{R_0^v}$, $\overline{\alpha}^{v}$, $\overline{\sigma}_{SE}^{v}$, and $\overline{R}_0^{v}$). 
From Figure \ref{f:case_3_4}, we can see that,  even for $2$ cities using all these $3$ metrics, IEnKF performs very badly in case $3$ and fails completely in case $4$. The reason is that IEnKF is based on EnKF, and hence it implicitly assumes a linear Gaussian
state space model. Specifically, when new observations become
available, the ensemble is updated by a linear ``shift” based on the
assumption of a linear Gaussian state space model. However, cases $3$ and $4$ focus on highly nonlinear and non-Gaussian problems. 
This phenomenon that EnKF may not perform well for highly nonlinear and non-Gaussian problems, was observed earlier, such as in \cite{ades2015equivalent,lei2010comparison,miller1999data}. In general,
the EnKF-based parameter learning approaches are applicable to problems with a relatively
small number of parameters but more work is needed for cases
where the parameter and state are both high dimensional \citep{katzfuss2016understanding}. 
From Figure \ref{f:case_3_4}, we can see that the performance of IF$2$ drops in both cases using all three metrics. The performance of our IBPF in both cases, is as good as that of the true parameter consistently with the EnKF metric and the BPF metric, and much better than that of the true parameter with the PF metric consistently. 

\begin{table}[hbt!]
	\centering
	\renewcommand{\arraystretch}{1.5}
	\begin{tabular}{|c|c|c|c|c|c|c|c|}
		\hline
		City &$\theta_{S_0^v}$&$\theta_{E_0^v}$& $\theta_{I_0^v}$& $\theta_{R_0^v}$& $\overline{\alpha}^{v}$& $\overline{\sigma}_{SE}^{v}$& $\overline{R}_0^{v}$\\
		\hline
		1 & 0.15622 & 1e-04 & 3e-05 & 0.95509 & 0.91527 & 0.14969 & 21.2511\\
		2 & 0.0634 & 3e-05 & 4e-05 & 0.98683 & 0.85549 & 0.13765 & 10.268 \\
		3 & 0.01233 & 7e-05 & 0.00016 & 0.99165 & 0.84884 & 0.16178 & 7.39276 \\
		4 & 0.03517 & 0.00026 & 4e-05 & 0.9827 & 0.91475 & 0.15096 & 21.88148\\
		5 & 0.02691 & 3e-05 & 8e-05 & 0.80367 & 0.89105 & 0.1423 & 13.50283 \\
		6 & 0.02422 & 2e-05 & 1e-04 & 0.98636 & 0.94772 & 0.15283 & 18.7391 \\
		7 & 0.05803 & 6e-05 & 1e-05 & 0.98599 & 0.86154 & 0.15373 & 10.20071 \\
		8 & 0.00911 & 2e-04 & 7e-05 & 0.96099 & 0.85265 & 0.16587 & 8.94473 \\
		9 & 0.04624 & 2e-05 & 0.00015 & 0.97072 & 0.95923 & 0.17158 & 25.75231 \\
		10 & 0.00895 & 2e-05 & 0.00013 & 0.96844 & 0.81455 & 0.16696 & 7.08004 \\
		\hline
	\end{tabular}
	\caption{Parameter learning results in terms of parameter values for $10$ cities of case $4$. This set of parameter values gives log-likelihood $-18555$ in $\lb$ metric in Table \ref{table:parameter_learning_results4}.}
	\label{table:parameter_values}
\end{table}

\begin{table}[hbt!]
	\centering
	\renewcommand{\arraystretch}{1.5}
	\begin{tabular}{|c|c|c|c|c|c|}
		\hline
		Replicate & 2 cities &4 cities &6 cities & 8 cities & 10 cities \\
		\hline
		1 & -4358 & -8231 & -11893 & -15248 & -18555\\
		2&-4358 & -8231 & -11893 & -15248 & -18555\\
		3&-4359 & -8233 & -11894 & -15249 & -18557\\
		4&-4360 & -8233 & -11894 & -15250 & -18558\\
		5&-4361 & -8233 & -11896 & -15251 & -18559\\
		6&-4362 & -8235 & -11897 & -15252 & -18559\\
		7&-4362 & -8238 & -11900 & -15255 & -18559\\
		8&-4364 & -8239 & -11901 & -15257 & -18569\\
		9&-4364 & -8241 & -11904 & -15258 & -18570\\
		10&-4366 & -8243 & -11906 & -15263 & -18583\\
		\hline
	\end{tabular}
	\caption{Parameter learning results in terms of log-likelihood for case $4$. We conducted $10$ replicates of all the parameter learning performance comparisons. In each replicate, all algorithms start with the same initial search values drawn uniformly. We listed results ranked from highest to lowest.}
	\label{table:Robustness}
\end{table}

\begin{table}[hbt!]
	\centering
	\renewcommand{\arraystretch}{1.5}
	\begin{tabular}{|c|c|c|c|c|c|c|c|c|}
		\hline
		\multirow{2}{*}{Replicate}    & \multicolumn{4}{c|}{IEnKF} & %
		\multicolumn{4}{c|}{IF2} \\
		\cline{2-9}
		& Time & $\le$ & $\lp$ & $\lb$ &Time & $\le$ & $\lp$ & $\lb$ \\
		\hline
		1 & 15.78 & \bf{-97124} & -70730 & -70118 &  17.44 & -4970 & \bf{-4358} & -4373\\
		2 & 15.82 & -97334 & -70763 & -70099 & 17.8 & -4925 & -4359 & -4387\\
		3 & 15.79 & -97446 & -70752 & -71398 &  17.77 & -4944 & -4359 & -4388\\
		4 & 15.81 & -97582 & -70729 & -70127 &  17.15 & -4934 & -4359 & -4361\\
		5 &  15.75 & -97962 & -70099 & -70762 & 17.13 & -4904 & -4360 & -4385\\
		\hline
	\end{tabular}

	\begin{tabular}{|c|c|c|c|c|}
		\hline
		\multirow{2}{*}{Replicate}    & \multicolumn{4}{c|}{IBPF} \\
		\cline{2-5}
		& Time & $\le$ & $\lp$ & $\lb$  \\
		\hline
		1 & 18.46 & -4929 & -4380 & \bf{-4358}\\
		2  & 18.45 & -5051 & -4381 & -4358\\
		3  & 18.41 & -4948 & -4381 & -4359 \\
		4  &  18.39 & -5019 & -4399 & -4360\\
		5  & 18.4 & -4976 & -4374 & -4361\\
		\hline
	\end{tabular}
	\caption{Performance and runtime comparison of top five replicates in terms of log-likelihood for two cities in case $4$. 
		Time is measured in hours used to finish the job. 
		Three log-likelihood metrics: $\le$ representing the EnKF metric, $\lp$   representing the PF metric, and $\lb$ representing the BPF metric, were applied to the best parameters learned from IEnKF, IF$2$, and IBPF as well as the true parameter $\theta$. The highest log-likelihood values in each metric are highlighted.}
	\label{table:Performance and runtime}
\end{table}

Now we explore more properties of IBPF through experiments in case $4$. Table \ref{table:parameter_values} reports IBPF's parameters learned for $10$ cities. This is the set of parameter values that produce the maximum log-likelihood $-18555$ in the $\lb$ metric in Table \ref{table:parameter_learning_results4}. We can see that all learned parameter values conform to common sense. Table \ref{table:Robustness} reports IBPF's parameter learning results in terms of log-likelihood for all $10$ replicates conducted, where each replicate starts with an initial search value drawn uniformly. The resulting log-likelihood values are ranked from highest to lowest; for example, the value of replicate $1$ for $10$ cities in Table \ref{table:Robustness} is the log-likelihood $-18555$ in $\lb$ metric in Table \ref{table:parameter_learning_results4}. When we calculate the per city per time step values as that in Figures \ref{f:case_1_2} and \ref{f:case_3_4}, all values are almost the same across replicates. For example, the highest value of replicate $1$ gives $-18555/(26\times 15)/10=-5.947$ and the lowest value of replicate $1$ gives $-18583/(26\times 15)/10=-5.956$. Thus the parameter learning results are robust.
Figure \ref{f:likelihood_itera} reports IBPF's parameter learning results in terms of log-likelihood for iterations $\{20, 40, 60, 80, 100, 120, 160, 180\}$ with different block sizes $\{1,2,3,4\}$ for $12$ cities. We also conducted analysis with block size $6$, i.e., there are two blocks and each has $6$ cities. The log-likelihood values for block size $6$ are $\{-75793, -96487, -88973, -84087, -76356\}$ for $\{20, 40, 60, 80, 100\}$ iterations, respectively. These values are not included in Figure \ref{f:likelihood_itera}, since they are much smaller than the values with block sizes $\{1,2,3,4\}$. From Figure \ref{f:likelihood_itera}, we can see that the performance of IBPF increases as block sizes decrease. In common sense, an increase in the block size for a fixed $J$ will make the variance term worse and the bias term better in the error bound. A closer observation of the bias term, one can see that $e^{-\beta d(\K,\partial{K})}$ measures the distance to the boundary of the block, while in Figure \ref{f:likelihood_itera} we have that fixed and vary $|\mathcal{K}|_{\infty}$ which is the maximal size of one single block in the partition. That is, in Figure \ref{f:likelihood_itera}, block size being $1$ means that $|\mathcal{K}|_{\infty}=1$. Hence, Figure \ref{f:likelihood_itera} confirms our technical finding of $|\mathcal{K}|_{\infty}$. From Figure \ref{f:likelihood_itera}, we can also see that IBPF achieves likelihood maximization as iterations increase which is consistent with the analysis in Section \ref{sec:Block_MLE}. For any block size, the block particle filter computes the log-likelihood corresponding to a one-step probabilistic forecast, with the forecast corresponding to the empirical distribution of the particles.
Log-likelihood is a proper scoring rule for such forecasts
\citep{gneiting2007strictly}, meaning that the likelihood of a forecast cannot, on average, exceed that of the ideal but inaccessible Bayesian filter.
This provides a justification for comparing filters by their corresponding log-likelihoods and preferring the highest.
\begin{figure}[htb!]
	\centering
	\includegraphics[width=0.73\textwidth,height=6.1cm]{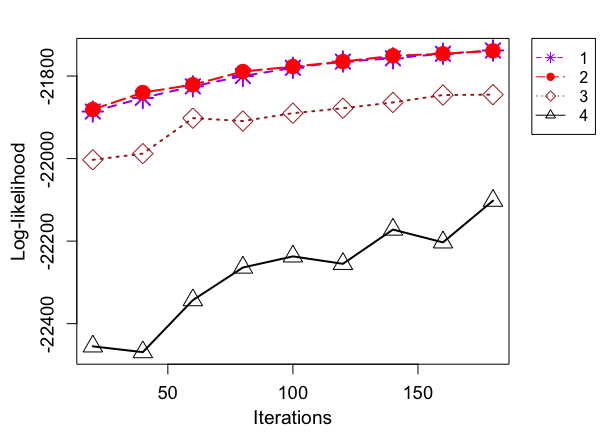}
	\caption{\footnotesize Performance analysis of IBPF for iterations $\{20, 40, 60, 80, 100, 120, 160, 180\}$ with different block sizes $\{1,2,3,4\}$ for $12$ cities of case $4$.}
	\label{f:likelihood_itera}
\end{figure}

At last, we discuss the computational cost. 
The computational cost of IF$2$ and IEnKF are $\mathcal{O}(MJN|V|)$  \citep{asfaw2023partially} and IBPF as well, where $M$ is the number of iterations, $J$ is the number of particles, $N$ is the number of time steps, and $|V|$ is the number of vertices. 
Table \ref{table:Performance and runtime} provides the performance and runtime comparison of top five replicates in terms of log-likelihood for two cities in case $4$. Each replicate is conducted with one node and $4$GB CPU on the Great Lakes Slurm cluster of the University of Michigan, Ann Arbor. We can see that the time consumed is comparable: IEnKF is in the range of $15-16$ hours, IF$2$ is in the range of $17-18$ hours, and IBPF is in the range of $18-19$ hours. Furthermore, we can see that IEnKF is still able to run but provides very bad results, and for this reason we call it ``Failed".

A demonstration on real data is beyond the scope of the current manuscript, since it requires further investigation of model misspecification and its consequences.
We refer interested readers to \cite{ionides2022iterated} for corresponding real data analysis.
The IBPF algorithm has been contributed to the \textit{spatPomp} package \citep{asfaw2023partially}  as the function  \code{ibpf}, publicly available at (\url{https://cran.r-project.org/web/packages/spatPomp/}). The source code on the measles experiments has been complied as the function  \code{he10} therein, followed with an illustration file \code{he10.R} in the ``tests" folder.
A tutorial \citep{ning2023using} publicly available at (\url{https://kidusasfaw.github.io/spatPomp/vignettes/ibpf.pdf}),  introduces \code{ibpf} and validates its correctness on a simple Gaussian example which is tractable using the Kalman filter. 

\section{Discussion and extensions}
\label{sec:Discussion_and_extensions}

In this paper, we have proposed the IBPF algorithm for high-dimensional parameter learning over partially observed, nonlinear, stochastic spatiotemporal processes, established its theoretical performance guarantees, and compared it with mainstream algorithms. 

\subsection{Discussion}
\label{sec:Discussion}
In this paper, we have only compared the EnKF-based parameter learning approach using IF. There are two main approaches to parameter estimation
within the EnKF framework without using IF. The first is called state augmentation \citep{anderson2001ensemble} which
works well in many examples but implicitly assumes that
the states and parameters jointly follow a linear Gaussian POMP
model. Hence, when some parameters violate
this assumption, the method fails completely
\citep{stroud2007sequential}. The second approach is based on
approximating likelihood functions constructed using the output
from EnKF, which estimates parameters either by maximum
likelihood or Bayesian methods. Examples of these approaches
include sequential maximum
likelihood \citep{mitchell2000adaptive}, offline
maximum
likelihood \citep{stroud2010ensemble}, and sequential Bayesian methods
\citep{stroud2007sequential,frei2012sequential}. In general,
these methods have been successful in examples with a relatively
small number of parameters, and more work is needed for cases
where the parameter and state are both high dimensional \citep{katzfuss2016understanding}. We note that our IBPF is designed to learn a large number of parameters for models with high-dimensional parameters and states. 

In recent years, approaches developed on learning high-dimensional parameters applying the Bayesian method (specifically, Markov chain Monte Carlo (MCMC)) with multidimensional time series data, including \citep{qiu2018multivariate,qiu2020multivariate,jammalamadaka2019predicting}.
We note that these literatures are limited to only working for linear models while IBPF is designed to work on general nonlinear models. Particle MCMC (PMCMC) methods, introduced by \cite{andrieu2010particle}, make use of PF to construct efficient proposals for the MCMC sampler, working for non-linear models. Particle Gibbs (PG) as a particularly widely used PMCMC algorithm,
modifies the PF step in the PMCMC algorithm to sample the latent variables conditioned on an
existing particle trajectory, resulting in what is called a conditional SMC (CSMC) step. A drawback of PG is that it can be particularly adversely affected by path degeneracy in the
CSMC step, in the way that conditioning on an existing trajectory means that whenever resampling of the trajectories results in a common ancestor who must correspond to this trajectory \citep{rainforth2016interacting}. Efforts on combating the path degeneracy effect, include but are not limited to, \citep{whiteley2010efficient, lindsten2013backward, lindsten2014particle, chopin2015particle, lindsten2015uniform}. We note that IBPF does not have the path degeneracy problem.

\subsection{Extensions}
\label{sec:extensions}
Algorithm~\ref{Algorithm_IBPF} describes the IBPF algorithm in which the initial values of the latent states are deterministically determined by model parameters, hence the theoretical treatment needs to consider only Dirac measures as initial distributions are needed. \cite{rebeschini2015can} also considered a  nonrandom initial condition which is a choice of convenience. For comments on extending the results to general initial conditions, we refer interested readers to Remark $2.3$ (page $2823$) of \cite{rebeschini2015can}.
A corresponding extension of Algorithm~\ref{Algorithm_IBPF} permits $f_{X_0}$ to be a probability density function.

At first glance, it may seem that BPF needs an approximation of spatial independence to justify the factorized distribution.
The reality is more delicate.
There is a possibility for considerable dependence to arise during the proposal stage of the filter, for which the particles follow the full joint transition density.
The block resampling then imposes a conditional independence approximation given the data, but conditional independence can be a much weaker assumption than unconditional independence.  
Recently, \citet{MIN2022108727} investigated the choice of blocks for a block particle filter, formulating the partitioning problem as a clustering problem and proposed a data-driven partitioning method based on constrained spectral clustering to automatically provide an appropriate partition.
For our measles model, we have found it surprisingly successful to have each city in its own block, and we have not found an advantage from larger block sizes.
To reason about how this might be consistent with the previous section, consider two cities (say, London and Birmingham).
Suppose they are tightly coupled dynamically, such that there will be an outbreak in London if and only if there is one in Birmingham.
Now suppose that, once an outbreak occurs, the dynamics in each city become dominated by local noise, and further that each city has a reasonably effective case reporting system.
Plausibly, the discrepancy between the actual and reported counts in London and Birmingham may be well modeled as close to independent, conditional on the reported counts.
That is enough to suggest that a block particle filter having Birmingham and London in different blocks may be successful despite the close relationship between their epidemic outcomes.

\appendix
\section{The IF$2$ algorithm and the IEnKF algorithm}
\label{sec:IF2}
In Algorithm \ref{Algorithm_IF2}, we provide the IF$2$ pseudocode in \cite{ionides2015inference}. In Algorithm \ref{Algorithm_IEnKF}, we provide the IEnKF pseudocode in \cite{asfaw2023partially}.
\begin{algorithm}[!htb]
\noindent\begin{tabular}{l}
		{\bf Input:}\rule[-1.5mm]{0mm}{6mm} \\
		Simulator for $f_{X_0}(x_0\,;\theta)$ \\
		Simulator for $f_{X_n|X_{n-1}}(x_n \mid x_{n-1}\,; \theta)$, $n\in 1{:}N$ \\
		Evaluator for $f_{Y_n|X_n}(y_n\mid x_n\,;\theta)$, $n\in 1{:}N$ \\
		Data, $y_{1:N}$ \\
		Number of iterations, $M$ \\
		Number of particles, $J$ \\
		Initial parameter swarm, $\{\Theta^0_j, \mbox{ $j\in 1{:}J$}\}$ \\
		Perturbation density, $f_{\Theta_n|\Theta_{n-1}}(\theta\mid \vartheta\,;\sigma)$, $n\in 1{:}N$\\
		Perturbation sequence, $\sigma_{1:M}$ \\
		{\bf Output:}\rule[-1.5mm]{0mm}{6mm} 
		Final parameter swarm $\{\Theta^M_j, \mbox{ $j\in 1{:}J$}\}$ and log-likelihood $\lp$\\
		For $m$ in $1\mycolon M$\rule[0mm]{0mm}{5mm}\\
		\asp     Draw $\Theta^{F,m}_{0,j}\sim h_0(\theta\mid\Theta^{m-1}_{j}\,; \sigma_m)$ for $j\in 1{:}J$\mystretch\\
		\asp     Draw $X_{0,j}^{F,m}\sim f_{X_0}(x_0 ; \Theta^{F,m}_{0,j})$ for $j\in 1{:}J$\mystretch\\
		\asp For $n$ in $1\mycolon N$\\
		\asp\asp Draw $\Theta^{P,m}_{n,j}\sim f_{\Theta_n|\Theta_{n-1}}(\theta_n\mid \Theta^{F,m}_{n-1,j}\,; \sigma_m)$ for $j\in 1{:}J$\mystretch\\
		\asp\asp   Draw $X_{n,j}^{P,m}\sim f_{X_n|X_{n-1}}(x_n \mid X^{F,m}_{n-1,j}\;; \Theta^{P,m}_{n,j})$ for $j\in 1{:}J$  \mystretch\\
		\asp\asp Compute $w_{n,j}^{m} =  f_{Y_n|X_n}(y_n\mid X_{n,j}^{P,m}\; ; \Theta^{P,m}_{n,j})$ for $j\in 1{:}J$  \mystretch\\
		\asp\asp  Draw $s_{1:J}$ with $\operatorname{Prob}(s_j=i)=  w_{n,i}^{m}\Big/\sum_{u=1}^n w_{n,u}^{m}$  \\
		\asp\asp Set  $X^{F,m}_{n,j}=X^{P,m}_{n,s_j}$ for $j\in 1{:}J$\\  
		\asp\asp  Set $\Theta^{F,m}_{n,j}=\Theta^{P,m}_{n,s_j}$ for $j\in 1{:}J$ \mystretch\\
		\asp End For\\ %
		\asp   Set $\Theta^{m}_{j}=(\Theta^{F,m}_{n,j})$ for $j\in 1{:}J$\\
		End For\\
		Set $\lp=\sum_{n=1}^N \log (\frac{1}{J}\sum_{j=1}^J w_{n,j}^{M})$
	\end{tabular}
\caption{(The IF$2$ algorithm)}
	\label{Algorithm_IF2}
\end{algorithm}

\begin{algorithm}[!htb]
	\noindent\begin{tabular}{l}
		{\bf Input:}\rule[-1.5mm]{0mm}{6mm} \\
		Simulator for $f_{X_0}(x_0\,;\theta)$ \\
		Simulator for $f_{X_n|X_{n-1}}(x_n \mid x_{n-1}\,; \theta)$, $n\in 1{:}N$ \\
		Evaluators for $e_{u,n}(x,\theta)$ and $\mathrm{v}_{u,n}(x,\theta)$, $n\in 1{:}N$ \\
		Data, $y_{1:N}$ \\
		Number of iterations, $M$ \\
		Number of particles, $J$ \\
		Initial parameter, $\theta^0$ \\
		Random walk intensities, $\sigma_{0:N,1:D_{\theta}}$ where $D_{\theta}$ is the dimension of $\theta$\\
		Cooling fraction in $50$ interations, $a$ \\
		{\bf Output:}\rule[-1.5mm]{0mm}{6mm} 
		Monte Carlo maximum likelihood estimate $\theta_M$  and log-likelihood $\le$\\
		Initialize parameters $\Theta^{F,0}_{N,j}=\theta^0$\rule[0mm]{0mm}{5mm}\\
		For $m$ in $1\mycolon M$\\
		\asp    Initialize parameters $\Theta^{F,m}_{0,j}\sim \mathcal{N}(\Theta^{F,m-1}_{N,j},a^{2m/50}\Sigma_0)$\\
		\asp\asp             where $(\Sigma_n)_{d_{\theta}, d_{\theta}'}=\sigma_{n,d_{\theta}}^2\mathbbm{1}_{d_{\theta}=d_{\theta}'}$  for $j\in 1{:}J$\mystretch\\		            
		\asp    Initialize filter particles $X_{0,j}^{F,m}\sim f_{X_0}(x_0 ; \Theta^{F,m}_{0,j})$ for $j\in 1{:}J$\mystretch\\
		\asp For $n$ in $1\mycolon N$\\
		\asp\asp Draw $\Theta^{P,m}_{n,j}\sim \mathcal{N}(\Theta^{F,m}_{n-1,j}, a^{2m/50}\Sigma_n)$ for $j\in 1{:}J$\mystretch\\
		\asp\asp   Draw $X_{n,j}^{P,m}\sim f_{X_n|X_{n-1}}(x_n \mid X^{F,m}_{n-1,j}\;; \Theta^{P,m}_{n,j})$ for $j\in 1{:}J$  \mystretch\\
		\asp\asp   Process and ensemble parameter $Z_{n,j}^{P,m}=\left( \begin{array}{c} X_{n,j}^{P,m}\\ \Theta^{P,m}_{n,j} \end{array} \right)$  for $j\in 1{:}J$ \mystretch\\
		\asp\asp   Centered process and parameter ensemble $\widetilde{Z}_{n,j}^{P,m}=Z_{n,j}^{P,m}-\frac{1}{J}\sum_{q=1}^J Z_{n,q}^{P,m}$   for $j\in 1{:}J$\mystretch\\
		\asp\asp   Forecast ensemble $\widehat{Y}_{n,j}^{u,m}=e_u(X_{n,j}^{u,P,m}, \Theta^{P,m}_{n,j})$   for $j\in 1{:}J$\mystretch\\
		\asp \asp  Centered forecast ensemble $\widetilde{Y}_{n,j}^{m}=\widehat{Y}_{n,j}^{m}-\frac{1}{J}\sum_{q=1}^J \widehat{Y}_{n,q}^{m}$  for $j\in 1{:}J$\mystretch\\
	    \asp \asp  Forecast measurement variance $R_{u,\tilde{u}}^m=\left(\frac{1}{J}\sum_{j=1}^J \mathrm{v}_u(X_{n,j}^{u, P,m},\Theta^{P,m}_{n,j})\right)_{u,\tilde{u}}$ \mystretch\\
		\asp \asp  Forecast sample covariance $\Sigma_Y^m=\frac{1}{J-1}\sum_{j=1}^J (\widetilde{Y}_{n,j}^{m})(\widetilde{Y}_{n,j}^{m})^T+R^m$ \mystretch\\
		\asp \asp  Prediction and forecast sample covarince $\Sigma_{ZY}^m=\frac{1}{J-1}\sum_{j=1}^J (\widetilde{Z}_{n,j}^{P,m})(\widetilde{Y}_{n,j}^{m})^T$ \mystretch\\
		\asp \asp  Kalman gain $K^m=\Sigma_{ZY}^m(\Sigma_Y^m)^{-1}$   \mystretch\\
		\asp \asp   Artificial measurement noise $\epsilon_{n,j}^{m}\sim \mathcal{N}(0,R)$  for $j\in 1{:}J$\mystretch\\
		\asp \asp   Errors $r_{n,j}^{m}=\widehat{Y}_{n,j}^{m}-y_n$ for $j\in 1{:}J$ \mystretch\\
		\asp \asp   Filter update $Z_{n,j}^{F,m}=\left( \begin{array}{c} X_{n,j}^{F,m}\\ \Theta^{F,m}_{n,j} \end{array} \right)=Z_{n,j}^{P,m}+K(r_{n,j}^{m}+\epsilon_{n,j}^{m})$   for $j\in 1{:}J$\mystretch\\
		\asp End For\\ %
		End For\\
		Set $\theta_{M}=\frac{1}{J}\sum_{j=1}^J (\Theta^{F,M}_{N,j})$\\
		Set $\le=\sum_{n=1}^N \log (\phi(y_n\, ;\, \frac{1}{J}\sum_{j=1}^J\widehat{Y}_{n,j}^{M}, \Sigma_Y^m))$ where $\phi$ is the normal density
	\end{tabular}
	\caption{(The IEnKF algorithm)}
	\label{Algorithm_IEnKF}
\end{algorithm}

\section{Existing results}
\label{sec:Existing_results}

The following Dobrushin comparison theorem can be seen in Theorem $3.1$ in \cite{rebeschini2015can} and Theorem $8.20$ in \cite{georgii2011gibbs}.
\begin{thm}[Dobrushin comparison theorem]
\label{thm:Dobrushin}
Let $I$ be a finite set. Let $\mathbb{S}=\prod_{i\in I}\mathbb{S}^i$ where $\mathbb{S}^i$ is a Polish space for each $i\in I$. Define the coordinate projections $X^i: x \rightarrow x^i$ for $x\in \mathbb{S}$ and $i\in I$. For probability measures $\rho$ and $\overline{\rho}$ on $\mathbb{S}$, define
$$\rho_x^i(A)=\rho(X^i\in A|X^{I\backslash \{i\}}=x^{I\backslash \{i\}}),$$
  $$\rho_{\overline{x}}^i(A)=\rho(X^i\in A|X^{I\backslash \{i\}}={\overline{x}}^{I\backslash \{i\}}),$$
$$\overline{\rho}_x^i(A)=\overline{\rho}(X^i\in A|X^{I\backslash \{i\}}=x^{I\backslash \{i\}}),$$
    $$C_{ij}=\frac{1}{2}\sup_{x,\overline{x}\in \mathbb{S}:x^{I\backslash \{j\}}=\overline{x}^{I\backslash \{j\}}}\| \rho_{x}^i- \rho_{\overline{x}}^i\| \quad\text{and}\quad b_j=\sup_{x\in \mathbb{S}}\|\rho_{x}^j-\overline{\rho}_{x}^j\|.$$
      If the Dobrushin condition 
    $$\max_{i\in I}\sum_{j\in I}C_{ij}<1,$$
      holds, then for every $J \subseteq I$,
    $$\|\rho-\overline{\rho}\|_J\leq \sum_{i\in J}\sum_{j\in I}D_{ij}b_j,$$
      where $D:=\sum_{n\geq 0}C^n<\infty.$
      \end{thm}
      
\begin{thm} [Lemma $4.1$ of \cite{rebeschini2015can}]
\label{thm:VanHandel4.1}
Let probability measures $\nu,\nu',\mathsf{F},\mathsf{F}'$ and $\epsilon>0$ be such that $\nu(A)\geq \epsilon \mathsf{F}(A)$ and $\nu'(A)\geq \epsilon \mathsf{F}'(A)$ for every measurable set $A$. Then
$$\|\nu-\nu'\|\leq 2(1-\epsilon)+\epsilon \|\mathsf{F}-\mathsf{F}'\|.$$
\end{thm} 

\begin{thm} [Lemma $4.2$ of \cite{rebeschini2015can}]
\label{thm:VanHandel4.2}
Let $\rho$ and $\rho'$ be probability measures and let $\Lambda$ be a bounded and strictly positive measurable function. 
Define $$\rho_{\Lambda}(A):=\frac{\int \mathbbm{1}_A(x)\Lambda(x)\rho(x)}{\int \Lambda(x)\rho(x)}\quad\text{and}\quad \rho'_{\Lambda}(A):=\frac{\int \mathbbm{1}_A(x)\Lambda(x)\rho'(x)}{\int \Lambda(x)\rho'(x)}.$$
Then 
$$\|\rho_{\Lambda}-\rho'_{\Lambda}\|\leq 2 \frac{\sup_x\Lambda(x)}{\inf_x\Lambda(x)}\|\rho-\rho'\|\quad\text{and}\quad \vertiii{\rho_{\Lambda}-\rho'_{\Lambda}}\leq 2 \frac{\sup_x\Lambda(x)}{\inf_x\Lambda(x)}\vertiii{\rho-\rho'}.$$
\end{thm} 

\begin{thm} [Lemma $4.3$ of \cite{rebeschini2015can}]
\label{thm:VanHandel4.3}
Let $I$ be a finite set and let $m$ be a pseudometric on $I$. Let $C=(C_{ij})_{i,j\in I}$ be a matrix with nonnegative entries. Suppose that
$$\max_{i\in I}\sum_{j\in I}e^{m(i,j)}C_{ij}\leq c<1.$$
Then the matrix $D=\sum_{n\geq 0}C^n$ satisfies
$$\max_{i\in I}\sum_{j\in I}e^{m(i,j)}D_{ij}\leq \frac{1}{1-c}.$$
\end{thm}

\begin{thm} [Lemma $4.16$ of \cite{rebeschini2015can}]
\label{thm:VanHandel4.16}
Let $$\mu=\mu^1\otimes \cdots \otimes \mu^d \quad\text{and}\quad \nu=\nu^1\otimes \cdots \otimes \nu^d$$ be product probability measures on 
$$\mathbb{S}=\mathbb{S}^1\times \cdots \times \mathbb{S}^d$$
and let $\Lambda: \mathbb{S}\rightarrow \mathbb{R}$ be a bounded and
strictly positive measurable function. Let $\mu_{\Lambda}$ and $\nu_{\Lambda}$ be probability measures
$$ \mu_{\Lambda}=\frac{\int \mathbbm{1}_A(x)\Lambda(x)\mu(dx)}{\int \Lambda(x)\mu(dx)}\quad\text{and}\quad \nu_{\Lambda}=\frac{\int \mathbbm{1}_A(x)\Lambda(x)\nu(dx)}{\int \Lambda(x)\nu(dx)}.$$
Suppose that there exists a constant $\epsilon>0$ such that the following holds: for every
$i = 1, \cdots, d$, there is a measurable function $\Lambda^i: \mathbb{S}\rightarrow \mathbb{R}$ such that
$$\epsilon \Lambda^i(x)\leq \Lambda\leq \epsilon^{-1} \Lambda^i(x),$$
for all $x\in \mathbb{S}$ such that $\Lambda^i(x)=\Lambda^i(\tilde{x})$ whenever $x^{\{1,\ldots,d\}\backslash \{i\}}=\tilde{x}^{\{1,\ldots,d\}\backslash \{i\}}$. Then
$$ \|\mu_{\Lambda}-\nu_{\Lambda}\|\leq \frac{2}{\epsilon^2}\sum_{i=1}^d\|\mu^i-\nu^i\|.$$
\end{thm}

\begin{thm} [Corollary $4.21$ of \cite{rebeschini2015can}]
\label{thm:VanHandel4.21}
For any subset of blocks $\mathcal{L}\subseteq \mathcal{K}$, we have
$$\vertiii{\bigotimes_{K\in\mathcal{L}}B^K\mu- \bigotimes_{K\in\mathcal{L}}B^K \mathsf{S}^{J} \mu}\leq \frac{4\operatorname{card}(\mathcal{L})}{\sqrt{J}},$$
for every probability measure $\mu$.
\end{thm}

\section{Preparation}
\label{sec:Preparation}
For any probability measure $\mu_{s-1}$ on $\mathbb{X}\times \mathbb{\Theta}$ at time $s-1$ for any integer $s\geq 1$, any vertex $v\in V$, any block $K'\in\mathcal{K}$, and any set $A^v=A_x^v \times A_{\theta}^v$ where $A_x^v\subseteq \mathbb{X}^v$ and $A_{\theta}^v\subseteq \mathbb{\Theta}^v$, we define the following quantities with $\chi$ indicating $(x,\theta)$: 
\begin{align}
\bullet\;&\mu_{\chi_{s-1}}^v(A_x^v \times A_{\theta}^v)\label{def:mu_s-1_v}\\
&:=\mathbf{P}^{\mu_{s-1}}\bigg[X_{s-1}^v\in A_x^v, \Theta_{s-1}^v\in A_{\theta}^v \;\big\rvert\; X_{s-1}^{V\backslash \{v\}}=x_{s-1}^{V\backslash \{v\}}, \Theta_{s-1}^{V\backslash \{v\}}=\theta_{s-1}^{V\backslash \{v\}}\bigg],\nonumber\\
\bullet\;&\mu_{\chi_{s-1},\chi_{s}}^v(A_x^v \times A_{\theta}^v)\label{def:mu_s-1_s}\\
&:=\mathbf{P}^{\mu_{s-1}}\bigg[X_{s-1}^v\in A_x^v, \Theta_{s-1}^v\in A_{\theta}^v\;\big\rvert\; X_{s-1}^{V\backslash \{v\}}=x_{s-1}^{V\backslash \{v\}}, \Theta_{s-1}^{V\backslash \{v\}}=\theta_{s-1}^{V\backslash \{v\}},\nonumber \\
                               &\hspace*{5.8cm} X_s=x_s,\Theta_s=\theta_s\bigg]\nonumber\\
&=\dfrac{\splitdfrac{
  \int\mathbbm{1}_{A^v}(x_{s-1}^v,\theta_{s-1}^v)\prod_{\omega\in N(v)}f_{X_s^{\omega}|X_{s-1}}(x_s^{\omega}\mid x_{s-1}\,;\theta_s^{\omega})}{\hspace*{2.5cm}\times f_{\Theta_s^{\omega}|\Theta_{s-1}}(\theta_s^{\omega}\mid \theta_{s-1}\,; \sigma)\mu_{\chi_{s-1}}^v(d x_{s-1}^v,d\theta_{s-1}^v)}}
{\splitdfrac{
  \int\prod_{\omega\in N(v)}f_{X_s^{\omega}|X_{s-1}}(x_s^{\omega}\mid x_{s-1}\,;\theta_s^{\omega})}{\hspace*{0.5cm}\times f_{\Theta_s^{\omega}|\Theta_{s-1}}(\theta_s^{\omega}\mid \theta_{s-1}\,; \sigma)\mu_{\chi_{s-1}}^v(d x_{s-1}^v,d\theta_{s-1}^v)}},\nonumber\\
\bullet\;&\mu_{\overline{\chi}_{s-1},\chi_s}^v(A_x^v \times A_{\theta}^v)\label{def:bar_mu_s-1_s}\\
&:=\mathbf{P}^{\mu_{s-1}}\bigg[X_{s-1}^v\in A_x^v, \Theta_{s-1}^v\in A_{\theta}^v\;\big\rvert\; X_{s-1}^{V\backslash \{v\}}=\overline{x}_{s-1}^{V\backslash \{v\}}, \Theta_{s-1}^{V\backslash \{v\}}=\overline{\theta}_{s-1}^{V\backslash \{v\}}, \nonumber\\
                               &\hspace*{5.8cm}  X_s=x_s,\Theta_s=\theta_s\bigg]\nonumber\\
&=\dfrac{\splitdfrac{
  \int\mathbbm{1}_{A^v}(\overline{x}_{s-1}^v,\overline{\theta}_{s-1}^v)\prod_{\omega\in N(v)}f_{X_s^{\omega}|X_{s-1}}(x_s^{\omega}\mid \overline{x}_{s-1}\,;\theta_s^{\omega})}{\times f_{\Theta_s^{\omega}|\Theta_{s-1}}(\theta_s^{\omega}\mid \overline{\theta}_{s-1}\,; \sigma)\mu_{\chi_{s-1}}^v(d \overline{x}_{s-1}^v,d\overline{\theta}_{s-1}^v)}}
{\splitdfrac{
  \int\prod_{\omega\in N(v)}f_{X_s^{\omega}|X_{s-1}}(x_s^{\omega}\mid \overline{x}_{s-1}\,;\theta_s^{\omega})}{\times f_{\Theta_s^{\omega}|\Theta_{s-1}}(\theta_s^{\omega}\mid \overline{\theta}_{s-1}\,; \sigma)\mu_{\chi_{s-1}}^v(d \overline{x}_{s-1}^v,d\overline{\theta}_{s-1}^v)}},\nonumber\\
\bullet\;&\mu_{\chi_{s-1},\chi_{s}}^{v,K'}(A_x^v \times A_{\theta}^v)\label{def:mu_s-1_s_K}\\
&:=\mathbf{P}^{\mu_{s-1}}\bigg[X_{s-1}^v\in A_x^v, \Theta_{s-1}^v\in A_{\theta}^v\;\big\rvert\; X_{s-1}^{V\backslash \{v\}}=x_{s-1}^{V\backslash \{v\}}, \Theta_{s-1}^{V\backslash \{v\}}=\theta_{s-1}^{V\backslash \{v\}},\nonumber\\
                               &\hspace*{5.8cm}  X_s^{K'}=x_s^{K'}, \Theta_s^{K'}=\theta_s^{K'}\bigg]\nonumber\\
&=\dfrac{\splitdfrac{
  \int\mathbbm{1}_{A^v}(x_{s-1}^v,\theta_{s-1}^v)\prod_{\omega\in N(v)\cap K'}f_{X_s^{\omega}|X_{s-1}}(x_s^{\omega}\mid x_{s-1}\,;\theta_s^{\omega})}{\times f_{\Theta_s^{\omega}|\Theta_{s-1}}(\theta_s^{\omega}\mid \theta_{s-1}\,; \sigma)\mu_{\chi_{s-1}}^v(d x_{s-1}^v,d\theta_{s-1}^v)}}
{\splitdfrac{
  \int\prod_{\omega\in N(v)\cap K'}f_{X_s^{\omega}|X_{s-1}}(x_s^{\omega}\mid x_{s-1}\,;\theta_s^{\omega})}{\times f_{\Theta_s^{\omega}|\Theta_{s-1}}(\theta_s^{\omega}\mid \theta_{s-1}\,; \sigma)\mu_{\chi_{s-1}}^v(d x_{s-1}^v,d\theta_{s-1}^v)}},\nonumber\\
\bullet\;&\mu_{\overline{\chi}_{s-1},\chi_{s}}^{v,K'}(A_x^v \times A_{\theta}^v)\label{def:bar_mu_s-1_s_K}\\
&:=\mathbf{P}^{\mu_{s-1}}\bigg[X_{s-1}^v\in A_x^v, \Theta_{s-1}^v\in A_{\theta}^v\;\big\rvert\; X_{s-1}^{V\backslash \{v\}}=\overline{x}_{s-1}^{V\backslash \{v\}}, \Theta_{s-1}^{V\backslash \{v\}}=\overline{\theta}_{s-1}^{V\backslash \{v\}}, \nonumber\\
                               &\hspace*{5.8cm} X_s^{K'}=x_s^{K'},\Theta_s^{K'}=\theta_s^{K'}\bigg]\nonumber\\
&=\dfrac{\splitdfrac{
  \int\mathbbm{1}_{A^v}(\overline{x}_{s-1}^v,\overline{\theta}_{s-1}^v)\prod_{\omega\in N(v)\cap K'}f_{X_s^{\omega}|X_{s-1}}(x_s^{\omega}\mid \overline{x}_{s-1}\,;\theta_s^{\omega})}{\times f_{\Theta_s^{\omega}|\Theta_{s-1}}(\theta_s^{\omega}\mid \overline{\theta}_{s-1}\,; \sigma)\mu_{\chi_{s-1}}^v(d \overline{x}_{s-1}^v,d\overline{\theta}_{s-1}^v)}}
{\splitdfrac{
  \int\prod_{\omega\in N(v)\cap K'}f_{X_s^{\omega}|X_{s-1}}(x_s^{\omega}\mid \overline{x}_{s-1}\,;\theta_s^{\omega})}{\times f_{\Theta_s^{\omega}|\Theta_{s-1}}(\theta_s^{\omega}\mid \overline{\theta}_{s-1}\,; \sigma)\mu_{\chi_{s-1}}^v(d \overline{x}_{s-1}^v,d\overline{\theta}_{s-1}^v)}}.\nonumber
\end{align}
Based on the above quantities, with $v'\in V$ we define
\begin{align}
\bullet\;&\label{eqn:C_vv'_mu}
C_{vv'}^{\mu_{s-1}}:=\frac{1}{2}\sup_{x_s\in \mathbb{X} \atop \theta_s\in \mathbb{\Theta}}\sup_{x_{s-1},\overline{x}_{s-1}\in \mathbb{X}:x_{s-1}^{V\backslash \{v'\}}=\overline{x}_{s-1}^{V\backslash \{v'\}}\atop
\theta_{s-1},\overline{\theta}_{s-1}\in \mathbb{\Theta}:\theta_{s-1}^{V\backslash \{v'\}}=\overline{\theta}_{s-1}^{V\backslash \{v'\}}
}\| \mu_{\chi_{s-1},\chi_{s}}^v-\mu_{\overline{\chi}_{s-1},\chi_s}^v\|,\\
\bullet\;&\label{eqn:tilde_C_vv'_mu}
  \widetilde{C}_{vv'}^{\mu_{s-1}}:=\frac{1}{2}\sup_{K'\in\mathcal{K}}\sup_{x_s\in \mathbb{X} \atop \theta_s\in \mathbb{\Theta}}\sup_{x_{s-1},\overline{x}_{s-1}\in \mathbb{X}:x_{s-1}^{V\backslash \{v'\}}=\overline{x}_{s-1}^{V\backslash \{v'\}}\atop
\theta_{s-1},\overline{\theta}_{s-1}\in \mathbb{\Theta}:\theta_{s-1}^{V\backslash \{v'\}}=\overline{\theta}_{s-1}^{V\backslash \{v'\}}
}\| \mu_{\chi_{s-1},\chi_{s}}^{v,K'}-\mu_{\overline{\chi}_{s-1},\chi_{s}}^{v,K'}\|,
\end{align}
and with $\beta$ being a finite positive constant we further define
\begin{align}
\label{eqn:corr_definition}
\bullet\;&\Corr(\mu_{s-1},\beta):=\max_{v\in V}\sum_{v'\in V} e^{\beta d(v,v')}C_{vv'}^{\mu_{s-1}}.\\
\bullet\;&\label{eqn:tilde_corr_definition}
\widetilde{\Corr}(\mu_{s-1},\beta):=\max_{v\in V}\sum_{v'\in V} e^{\beta d(v,v')}\widetilde{C}_{vv'}^{\mu_{s-1}}.
\end{align}

Then, for $\mathsf{F}_s$ defined in \eqref{eqn:pi_recursion}
and $\widetilde{\mathsf{F}}_s$ defined in \eqref{eqn:widetilde_pi_recursion}, 
any probability measure $\nu_{s-1}$ on $\mathbb{X}\times \mathbb{\Theta}$ at time $s-1$ for any integer $s\geq 1$,  and any set $A \subseteq \mathbb{X}\times \mathbb{\Theta}$, 
we have
\begin{align}
\bullet\;&(\mathsf{F}_s\nu_{s-1})(A)\label{def:F_s_rho}\\
&=\dfrac{\splitdfrac{
  \int\mathbbm{1}_A(x_s,\theta_s)\prod_{\omega\in V}f_{X_s^{\omega}|X_{s-1}}(x_s^{\omega}\mid x_{s-1}\,;\theta_s^{\omega}) f_{\Theta_s^{\omega}|\Theta_{s-1}}(\theta_s^{\omega}\mid \theta_{s-1}\,; \sigma)}{\times f_{Y_s^{\omega}|X_s^{\omega}}(Y_s^{\omega}\mid x_s^{\omega}\,;\theta_s^{\omega})\nu_{s-1}(dx_{s-1},d\theta_{s-1})\psi(dx_s)\lambda(d\theta_s)}}
{\splitdfrac{
  \int \prod_{\omega\in V}f_{X_s^{\omega}|X_{s-1}}(x_s^{\omega}\mid x_{s-1}\,;\theta_s^{\omega}) f_{\Theta_s^{\omega}|\Theta_{s-1}}(\theta_s^{\omega}\mid \theta_{s-1}\,; \sigma)}{\times f_{Y_s^{\omega}|X_s^{\omega}}(Y_s^{\omega}\mid x_s^{\omega}\,;\theta_s^{\omega})\nu_{s-1}(dx_{s-1},d\theta_{s-1})\psi(dx_s)\lambda(d\theta_s)}},\nonumber\\
\bullet\;&(\widetilde{\mathsf{F}}_s\nu_{s-1})(A)\label{def:widetildeF_s_rho}\\
&=\dfrac{\splitdfrac{
  \int\mathbbm{1}_A(x_s,\theta_s)\prod_{K\in \mathcal{K}}\big[\int\prod_{{\omega}\in K}f_{X_s^{\omega}|X_{s-1}}(x_s^{\omega}\mid x_{s-1}\,;\theta_s^{\omega})}{\splitdfrac{\times f_{\Theta_s^{\omega}|\Theta_{s-1}}(\theta_s^{\omega}\mid \theta_{s-1}\,; \sigma) f_{Y_s^{\omega}|X_s^{\omega}}(Y_s^{\omega}\mid x_s^{\omega}\,;\theta_s^{\omega})\nu_{s-1}(dx_{s-1},d\theta_{s-1})\big]}{\times\psi(dx_s)\lambda(d\theta_s)}}}
{\splitdfrac{
  \int\prod_{K\in \mathcal{K}}\big[\int\prod_{{\omega}\in K}f_{X_s^{\omega}|X_{s-1}}(x_s^{\omega}\mid x_{s-1}\,;\theta_s^{\omega})}{\splitdfrac{\times f_{\Theta_s^{\omega}|\Theta_{s-1}}(\theta_s^{\omega}\mid \theta_{s-1}\,; \sigma) f_{Y_s^{\omega}|X_s^{\omega}}(Y_s^{\omega}\mid x_s^{\omega}\,;\theta_s^{\omega})\nu_{s-1}(dx_{s-1},d\theta_{s-1})\big]}{\times\psi(dx_s)\lambda(d\theta_s)}}}.\nonumber
\end{align}
According to the definition of $\mu_{\chi_{s-1}}^v$ given in \eqref{def:mu_s-1_v}, for any $K\in \mathcal{K}$ and $v\in K$,
we have
\begin{align}
\bullet\;&(\mathsf{F}_s\nu_{s-1})_{\chi_s}^v(A^v)\label{def:F_s_rho}\\
&=\dfrac{\splitdfrac{
  \int\mathbbm{1}_{A^v}(x_s^v,\theta_s^v)\prod_{\omega\in V}f_{X_s^{\omega}|X_{s-1}}(x_s^{\omega}\mid x_{s-1}\,;\theta_s^{\omega}) f_{\Theta_s^{\omega}|\Theta_{s-1}}(\theta_s^{\omega}\mid \theta_{s-1}\,; \sigma)}{\times f_{Y_s^{v}|X_s^{v}}(Y_s^{v}\mid x_s^{v}\,;\theta_s^{v})\nu_{s-1}(dx_{s-1},d\theta_{s-1})\psi^v(dx_s)\lambda^v(d\theta_s)}}
{\splitdfrac{
  \int \prod_{\omega\in V}f_{X_s^{\omega}|X_{s-1}}(x_s^{\omega}\mid x_{s-1}\,;\theta_s^{\omega}) f_{\Theta_s^{\omega}|\Theta_{s-1}}(\theta_s^{\omega}\mid \theta_{s-1}\,; \sigma)}{\times f_{Y_s^{v}|X_s^{v}}(Y_s^{v}\mid x_s^{v}\,;\theta_s^{v})\nu_{s-1}(dx_{s-1},d\theta_{s-1})\psi^v(dx_s)\lambda^v(d\theta_s)}},\nonumber\\
\bullet\;&(\widetilde{\mathsf{F}}_s\nu_{s-1})_{\chi_s}^v(A^v)\label{def:widetildeF_s_rho}\\
&=\dfrac{\splitdfrac{
  \int\mathbbm{1}_{A^v}(x_s^v,\theta_s^v)\prod_{\omega\in K}f_{X_s^{\omega}|X_{s-1}}(x_s^{\omega}\mid x_{s-1}\,;\theta_s^{\omega}) f_{\Theta_s^{\omega}|\Theta_{s-1}}(\theta_s^{\omega}\mid \theta_{s-1}\,; \sigma)}{\times f_{Y_s^{v}|X_s^{v}}(Y_s^{v}\mid x_s^{v}\,;\theta_s^{v})\nu_{s-1}(dx_{s-1},d\theta_{s-1})\psi^v(dx_s)\lambda^v(d\theta_s)}}
{\splitdfrac{
  \int \prod_{\omega\in K}f_{X_s^{\omega}|X_{s-1}}(x_s^{\omega}\mid x_{s-1}\,;\theta_s^{\omega}) f_{\Theta_s^{\omega}|\Theta_{s-1}}(\theta_s^{\omega}\mid \theta_{s-1}\,; \sigma)}{\times f_{Y_s^{v}|X_s^{v}}(Y_s^{v}\mid x_s^{v}\,;\theta_s^{v})\nu_{s-1}(dx_{s-1},d\theta_{s-1})\psi^v(dx_s)\lambda^v(d\theta_s)}}.\nonumber
\end{align}
By \eqref{def:F_s_rho} and \eqref{def:widetildeF_s_rho}, according to the definition of $\mu_{\chi_{s-1},\chi_{s}}^v(A_x^v \times A_{\theta}^v)$ given in \eqref{def:mu_s-1_s}, for any $K\in \mathcal{K}$ and $v\in K$,
we have
\begin{align}
\bullet\;&(\mathsf{F}_s\nu_{s-1})_{\chi_s,\chi_{s+1}}^{v}(A^v)\label{def:F_s_rho_s1}\\
&=\dfrac{\splitdfrac{
  \int\mathbbm{1}_{A^v}(x_s^v,\theta_s^v)\prod_{\omega\in V}f_{X_s^{\omega}|X_{s-1}}(x_s^{\omega}\mid x_{s-1}\,;\theta_s^{\omega})f_{\Theta_s^{\omega}|\Theta_{s-1}}(\theta_s^{\omega}\mid \theta_{s-1}\,; \sigma)}{\splitdfrac{\times f_{Y_s^v|X_s^v}(Y_s^{v}\mid x_s^{v}\,;\theta_s^{v})\prod_{u\in N(v)}f_{X_{s+1}^u|X_s}(x_{s+1}^u\mid x_s\,;\theta_{s+1}^u)}{\times f_{\Theta_{s+1}^{u}|\Theta_{s}}(\theta_{s+1}^{u}\mid \theta_{s}\,; \sigma)\psi^v(dx_s^v)\lambda^v(d\theta_s^v)\nu_{s-1}(dx_{s-1},d\theta_{s-1})}}}
{\splitdfrac{
  \int\prod_{\omega\in V}f_{X_s^{\omega}|X_{s-1}}(x_s^{\omega}\mid x_{s-1}\,;\theta_s^{\omega})f_{\Theta_s^{\omega}|\Theta_{s-1}}(\theta_s^{\omega}\mid \theta_{s-1}\,; \sigma)}{\splitdfrac{\times f_{Y_s^v|X_s^v}(Y_s^{v}\mid x_s^{v}\,;\theta_s^{v})\prod_{u\in N(v)}f_{X_{s+1}^u|X_s}(x_{s+1}^u\mid x_s\,;\theta_{s+1}^u)}{\times f_{\Theta_{s+1}^{u}|\Theta_{s}}(\theta_{s+1}^{u}\mid \theta_{s}\,; \sigma)\psi^v(dx_s^v)\lambda^v(d\theta_s^v)\nu_{s-1}(dx_{s-1},d\theta_{s-1})}}},\nonumber\\
\bullet\;&(\widetilde{\mathsf{F}}_s\nu_{s-1})_{\chi_s,\chi_{s+1}}^{v}(A^v)\label{def:widetildeF_s_rho_s1}\\
&=\dfrac{\splitdfrac{
  \int\mathbbm{1}_{A^v}(x_s^v,\theta_s^v)\prod_{\omega\in K}f_{X_s^{\omega}|X_{s-1}}(x_s^{\omega}\mid x_{s-1}\,;\theta_s^{\omega})f_{\Theta_s^{\omega}|\Theta_{s-1}}(\theta_s^{\omega}\mid \theta_{s-1}\,; \sigma)}{\splitdfrac{\times f_{Y_s^v|X_s^v}(Y_s^{v}\mid x_s^{v}\,;\theta_s^{v})\prod_{u\in N(v)}f_{X_{s+1}^u|X_s}(x_{s+1}^u\mid x_s\,;\theta_{s+1}^u)}{\times f_{\Theta_{s+1}^{u}|\Theta_{s}}(\theta_{s+1}^{u}\mid \theta_{s}\,; \sigma)\psi^v(dx_s^v)\lambda^v(d\theta_s^v)\nu_{s-1}(dx_{s-1},d\theta_{s-1})}}}
{\splitdfrac{
  \int\prod_{\omega\in K}f_{X_s^{\omega}|X_{s-1}}(x_s^{\omega}\mid x_{s-1}\,;\theta_s^{\omega})f_{\Theta_s^{\omega}|\Theta_{s-1}}(\theta_s^{\omega}\mid \theta_{s-1}\,; \sigma)}{\splitdfrac{\times f_{Y_s^v|X_s^v}(Y_s^{v}\mid x_s^{v}\,;\theta_s^{v})\prod_{u\in N(v)}f_{X_{s+1}^u|X_s}(x_{s+1}^u\mid x_s\,;\theta_{s+1}^u)}{\times f_{\Theta_{s+1}^{u}|\Theta_{s}}(\theta_{s+1}^{u}\mid \theta_{s}\,; \sigma)\psi^v(dx_s^v)\lambda^v(d\theta_s^v)\nu_{s-1}(dx_{s-1},d\theta_{s-1})}}}.\nonumber
\end{align}
By \eqref{def:F_s_rho} and \eqref{def:widetildeF_s_rho}, according to the definition of $\mu_{\overline{\chi}_{s-1},\chi_s}^v$ given in \eqref{def:bar_mu_s-1_s}, for any $K\in \mathcal{K}$ and $v\in K$,
we have
\begin{align}
\bullet\;&(\mathsf{F}_s\nu_{s-1})_{\overline{\chi}_s,\chi_{s+1}}^{v}(A^v)\label{def:F_s_rho_s1_overline}\\
&=\dfrac{\splitdfrac{
  \int\mathbbm{1}_{A^v}(\overline{x}_s^v,\overline{\theta}_s^v)\prod_{\omega\in V}f_{X_s^{\omega}|X_{s-1}}(\overline{x}_s^{\omega}\mid x_{s-1}\,;\overline{\theta}_s^{\omega})f_{\Theta_s^{\omega}|\Theta_{s-1}}(\overline{\theta}_s^{\omega}\mid \theta_{s-1}\,; \sigma)}{\splitdfrac{\times f_{Y_s^v|X_s^v}(Y_s^{v}\mid \overline{x}_s^{v}\,;\overline{\theta}_s^{v})\prod_{u\in N(v)}f_{X_{s+1}^u|X_s}(x_{s+1}^u\mid \overline{x}_s\,;\theta_{s+1}^u)}{\times f_{\Theta_{s+1}^{u}|\Theta_s}(\theta_{s+1}^{u}\mid \overline{\theta}_s\,; \sigma)\psi^v(d\overline{x}_s^v)\lambda^v(d\overline{\theta}_s^v)\nu_{s-1}(dx_{s-1},d\theta_{s-1})}}}
{\splitdfrac{
  \int\prod_{\omega\in V}f_{X_s^{\omega}|X_{s-1}}(\overline{x}_s^{\omega}\mid x_{s-1}\,;\overline{\theta}_s^{\omega})f_{\Theta_s^{\omega}|\Theta_{s-1}}(\overline{\theta}_s^{\omega}\mid \theta_{s-1}\,; \sigma)}{\splitdfrac{\times f_{Y_s^v|X_s^v}(Y_s^{v}\mid \overline{x}_s^{v}\,;\overline{\theta}_s^{v})\prod_{u\in N(v)}f_{X_{s+1}^u|X_s}(x_{s+1}^u\mid \overline{x}_s\,;\theta_{s+1}^u)}{\times f_{\Theta_{s+1}^{u}|\Theta_s}(\theta_{s+1}^{u}\mid \overline{\theta}_s\,; \sigma)\psi^v(d\overline{x}_s^v)\lambda^v(d\overline{\theta}_s^v)\nu_{s-1}(dx_{s-1},d\theta_{s-1})}}},\nonumber\\
\bullet\;&(\widetilde{\mathsf{F}}_s\nu_{s-1})_{\overline{\chi}_s,\chi_{s+1}}^{v}(A^v)\label{def:widetildeF_s_rho_s1_overline}\\
&=\dfrac{\splitdfrac{
  \int\mathbbm{1}_{A^v}(\overline{x}_s^v,\overline{\theta}_s^v)\prod_{\omega\in K}f_{X_s^{\omega}|X_{s-1}}(\overline{x}_s^{\omega}\mid x_{s-1}\,;\overline{\theta}_s^{\omega})f_{\Theta_s^{\omega}|\Theta_{s-1}}(\overline{\theta}_s^{\omega}\mid \theta_{s-1}\,; \sigma)}{\splitdfrac{\times f_{Y_s^v|X_s^v}(Y_s^{v}\mid \overline{x}_s^{v}\,;\overline{\theta}_s^{v})\prod_{u\in N(v)}f_{X_{s+1}^u|X_s}(x_{s+1}^u\mid \overline{x}_s\,;\theta_{s+1}^u)}{\times f_{\Theta_{s+1}^{u}|\Theta_s}(\theta_{s+1}^{u}\mid \overline{\theta}_s\,; \sigma)\psi^v(d\overline{x}_s^v)\lambda^v(d\overline{\theta}_s^v)\nu_{s-1}(dx_{s-1},d\theta_{s-1})}}}
{\splitdfrac{
  \int\prod_{\omega\in K}f_{X_s^{\omega}|X_{s-1}}(\overline{x}_s^{\omega}\mid x_{s-1}\,;\overline{\theta}_s^{\omega})f_{\Theta_s^{\omega}|\Theta_{s-1}}(\overline{\theta}_s^{\omega}\mid \theta_{s-1}\,; \sigma)}{\splitdfrac{\times f_{Y_s^v|X_s^v}(Y_s^{v}\mid \overline{x}_s^{v}\,;\overline{\theta}_s^{v})\prod_{u\in N(v)}f_{X_{s+1}^u|X_s}(x_{s+1}^u\mid \overline{x}_s\,;\theta_{s+1}^u)}{\times f_{\Theta_{s+1}^{u}|\Theta_s}(\theta_{s+1}^{u}\mid \overline{\theta}_s\,; \sigma)\psi^v(d\overline{x}_s^v)\lambda^v(d\overline{\theta}_s^v)\nu_{s-1}(dx_{s-1},d\theta_{s-1})}}}.\nonumber
\end{align}
By \eqref{def:F_s_rho} and \eqref{def:widetildeF_s_rho}, according to the definition of $\mu_{\chi_{s-1},\chi_{s}}^{v,K'}$ given in \eqref{def:mu_s-1_s_K}, for any $K, K'\in \mathcal{K}$ and $v\in K$,
we have
\begin{align}
\bullet\;&(\mathsf{F}_s\nu_{s-1})_{\chi_s,\chi_{s+1}}^{v,K'}(A^v)\label{def:F_s_rho_s1_K}\\
&=\dfrac{\splitdfrac{
  \int\mathbbm{1}_{A^v}(x_s^v,\theta_s^v)\prod_{\omega\in V}f_{X_s^{\omega}|X_{s-1}}(x_s^{\omega}\mid x_{s-1}\,;\theta_s^{\omega})f_{\Theta_s^{\omega}|\Theta_{s-1}}(\theta_s^{\omega}\mid \theta_{s-1}\,; \sigma)}{\splitdfrac{\times f_{Y_s^v|X_s^v}(Y_s^{v}\mid x_s^{v}\,;\theta_s^{v})\prod_{u\in N(v)\cap K'}f_{X_{s+1}^u|X_s}(x_{s+1}^u\mid x_s\,;\theta_{s+1}^u)}{\times f_{\Theta_{s+1}^{u}|\Theta_{s}}(\theta_{s+1}^{u}\mid \theta_{s}\,; \sigma)\psi^v(dx_s^v)\lambda^v(d\theta_s^v)\nu_{s-1}(dx_{s-1},d\theta_{s-1})}}}
{\splitdfrac{
  \int\prod_{\omega\in V}f_{X_s^{\omega}|X_{s-1}}(x_s^{\omega}\mid x_{s-1}\,;\theta_s^{\omega})f_{\Theta_s^{\omega}|\Theta_{s-1}}(\theta_s^{\omega}\mid \theta_{s-1}\,; \sigma)}{\splitdfrac{\times f_{Y_s^v|X_s^v}(Y_s^{v}\mid x_s^{v}\,;\theta_s^{v})\prod_{u\in N(v)\cap K'}f_{X_{s+1}^u|X_s}(x_{s+1}^u\mid x_s\,;\theta_{s+1}^u)}{\times f_{\Theta_{s+1}^{u}|\Theta_{s}}(\theta_{s+1}^{u}\mid \theta_{s}\,; \sigma)\psi^v(dx_s^v)\lambda^v(d\theta_s^v)\nu_{s-1}(dx_{s-1},d\theta_{s-1})}}},\nonumber\\
\bullet\;&(\widetilde{\mathsf{F}}_s\nu_{s-1})_{\chi_s,\chi_{s+1}}^{v,K'}(A^v)\label{def:widetildeF_s_rho_s1_K}\\    
    &=\dfrac{\splitdfrac{
      \int\mathbbm{1}_{A^v}(x_s^v,\theta_s^v)\prod_{\omega\in K}f_{X_s^{\omega}|X_{s-1}}(x_s^{\omega}\mid x_{s-1}\,;\theta_s^{\omega})f_{\Theta_s^{\omega}|\Theta_{s-1}}(\theta_s^{\omega}\mid \theta_{s-1}\,; \sigma)}{\splitdfrac{\times f_{Y_s^v|X_s^v}(Y_s^{v}\mid x_s^{v}\,;\theta_s^{v})\prod_{u\in N(v)\cap K'}f_{X_{s+1}^u|X_s}(x_{s+1}^u\mid x_s\,;\theta_{s+1}^u)}{\times f_{\Theta_{s+1}^{u}|\Theta_{s}}(\theta_{s+1}^{u}\mid \theta_{s}\,; \sigma)\psi^v(dx_s^v)\lambda^v(d\theta_s^v)\nu_{s-1}(dx_{s-1},d\theta_{s-1})}}}
{\splitdfrac{
  \int\prod_{\omega\in K}f_{X_s^{\omega}|X_{s-1}}(x_s^{\omega}\mid x_{s-1}\,;\theta_s^{\omega})f_{\Theta_s^{\omega}|\Theta_{s-1}}(\theta_s^{\omega}\mid \theta_{s-1}\,; \sigma)}{\splitdfrac{\times f_{Y_s^v|X_s^v}(Y_s^{v}\mid x_s^{v}\,;\theta_s^{v})\prod_{u\in N(v)\cap K'}f_{X_{s+1}^u|X_s}(x_{s+1}^u\mid x_s\,;\theta_{s+1}^u)}{\times f_{\Theta_{s+1}^{u}|\Theta_{s}}(\theta_{s+1}^{u}\mid \theta_{s}\,; \sigma)\psi^v(dx_s^v)\lambda^v(d\theta_s^v)\nu_{s-1}(dx_{s-1},d\theta_{s-1})}}}.\nonumber
    \end{align}
By \eqref{def:F_s_rho} and \eqref{def:widetildeF_s_rho}, according to the definition of $\mu_{\overline{\chi}_{s-1},\chi_{s}}^{v,K'}$ given in \eqref{def:bar_mu_s-1_s_K}, for any $K,K'\in \mathcal{K}$ and $v\in K$,
    we have  
\begin{align}
\bullet\;&(\mathsf{F}_s\nu_{s-1})_{\overline{\chi}_s,\chi_{s+1}}^{v,K'}(A^v)\label{def:F_s_rho_s1_K_overline}\\
&=\dfrac{\splitdfrac{
\int\mathbbm{1}_{A^v}(\overline{x}_s^v,\overline{\theta}_s^v)\prod_{\omega\in V}f_{X_s^{\omega}|X_{s-1}}(\overline{x}_s^{\omega}\mid x_{s-1}\,;\overline{\theta}_s^{\omega})f_{\Theta_s^{\omega}|\Theta_{s-1}}(\overline{\theta}_s^{\omega}\mid \theta_{s-1}\,; \sigma)}{\splitdfrac{\times f_{Y_s^v|X_s^v}(Y_s^{v}\mid \overline{x}_s^{v}\,;\overline{\theta}_s^{v})\prod_{u\in N(v)\cap K'}f_{X_{s+1}^u|X_s}(x_{s+1}^u\mid \overline{x}_s\,;\theta_{s+1}^u)}{\times f_{\Theta_{s+1}^{u}|\Theta_s}(\theta_{s+1}^{u}\mid \overline{\theta}_s\,; \sigma)\psi^v(d\overline{x}_s^v)\lambda^v(d\overline{\theta}_s^v)\nu_{s-1}(dx_{s-1},d\theta_{s-1})}}}
{\splitdfrac{
  \int\prod_{\omega\in V}f_{X_s^{\omega}|X_{s-1}}(\overline{x}_s^{\omega}\mid x_{s-1}\,;\overline{\theta}_s^{\omega})f_{\Theta_s^{\omega}|\Theta_{s-1}}(\overline{\theta}_s^{\omega}\mid \theta_{s-1}\,; \sigma)}{\splitdfrac{\times f_{Y_s^v|X_s^v}(Y_s^{v}\mid \overline{x}_s^{v}\,;\overline{\theta}_s^{v})\prod_{u\in N(v)\cap K'}f_{X_{s+1}^u|X_s}(x_{s+1}^u\mid \overline{x}_s\,;\theta_{s+1}^u)}{\times f_{\Theta_{s+1}^{u}|\Theta_s}(\theta_{s+1}^{u}\mid \overline{\theta}_s\,; \sigma)\psi^v(d\overline{x}_s^v)\lambda^v(d\overline{\theta}_s^v)\nu_{s-1}(dx_{s-1},d\theta_{s-1})}}},\nonumber\\
\bullet\;&(\widetilde{\mathsf{F}}_s\nu_{s-1})_{\overline{\chi}_s,\chi_{s+1}}^{v,K'}(A^v)\label{def:widetildeF_s_rho_s1_K_overline}\\    
    &=\dfrac{\splitdfrac{
      \int\mathbbm{1}_{A^v}(\overline{x}_s^v,\overline{\theta}_s^v)\prod_{\omega\in K}f_{X_s^{\omega}|X_{s-1}}(\overline{x}_s^{\omega}\mid x_{s-1}\,;\overline{\theta}_s^{\omega})f_{\Theta_s^{\omega}|\Theta_{s-1}}(\overline{\theta}_s^{\omega}\mid \theta_{s-1}\,; \sigma)}{\splitdfrac{\times f_{Y_s^v|X_s^v}(Y_s^{v}\mid \overline{x}_s^{v}\,;\overline{\theta}_s^{v})\prod_{u\in N(v)\cap K'}f_{X_{s+1}^u|X_s}(x_{s+1}^u\mid \overline{x}_s\,;\theta_{s+1}^u)}{\times f_{\Theta_{s+1}^{u}|\Theta_s}(\theta_{s+1}^{u}\mid \overline{\theta}_s\,; \sigma)\psi^v(d\overline{x}_s^v)\lambda^v(d\overline{\theta}_s^v)\nu_{s-1}(dx_{s-1},d\theta_{s-1})}}}
{\splitdfrac{
\int\prod_{\omega\in K}f_{X_s^{\omega}|X_{s-1}}(\overline{x}_s^{\omega}\mid x_{s-1}\,;\overline{\theta}_s^{\omega})f_{\Theta_s^{\omega}|\Theta_{s-1}}(\overline{\theta}_s^{\omega}\mid \theta_{s-1}\,; \sigma)}{\splitdfrac{\times f_{Y_s^v|X_s^v}(Y_s^{v}\mid \overline{x}_s^{v}\,;\overline{\theta}_s^{v})\prod_{u\in N(v)\cap K'}f_{X_{s+1}^u|X_s}(x_{s+1}^u\mid \overline{x}_s\,;\theta_{s+1}^u)}{\times f_{\Theta_{s+1}^{u}|\Theta_s}(\theta_{s+1}^{u}\mid \overline{\theta}_s\,; \sigma)\psi^v(d\overline{x}_s^v)\lambda^v(d\overline{\theta}_s^v)\nu_{s-1}(dx_{s-1},d\theta_{s-1})}}}.\nonumber
    \end{align}


\section{Proofs for bounding bias}
\label{sec:bounding_bias_lemmas}

\begin{lemma} 
\label{beta_property}
Under condition \eqref{eqn:main_thm_condition}, for $\beta$ defined in \eqref{eqn:beta_definition}, the following holds:
\begin{enumerate}[(i)]
\item $(1-\epsilon_x^{2\Delta}[\epsilon_{\theta}(\sigma)]^{2\Delta})e^{2\beta r}\Delta^2\leq \frac{1}{16}$;
\item 
$(1-\epsilon_x^2[\epsilon_{\theta}(\sigma)]^2)e^{\beta(r+1)}\Delta \leq (1-\epsilon_x^{2\Delta}[\epsilon_{\theta}(\sigma)]^{2\Delta})e^{2\beta r}\Delta^2$.
\end{enumerate}
\end{lemma}

\begin{proof}
By the definition of $\beta$, we have
$$e^{\beta}=\left(\frac{1}{16\Delta_{\mathcal{K}}\Delta^2(1-\epsilon_x^{2\Delta}[\epsilon_{\theta}(\sigma)]^{2\Delta})} \right)^{\frac{1}{2r}}.$$
That is, 
$$e^{2r\beta}=\frac{1}{16\Delta_{\mathcal{K}}\Delta^2(1-\epsilon_x^{2\Delta}[\epsilon_{\theta}(\sigma)]^{2\Delta})}.$$
Since $r, \Delta, \Delta_{\mathcal{K}}\geq 1$ and $0\leq \epsilon_x,\epsilon_{\theta}(\sigma)\leq 1$, we have
$$(1-\epsilon_x^{2\Delta}[\epsilon_{\theta}(\sigma)]^{2\Delta})e^{2\beta r}\Delta^2\leq \frac{(1-\epsilon_x^{2\Delta}[\epsilon_{\theta}(\sigma)]^{2\Delta}) \Delta^2}{16\Delta_{\mathcal{K}}\Delta^2(1-\epsilon_x^{2\Delta}[\epsilon_{\theta}(\sigma)]^{2\Delta})} \leq \frac{1}{16}$$
\begin{align*}
(1-\epsilon_x^2[\epsilon_{\theta}(\sigma)]^2)e^{\beta(r+1)}\Delta &\leq (1-\epsilon_x^2[\epsilon_{\theta}(\sigma)]^2)e^{2\beta r}\Delta\leq (1-\epsilon_x^{2\Delta}[\epsilon_{\theta}(\sigma)]^{2\Delta})e^{2\beta r}\Delta^2.
\end{align*}

\end{proof}

\begin{lemma}
\label{lemma:block_corr}
Under Assumption \ref{assumption} in Section \ref{sec:beat_COD}, for any $n \geq 0$, 
we have
\begin{align*}
\widetilde{\Corr}(\widetilde{\pi}_{n},\beta)< \frac{1}{8},
\end{align*}
where $\widetilde{\Corr}(\,\cdot\,,\beta)$ is defined in \eqref{eqn:tilde_corr_definition} and $\beta$ is given in \eqref{eqn:beta_definition}.
\end{lemma}

\begin{proof}
Since $\widetilde{\pi}_0={\pi}_0=\delta_x\delta_{\theta}$ which is non-random, by the definition of $\widetilde{\Corr}(\,\cdot\,,\beta)$ given in \eqref{eqn:tilde_corr_definition}, we have $\widetilde{\Corr}(\widetilde{\pi}_0,\beta)=0$. In the following, we prove by the method of induction and assume that for $n\geq 1$ 
  $$\widetilde{\Corr}(\widetilde{\pi}_{n-1},\beta)< \frac{1}{8}.$$
  
  Let $K,K'\in \mathcal{K}$, $v\in K$, $v'\in V$ and $v\neq v'$. Let $x_n,\overline{x}_n \in \mathbb{X}$ such that $x_n^{V\backslash \{v'\}}=\overline{x}_n^{V\backslash \{v'\}}.$
Let $\theta_n,\overline{\theta}_n \in \mathbb{\Theta}$ such that $\theta_n^{V\backslash \{v'\}}=\overline{\theta}_n^{V\backslash \{v'\}}$.
Define $I=(\{n-1\}\times V)\cup (n,v)$ and $\mathbb{S}=(\mathbb{X}\times \mathbb{\Theta})\times (\mathbb{X}^v\times \mathbb{\Theta}^v)$. Define
\begin{align*}
\rho(A)
=\dfrac{\splitdfrac{
\int\mathbbm{1}_A(x_{n-1},\theta_{n-1},x_n^v,\theta_n^v)\prod_{\omega\in K}f_{X_n^{\omega}|X_{n-1}}(x_n^{\omega}\mid x_{n-1}\,;\theta_n^{\omega})}{\splitdfrac{\times f_{\Theta_n^{\omega}|\Theta_{n-1}}(\theta_n^{\omega}\mid \theta_{n-1}\,; \sigma)
\prod_{u\in N(v)\cap K'}f_{X_{n+1}^u|X_n}(x_{n+1}^u\mid x_n\;;\theta_{n+1}^u)}{\splitdfrac{\times  f_{\Theta_{n+1}^u\mid \Theta_n}(\theta_{n+1}^u\mid \theta_n\,; \sigma)f_{Y_n^v|X_n^v}(Y_n^v\mid x_n^v\,;\theta_n^v)}{ 
 \times \psi^v(dx_n^v)\lambda^v(d\theta_n^v)\widetilde{\pi}_{n-1}(dx_{n-1},d\theta_{n-1})}}}}
{\splitdfrac{
\int\prod_{\omega\in K}f_{X_n^{\omega}|X_{n-1}}(x_n^{\omega}\mid x_{n-1}\,;\theta_n^{\omega})}{\splitdfrac{\times f_{\Theta_n^{\omega}|\Theta_{n-1}}(\theta_n^{\omega}\mid \theta_{n-1}\,; \sigma)
\prod_{u\in N(v)\cap K'}f_{X_{n+1}^u|X_n}(x_{n+1}^u\mid x_n\;;\theta_{n+1}^u)}{\splitdfrac{\times  f_{\Theta_{n+1}^u\mid \Theta_n}(\theta_{n+1}^u\mid \theta_n\,; \sigma)f_{Y_n^v|X_n^v}(Y_n^v\mid x_n^v\,;\theta_n^v)}{ 
 \times \psi^v(dx_n^v)\lambda^v(d\theta_n^v)\widetilde{\pi}_{n-1}(dx_{n-1},d\theta_{n-1})}}}}
\end{align*}
and
\begin{align*}
\overline{\rho}(A)
=\dfrac{\splitdfrac{
\int\mathbbm{1}_A(x_{n-1},\theta_{n-1},\overline{x}_n^v,\overline{\theta}_n^v)\prod_{\omega\in K}f_{X_n^{\omega}|X_{n-1}}(\overline{x}_n^{\omega}\mid x_{n-1}\,;\theta_n^{\omega})}{\splitdfrac{\times f_{\Theta_n^{\omega}|\Theta_{n-1}}(\overline{\theta}_n^{\omega}\mid \theta_{n-1}\,; \sigma)
\prod_{u\in N(v)\cap K'}f_{X_{n+1}^u|X_n}(x_{n+1}^u\mid \overline{x}_n\;;\theta_{n+1}^u) }{\splitdfrac{\times f_{\Theta_{n+1}^u\mid \Theta_n}(\theta_{n+1}^u\mid \overline{\theta}_n\,; \sigma) f_{Y_n^v|X_n^v}(Y_n^v\mid \overline{x}_n^v\,;\overline{\theta}_n^v) }{
\times  \psi^v(d\overline{x}_n^v)\lambda^v(d\overline{\theta}_n^v)\widetilde{\pi}_{n-1}(dx_{n-1},d\theta_{n-1})}}}}
{\splitdfrac{
\int \prod_{\omega\in K}f_{X_n^{\omega}|X_{n-1}}(\overline{x}_n^{\omega}\mid x_{n-1}\,;\theta_n^{\omega})}{\splitdfrac{\times f_{\Theta_n^{\omega}|\Theta_{n-1}}(\overline{\theta}_n^{\omega}\mid \theta_{n-1}\,; \sigma)
\prod_{u\in N(v)\cap K'}f_{X_{n+1}^u|X_n}(x_{n+1}^u\mid \overline{x}_n\;;\theta_{n+1}^u) }{\splitdfrac{\times f_{\Theta_{n+1}^u\mid \Theta_n}(\theta_{n+1}^u\mid \overline{\theta}_n\,; \sigma) f_{Y_n^v|X_n^v}(Y_n^v\mid \overline{x}_n^v\,;\overline{\theta}_n^v) }{
\times  \psi^v(d\overline{x}_n^v)\lambda^v(d\overline{\theta}_n^v)\widetilde{\pi}_{n-1}(dx_{n-1},d\theta_{n-1})}}}}.
\end{align*}
Then according to the definitions of 
$(\widetilde{\mathsf{F}}_s\nu_{s-1})_{\chi_s,\chi_{s+1}}^{v,K'}$in \eqref{def:widetildeF_s_rho_s1_K} and 
    $(\widetilde{\mathsf{F}}_s\nu_{s-1})_{\overline{\chi}_s,\chi_{s+1}}^{v,K'}$in \eqref{def:widetildeF_s_rho_s1_K_overline} for $s\geq 1$,
we have
$$\|(\widetilde{\mathsf{F}}_n\widetilde{\pi}_{n-1})_{\chi_n,\chi_{n+1}}^{v,K'}-(\widetilde{\mathsf{F}}_n\widetilde{\pi}_{n-1})_{\overline{\chi}_{n},\chi_{n+1}}^{v,K'} \|=\|\rho-\overline{\rho}\|_{(n,v)}.$$

In order to use Theorem \ref{thm:Dobrushin} (Dobrushin comparison theorem)  in Appendix \ref{sec:Existing_results} to bound $\|\rho-\overline{\rho}\|_{(n,v)}$, we need to bound $C_{ij}$ and $b_i$ with $i=(k,t)$ and $j=(k',t')$. 
Set 
$$\rho^i=\rho_{(x_{n-1},\theta_{n-1},{x}_n^v,{\theta}_n^v)}^i\quad\text{and}\quad\overline{\rho}^i=\overline{\rho}_{(x_{n-1},\theta_{n-1},\overline{x}_n^v,\overline{\theta}_n^v)}^i,$$
whose definitions are given in Theorem \ref{thm:Dobrushin} in Appendix \ref{sec:Existing_results}. Note that
$$\overline{\rho}^i\neq \rho_{(\overline{x}_{n-1},\overline{\theta}_{n-1},\overline{x}_n^v,\overline{\theta}_n^v)}^i.$$
We display our discussions as follows:
\begin{itemize}
\item When $k=n-1$, we have
\begin{align*}
\rho^i(A)=\dfrac{\splitdfrac{
\int\mathbbm{1}_A(x_{n-1}^t,\theta_{n-1}^t)\prod_{\omega\in N(t)\cap K}f_{X_n^{\omega}|X_{n-1}}(x_n^{\omega}\mid x_{n-1}\,;\theta_n^{\omega})}{\times f_{\Theta_n^{\omega}|\Theta_{n-1}}(\theta_n^{\omega}\mid \theta_{n-1}\,; \sigma) \widetilde{\pi}_{n-1}^t(dx_{n-1}^t,d\theta_{n-1}^t)}}
{\splitdfrac{
\int \prod_{\omega\in N(t)\cap K}f_{X_n^{\omega}|X_{n-1}}(x_n^{\omega}\mid x_{n-1}\,;\theta_n^{\omega})}{\times f_{\Theta_n^{\omega}|\Theta_{n-1}}(\theta_n^{\omega}\mid \theta_{n-1}\,; \sigma) \widetilde{\pi}_{n-1}^t(dx_{n-1}^t,d\theta_{n-1}^t)}}.
\end{align*}
We can see that $\rho^i=\widetilde{\pi}_{\chi_{n-1},\chi_n}^{t,K}$, by the definition of $\widetilde{\pi}_{\chi_{n-1},\chi_n}^{t,K}$ given in \eqref{def:mu_s-1_s_K}. Therefore, if $k'=n-1$, by the definition of $\widetilde{C}_{tt'}^{\widetilde{\pi}_{n-1}}$ given in \eqref{eqn:tilde_C_vv'_mu}, we know that $C_{ij}\leq \widetilde{C}_{tt'}^{\widetilde{\pi}_{n-1}}$.
Note that
\begin{align*}
\rho^i(A)\geq \epsilon_x^2[\epsilon_{\theta}(\sigma)]^2\dfrac{\splitdfrac{
\int\mathbbm{1}_A(x_{n-1}^t,\theta_{n-1}^t)\prod_{\omega\in N(t)\cap (K\backslash \{v\})}f_{X_n^{\omega}|X_{n-1}}(x_n^{\omega}\mid x_{n-1}\,;\theta_n^{\omega})}{\times f_{\Theta_n^{\omega}|\Theta_{n-1}}(\theta_n^{\omega}\mid \theta_{n-1}\,; \sigma) \widetilde{\pi}_{n-1}^t(dx_{n-1}^t,d\theta_{n-1}^t)}}
{\splitdfrac{
\int \prod_{\omega\in N(t)\cap (K\backslash \{v\})}f_{X_n^{\omega}|X_{n-1}}(x_n^{\omega}\mid x_{n-1}\,;\theta_n^{\omega})}{\times f_{\Theta_n^{\omega}|\Theta_{n-1}}(\theta_n^{\omega}\mid \theta_{n-1}\,; \sigma) \widetilde{\pi}_{n-1}^t(dx_{n-1}^t,d\theta_{n-1}^t)}},
\end{align*}
and then we have $C_{ij}\leq 1-\epsilon_x^2[\epsilon_{\theta}(\sigma)]^2$ if $k'=n$ and $v\in N(t)$ by Theorem \ref{thm:VanHandel4.1} in Appendix \ref{sec:Existing_results}, and $C_{ij}=0$ otherwise. Recalling that at the beginning of this proof we set $v\in K$, $v'\in V$, $v\neq v'$, $x_n,\overline{x}_n \in \mathbb{X}$ such that $x_n^{V\backslash \{v'\}}=\overline{x}_n^{V\backslash \{v'\}}$, and $\theta_n,\overline{\theta}_n \in \mathbb{\Theta}$ such that $\theta_n^{V\backslash \{v'\}}=\overline{\theta}_n^{V\backslash \{v'\}}$, then we have $\rho^i=\overline{\rho}^i$ if $v'\notin N(t)\cap K$. Furthermore, note that
\begin{align*}
\rho^i(A)
\geq \epsilon_x^2[\epsilon_{\theta}(\sigma)]^2\dfrac{\splitdfrac{
\int\mathbbm{1}_A(x_{n-1}^t,\theta_{n-1}^t)\prod_{\omega\in N(t)\cap (K\backslash \{v'\})}f_{X_n^{\omega}|X_{n-1}}(x_n^{\omega}\mid x_{n-1}\,;\theta_n^{\omega})}{\times f_{\Theta_n^{\omega}|\Theta_{n-1}}(\theta_n^{\omega}\mid \theta_{n-1}\,; \sigma) \widetilde{\pi}_{n-1}^t(dx_{n-1}^t,d\theta_{n-1}^t)}}
  {\splitdfrac{
    \int \prod_{\omega\in N(t)\cap (K\backslash \{v'\})}f_{X_n^{\omega}|X_{n-1}}(x_n^{\omega}\mid x_{n-1}\,;\theta_n^{\omega})}{\times f_{\Theta_n^{\omega}|\Theta_{n-1}}(\theta_n^{\omega}\mid \theta_{n-1}\,; \sigma) \widetilde{\pi}_{n-1}^t(dx_{n-1}^t,d\theta_{n-1}^t)}},
\end{align*}
and
\begin{align*}
\overline{\rho}^i(A)
&=\dfrac{\splitdfrac{
\int\mathbbm{1}_A(x_{n-1}^t,\theta_{n-1}^t)\prod_{\omega\in N(t)\cap K}f_{X_n^{\omega}|X_{n-1}}(\overline{x}_n^{\omega}\mid x_{n-1}\,;\overline{\theta}_n^{\omega})}{\times f_{\Theta_n^{\omega}|\Theta_{n-1}}(\overline{\theta}_n^{\omega}\mid \theta_{n-1}\,; \sigma) \widetilde{\pi}_{n-1}^t(dx_{n-1}^t,d\theta_{n-1}^t)}}
{\splitdfrac{
\int \prod_{\omega\in N(t)\cap K}f_{X_n^{\omega}|X_{n-1}}(\overline{x}_n^{\omega}\mid x_{n-1}\,;\overline{\theta}_n^{\omega})}{\times f_{\Theta_n^{\omega}|\Theta_{n-1}}(\overline{\theta}_n^{\omega}\mid \theta_{n-1}\,; \sigma) \widetilde{\pi}_{n-1}^t(dx_{n-1}^t,d\theta_{n-1}^t)}}\\
& \geq \epsilon_x^2[\epsilon_{\theta}(\sigma)]^2\dfrac{\splitdfrac{
\int\mathbbm{1}_A(x_{n-1}^t,\theta_{n-1}^t)\prod_{\omega\in N(t)\cap (K\backslash \{v'\})}f_{X_n^{\omega}|X_{n-1}}(\overline{x}_n^{\omega}\mid x_{n-1}\,;\overline{\theta}_n^{\omega})}{\times f_{\Theta_n^{\omega}|\Theta_{n-1}}(\overline{\theta}_n^{\omega}\mid \theta_{n-1}\,; \sigma) \widetilde{\pi}_{n-1}^t(dx_{n-1}^t,d\theta_{n-1}^t)}}
  {\splitdfrac{
    \int \prod_{\omega\in N(t)\cap (K\backslash \{v'\})}f_{X_n^{\omega}|X_{n-1}}(\overline{x}_n^{\omega}\mid x_{n-1}\,;\overline{\theta}_n^{\omega})}{\times f_{\Theta_n^{\omega}|\Theta_{n-1}}(\overline{\theta}_n^{\omega}\mid \theta_{n-1}\,; \sigma) \widetilde{\pi}_{n-1}^t(dx_{n-1}^t,d\theta_{n-1}^t)}}.
\end{align*}
Hence, we have $b_i=0$ if $v'\notin N(t)\cap K$, and by Theorem \ref{thm:VanHandel4.1} in Appendix \ref{sec:Existing_results} we have $b_i=2(1-\epsilon_x^{2}[\epsilon_{\theta}(\sigma)]^2)$ otherwise.\\

\item When $k=n$, we have
\begin{align*}
\rho^i(A)
&=\dfrac{\splitdfrac{
  \int\mathbbm{1}_A(x_n^v,\theta_n^v) f_{X_n^v|X_{n-1}}(x_n^v\mid x_{n-1}\,;\theta_n^v)f_{\Theta_n^v|\Theta_{n-1}}(\theta_n^v\mid \theta_{n-1}\,; \sigma)}{\splitdfrac{\times 
    \prod_{u\in N(v)\cap K'}f_{X_{n+1}^u|X_n}(x_{n+1}^u\mid x_n\;;\theta_{n+1}^u) f_{\Theta_{n+1}^u\mid \Theta_n}(\theta_{n+1}^u\mid \theta_n\,; \sigma)}{\times f_{Y_n^v|X_n^v}(Y_n^v\mid x_n^v\,;\theta_n^v) 
\psi^v(dx_n^v)\lambda^v(d\theta_n^v)}}}
{\splitdfrac{
\int  f_{X_n^v|X_{n-1}}(x_n^v\mid x_{n-1}\,;\theta_n^v)f_{\Theta_n^v|\Theta_{n-1}}(\theta_n^v\mid \theta_{n-1}\,; \sigma)}{\splitdfrac{\times 
\prod_{u\in N(v)\cap K'}f_{X_{n+1}^u|X_n}(x_{n+1}^u\mid x_n\;;\theta_{n+1}^u) f_{\Theta_{n+1}^u\mid \Theta_n}(\theta_{n+1}^u\mid \theta_n\,; \sigma)}{\times f_{Y_n^v|X_n^v}(Y_n^v\mid x_n^v\,;\theta_n^v) 
  \psi^v(dx_n^v)\lambda^v(d\theta_n^v)}}}\\
&\geq \epsilon_x^2[\epsilon_{\theta}(\sigma)]^2\dfrac{\splitdfrac{
  \int\mathbbm{1}_A(x_n^v,\theta_n^v)\prod_{u\in N(v)\cap K'}f_{X_{n+1}^u|X_n}(x_{n+1}^u\mid x_n\;;\theta_{n+1}^u)  }{\times f_{\Theta_{n+1}^u\mid \Theta_n}(\theta_{n+1}^u\mid \theta_n\,; \sigma)
 f_{Y_n^v|X_n^v}(Y_n^v\mid x_n^v\,;\theta_n^v) 
\psi^v(dx_n^v)\lambda^v(d\theta_n^v)}}
{\splitdfrac{
\int \prod_{u\in N(v)\cap K'}f_{X_{n+1}^u|X_n}(x_{n+1}^u\mid x_n\;;\theta_{n+1}^u)  }{\times f_{\Theta_{n+1}^u\mid \Theta_n}(\theta_{n+1}^u\mid \theta_n\,; \sigma)
  f_{Y_n^v|X_n^v}(Y_n^v\mid x_n^v\,;\theta_n^v) 
  \psi^v(dx_n^v)\lambda^v(d\theta_n^v)}}.
\end{align*}
Then we have $C_{ij}\leq 1-\epsilon_x^2[\epsilon_{\theta}(\sigma)]^2$ if $k'=n-1$ and $t'\in N(v)$ by Theorem \ref{thm:VanHandel4.1} in Appendix \ref{sec:Existing_results}, and $C_{ij}=0$ otherwise. Note that $b_i=0$ if $v'\notin \cup_{\omega\in N(v)\cap K'}N(\omega)$. Furthermore,
\begin{align*}
\rho^i(A)
&\geq \epsilon_x^{2\Delta}[\epsilon_{\theta}(\sigma)]^{2\Delta} \dfrac{\splitdfrac{
  \int\mathbbm{1}_A({x}_n^v,{\theta}_n^v) f_{X_n^v|X_{n-1}}({x}_n^v\mid x_{n-1}\,;{\theta}_n^v)f_{\Theta_n^v|\Theta_{n-1}}({\theta}_n^v\mid \theta_{n-1}\,; \sigma)}{\times f_{Y_n^v|X_n^v}(Y_n^v\mid {x}_n^v\,;{\theta}_n^v) 
    \psi^v(d{x}_n^v)\lambda^v(d{\theta}_n^v)}}
{\splitdfrac{
  \int  f_{X_n^v|X_{n-1}}({x}_n^v\mid x_{n-1}\,;{\theta}_n^v)f_{\Theta_n^v|\Theta_{n-1}}({\theta}_n^v\mid \theta_{n-1}\,; \sigma)}{\times f_{Y_n^v|X_n^v}(Y_n^v\mid {x}_n^v\,;{\theta}_n^v) 
    \psi^v(d{x}_n^v)\lambda^v(d{\theta}_n^v)}}
\end{align*}
and
\begin{align*}
\overline{\rho}^i(A)
&=\dfrac{\splitdfrac{
  \int\mathbbm{1}_A(\overline{x}_n^v,\overline{\theta}_n^v) f_{X_n^v|X_{n-1}}(\overline{x}_n^v\mid x_{n-1}\,;\overline{\theta}_n^v)f_{\Theta_n^v|\Theta_{n-1}}(\overline{\theta}_n^v\mid \theta_{n-1}\,; \sigma)}{\splitdfrac{\times 
    \prod_{u\in N(v)\cap K'}f_{X_{n+1}^u|X_n}(x_{n+1}^u\mid \overline{x}_n\;;\theta_{n+1}^u) f_{\Theta_{n+1}^u\mid \Theta_n}(\theta_{n+1}^u\mid \overline{\theta}_n\,; \sigma)}{\times f_{Y_n^v|X_n^v}(Y_n^v\mid \overline{x}_n^v\,;\overline{\theta}_n^v) 
\psi^v(d\overline{x}_n^v)\lambda^v(d\overline{\theta}_n^v)}}}
{\splitdfrac{
\int f_{X_n^v|X_{n-1}}(\overline{x}_n^v\mid x_{n-1}\,;\overline{\theta}_n^v)f_{\Theta_n^v|\Theta_{n-1}}(\overline{\theta}_n^v\mid \theta_{n-1}\,; \sigma)}{\splitdfrac{\times 
\prod_{u\in N(v)\cap K'}f_{X_{n+1}^u|X_n}(x_{n+1}^u\mid \overline{x}_n\;;\theta_{n+1}^u) f_{\Theta_{n+1}^u\mid \Theta_n}(\theta_{n+1}^u\mid \overline{\theta}_n\,; \sigma)}{\times f_{Y_n^v|X_n^v}(Y_n^v\mid \overline{x}_n^v\,;\overline{\theta}_n^v) 
  \psi^v(d\overline{x}_n^v)\lambda^v(d\overline{\theta}_n^v)}}}\\
&\geq \epsilon_x^{2\Delta}[\epsilon_{\theta}(\sigma)]^{2\Delta} \dfrac{\splitdfrac{
  \int\mathbbm{1}_A(\overline{x}_n^v,\overline{\theta}_n^v) f_{X_n^v|X_{n-1}}(\overline{x}_n^v\mid x_{n-1}\,;\overline{\theta}_n^v)f_{\Theta_n^v|\Theta_{n-1}}(\overline{\theta}_n^v\mid \theta_{n-1}\,; \sigma)}{\times f_{Y_n^v|X_n^v}(Y_n^v\mid \overline{x}_n^v\,;\overline{\theta}_n^v) 
\psi^v(d\overline{x}_n^v)\lambda^v(d\overline{\theta}_n^v)}}
{\splitdfrac{
  \int  f_{X_n^v|X_{n-1}}(\overline{x}_n^v\mid x_{n-1}\,;\overline{\theta}_n^v)f_{\Theta_n^v|\Theta_{n-1}}(\overline{\theta}_n^v\mid \theta_{n-1}\,; \sigma)}{\times f_{Y_n^v|X_n^v}(Y_n^v\mid \overline{x}_n^v\,;\overline{\theta}_n^v) 
\psi^v(d\overline{x}_n^v)\lambda^v(d\overline{\theta}_n^v)}}.
\end{align*}
By Theorem \ref{thm:VanHandel4.1} in Appendix \ref{sec:Existing_results}, we have $b_i=2(1-\epsilon_x^{2\Delta}[\epsilon_{\theta}(\sigma)]^{2\Delta})$  if $v'\in \cup_{\omega\in N(v)\cap K'}N(\omega)$.
\end{itemize}

Define the matrix $(C_{ij}(v))_{i,j\in I}$ whose entries are given below:
  \begin{align*}
C_{(n-1,t)(n-1,t’)}(v)=&\widetilde{C}_{tt'}^{\widetilde{\pi}_{n-1}},\\
C_{(n-1,t)(n,v)}(v)=&C_{(n,v)(n-1,t)}(v)=(1-\epsilon_x^2[\epsilon_{\theta}(\sigma)]^2)\mathbbm{1}_{t\in N(v)},\\
C_{(n,v)(n,v)}(v)=&0.
\end{align*}
In sum, we have that
\begin{align}
\label{eqn:block_corr_half}
\sum_{(k',t')\in I}e^{\beta |k-k'|}e^{\beta d(t,t')}C_{(k,t)(k',t')}&\leq \sum_{(k',t')\in I}e^{\beta |k-k'|}e^{\beta d(t,t')}C_{(k,t)(k',t')}(v)\nonumber\\
&\leq \widetilde{\Corr}(\widetilde{\pi}_{n-1},\beta)+(1-\epsilon_x^2[\epsilon_{\theta}(\sigma)]^2)e^{\beta(r+1)}\Delta\nonumber\\
&< \frac{1}{8}+\frac{1}{16}\nonumber\\
&< \frac{1}{2},
\end{align}
where the third inequality is by the assumption of the induction method and Lemma \ref{beta_property} in Appendix \ref{sec:bounding_bias_lemmas}.

Next, 
applying Theorem \ref{thm:Dobrushin} (Dobrushin comparison theorem)  in Appendix \ref{sec:Existing_results},
we obtain that
\begin{align*}
\|\rho-\overline{\rho}\|_{(n,v)}
\leq & 2(1-\epsilon_x^2[\epsilon_{\theta}(\sigma)]^2)\mathbbm{1}_{\{v'\in K\}}\sum_{t'\in N(v')}D_{(n,v)(n-1,t')}(v)\\
&+2(1-\epsilon_x^{2\Delta}[\epsilon_{\theta}(\sigma)]^{2\Delta})\mathbbm{1}_{\{v'\in \cup_{\omega\in N(v)\cap K'}N(\omega)
\}}D_{(n,v)(n,v)}(v),
\end{align*}
where $D(v)=\sum_{s\geq 0}[C(v)]^s$. 
Therefore,
\begin{align*}
&\hspace*{-0.7cm}\frac{1}{2}\sup_{K'\in\mathcal{K}}\sup_{x_{n+1}\in \mathbb{X}\atop
  \theta_{n+1}\in \mathbb{\Theta}}\sup_{x_n,\overline{x}_n\in \mathbb{X}:x_n^{V\backslash \{v'\}}=\overline{x}_n^{V\backslash \{v'\}}\atop
    \theta_n,\overline{\theta}_n\in \mathbb{\Theta}:\theta_n^{V\backslash \{v'\}}=\overline{\theta}_n^{V\backslash \{v'\}}
  }\|(\widetilde{\mathsf{F}}_n\widetilde{\pi}_{n-1})_{\chi_n,\chi_{n+1}}^{v,K'}-(\widetilde{\mathsf{F}}_n\widetilde{\pi}_{n-1})_{\overline{\chi}_{n},\chi_{n+1}}^{v,K'} \| \\
  &\leq  (1-\epsilon_x^2[\epsilon_{\theta}(\sigma)]^2)\mathbbm{1}_{\{v'\in K\}}\sum_{t'\in N(v')}D_{(n,v)(n-1,t')}(v)\\
    &\quad\quad+(1-\epsilon_x^{2\Delta}[\epsilon_{\theta}(\sigma)]^{2\Delta})\mathbbm{1}_{\{v'\in \cup_{\omega\in N(v)\cap K'}N(\omega)\}
  }D_{(n,v)(n,v)}(v).
  \end{align*}
  By the definition of $\widetilde{C}_{vv'}^{\mu_{s-1}}$ in equation \eqref{eqn:tilde_C_vv'_mu}, we have
  \begin{align*}
  \widetilde{C}_{vv'}^{\widetilde{\pi}_n}
\leq & (1-\epsilon_x^2[\epsilon_{\theta}(\sigma)]^2)\mathbbm{1}_{\{v'\in K\}}\sum_{t'\in N(v')}D_{(n,v)(n-1,t')}(v)\\
&+(1-\epsilon_x^{2\Delta}[\epsilon_{\theta}(\sigma)]^{2\Delta})\mathbbm{1}_{\{v'\in \cup_{\omega\in N(v)\cap K'}N(\omega)\}
}D_{(n,v)(n,v)}(v).
\end{align*}
Note that
\begin{align*}
\sum_{v'\in V} e^{\beta d(v,v')}\widetilde{C}_{vv'}^{\widetilde{\pi}_n}
                                                                                              \leq & (1-\epsilon_x^2[\epsilon_{\theta}(\sigma)]^2)\sum_{v'\in K} e^{\beta d(v,v')}\sum_{t'\in N(v')}D_{(n,v)(n-1,t')}(v)\\
&+(1-\epsilon_x^{2\Delta}[\epsilon_{\theta}(\sigma)]^{2\Delta})\sum_{v'\in \cup_{\omega\in N(v)\cap K'}N(\omega)} e^{\beta d(v,v')}D_{(n,v)(n,v)}(v)\\
                                                                                                                                                                                                   \leq & (1-\epsilon_x^2[\epsilon_{\theta}(\sigma)]^2)\sum_{v'\in K} \sum_{t'\in N(v')}e^{\beta d(v,v')}D_{(n,v)(n-1,t')}(v)\\
&+(1-\epsilon_x^{2\Delta}[\epsilon_{\theta}(\sigma)]^{2\Delta})\sum_{v'\in \cup_{\omega\in N(v)\cap K'}N(\omega)} e^{\beta d(v,v')}D_{(n,v)(n,v)}(v)\\
                                                                                                                                                                                                                                                                                                        \leq & (1-\epsilon_x^2[\epsilon_{\theta}(\sigma)]^2)\sum_{v'\in K} \sum_{t'\in N(v')}e^{\beta d(v,t')}e^{\beta d(t',v')}D_{(n,v)(n-1,t')}(v)\\
&+(1-\epsilon_x^{2\Delta}[\epsilon_{\theta}(\sigma)]^{2\Delta})\sum_{v'\in \cup_{\omega\in N(v)\cap K'}N(\omega)} e^{\beta d(v,\omega)}e^{\beta d(\omega,v')}D_{(n,v)(n,v)}(v)\\
                                                                                                                                                                                                                                                                                                                                                                                                                               \leq & (1-\epsilon_x^2[\epsilon_{\theta}(\sigma)]^2)\sum_{v'\in K} \sum_{t'\in N(v')}e^{\beta d(v,t')}e^{\beta r}D_{(n,v)(n-1,t')}(v)\\
&+(1-\epsilon_x^{2\Delta}[\epsilon_{\theta}(\sigma)]^{2\Delta})e^{2\beta r}\Delta^2 D_{(n,v)(n,v)}(v)\\
\leq & (1-\epsilon_x^2[\epsilon_{\theta}(\sigma)]^2)\Delta e^{\beta r}\sum_{v'\in V} e^{\beta d(v,v')}D_{(n,v)(n-1,v')}(v)\\
                                                                                                                                                                                                                                                                                                                                                                                                                               &+(1-\epsilon_x^{2\Delta}[\epsilon_{\theta}(\sigma)]^{2\Delta})e^{2\beta r}\Delta^2 D_{(n,v)(n,v)}(v)\\
                                                                                                                                                                                                                                                                                                                                                                                                                               \leq & (1-\epsilon_x^{2\Delta}[\epsilon_{\theta}(\sigma)]^{2\Delta})e^{2\beta r} \Delta^2 \sum_{(k',v')\in I}e^{\beta\{|k-k'|+d(v,v')\}} D_{(n,v)(k',v')}(v).
                                                                                                                                                                                                                                                                                                                                                                                                                               \end{align*}
                                                                                                                                                                                                                                                                                                                                                                                                                               By the definition of $\widetilde{\Corr}(\,\cdot\,,\beta)$ in equation \eqref{eqn:tilde_corr_definition}, we have
                                                                                                                                                                                                                                                                                                                                                                                                                               \begin{align*}
                                                                                                                                                                                                                                                                                                                                                                                                                               \widetilde{\Corr}(\widetilde{\pi}_n,\beta)
                                                                                                                                                                                                                                                                                                                                                                                                                               =&\max_{v\in V}\sum_{v'\in V} e^{\beta d(v,v')}\widetilde{C}_{vv'}^{\widetilde{\pi}_n}\\
\leq & \max_{v\in V}(1-\epsilon_x^{2\Delta}[\epsilon_{\theta}(\sigma)]^{2\Delta})e^{2\beta r} \Delta^2 \sum_{(k',v')\in I}e^{\beta\{|k-k'|+d(v,v')\}} D_{(n,v)(k',v')}(v)\\
< & 2(1-\epsilon_x^{2\Delta}[\epsilon_{\theta}(\sigma)]^{2\Delta})e^{2\beta r}\Delta^2\\
\leq & \frac{1}{8},
\end{align*}
where the second inequality holds by \eqref{eqn:block_corr_half} and Theorem \ref{thm:VanHandel4.3}  in Appendix \ref{sec:Existing_results}, and the last inequality holds by Lemma \ref{beta_property} in Appendix \ref{sec:bounding_bias_lemmas}.
\end{proof}

\begin{lemma}
\label{thm:onethird}
Under Assumption \ref{assumption} in Section \ref{sec:beat_COD}, when condition \eqref{eqn:main_thm_condition} holds,
for any $n\geq 0$,
we have
\begin{align*}
\Corr(\widetilde{\pi}_{n},\beta)
< \frac{1}{3},
\end{align*}
where $\Corr(\,\cdot\,,\beta)$ is defined in \eqref{eqn:corr_definition} and $\beta$ is given in \eqref{eqn:beta_definition}.
\end{lemma}

\begin{proof}
For $\mu_{\chi_{s-1},\chi_{s}}^v$ defined in \eqref{def:mu_s-1_s} and $\mu_{\chi_{s-1},\chi_{s}}^{v,K'}$ defined in \eqref{def:mu_s-1_s_K}, 
we have
\begin{align*}
\widetilde{\pi}_{\chi_{n},\chi_{n+1}}^{v}(A)
&=\dfrac{\splitdfrac{
  \int\mathbbm{1}_A(x_{n}^v,\theta_{n}^v)\prod_{\omega\in N(v)\backslash K}f_{X_{n+1}^{\omega}|X_{n}}(x_{n+1}^{\omega}\mid x_{n}\,;\theta_{n+1}^{\omega})}{\times f_{\Theta_{n+1}^{\omega}|\Theta_{n}}(\theta_{n+1}^{\omega}\mid \theta_{n}\,; \sigma)\widetilde{\pi}_{\chi_{n},\chi_{n+1}}^{v,K}(d x_{n}^v,d\theta_{n}^v)}}
{\splitdfrac{
  \int\prod_{\omega\in N(v)\backslash K}f_{X_{n+1}^{\omega}|X_{n}}(x_{n+1}^{\omega}\mid x_{n}\,;\theta_{n+1}^{\omega})}{\times f_{\Theta_{n+1}^{\omega}|\Theta_{n}}(\theta_{n+1}^{\omega}\mid \theta_{n}\,; \sigma)\widetilde{\pi}_{\chi_{n},\chi_{n+1}}^{v,K}(d x_{n}^v,d\theta_{n}^v)}}.
\end{align*}
Let $x_n,\overline{x}_n\in \mathbb{X}$ be such that $x_n^{V\backslash \{v'\}}=\overline{x}_n^{V\backslash \{v'\}}$, and $\theta_n,\overline{\theta}_n\in \mathbb{\Theta}$ be such that $\theta_n^{V\backslash \{v'\}}=\overline{\theta}_n^{V\backslash \{v'\}}$. If $v'\notin \cup_{\omega\in N(v)}N(\omega)$, for $\mu_{\overline{\chi}_{s-1},\chi_s}^v$ defined in \eqref{def:bar_mu_s-1_s} and $\mu_{\overline{\chi}_{s-1},\chi_{s}}^{v,K'}$ defined in \eqref{def:bar_mu_s-1_s_K},
  by Theorem \ref{thm:VanHandel4.2} in Appendix \ref{sec:Existing_results}, we have
  $$\|\widetilde{\pi}_{\chi_{n},\chi_{n+1}}^{v}-\widetilde{\pi}_{\overline{\chi}_{n},\chi_{n+1}}^{v}\|\leq 2\epsilon_x^{-2\Delta}[\epsilon_{\theta}(\sigma)]^{-2\Delta}\|\widetilde{\pi}_{\chi_{n},\chi_{n+1}}^{v,K}-\widetilde{\pi}_{\overline{\chi}_{n},\chi_{n+1}}^{v,K}\|.$$
    Note that
  \begin{align*}
  \widetilde{\pi}_{\chi_{n},\chi_{n+1}}^{v}(A)\geq &\epsilon_x^{2\Delta}[\epsilon_{\theta}(\sigma)]^{2\Delta}\widetilde{\pi}_{\chi_{n},\chi_{n+1}}^{v,K}(A),\\
  \widetilde{\pi}_{\overline{\chi}_{n},\chi_{n+1}}^{v}(A)\geq &\epsilon_x^{2\Delta}[\epsilon_{\theta}(\sigma)]^{2\Delta}\widetilde{\pi}_{\overline{\chi}_{n},\chi_{n+1}}^{v,K}(A),
  \end{align*}
  and then by Theorem \ref{thm:VanHandel4.1} in Appendix \ref{sec:Existing_results} we have that if $v'\in \cup_{\omega\in N(v)}N(\omega)$
  $$\|\widetilde{\pi}_{\chi_{n},\chi_{n+1}}^{v}-\widetilde{\pi}_{\overline{\chi}_{n},\chi_{n+1}}^{v}\|\leq 
2(1-\epsilon_x^{2\Delta}[\epsilon_{\theta}(\sigma)]^{2\Delta})+
  \epsilon_x^{2\Delta}[\epsilon_{\theta}(\sigma)]^{2\Delta}\|\widetilde{\pi}_{\chi_{n},\chi_{n+1}}^{v,K}-\widetilde{\pi}_{\overline{\chi}_{n},\chi_{n+1}}^{v,K}\|.$$
  One can establish the inequality 
\begin{align*}
&\hspace*{-0.25cm}\Corr(\widetilde{\pi}_{n},\beta)\\
\leq & (1-\epsilon_x^{2\Delta}[\epsilon_{\theta}(\sigma)]^{2\Delta}) \max_{v\in V}\sum_{v'\in \cup_{\omega\in N(v)}N(\omega)}e^{\beta d(v,v')}+2\epsilon_x^{-2\Delta}[\epsilon_{\theta}(\sigma)]^{-2\Delta}\widetilde{\Corr}(\widetilde{\pi}_{n},\beta)\\
\leq &(1-\epsilon_x^{2\Delta}[\epsilon_{\theta}(\sigma)]^{2\Delta})e^{2\beta r}\Delta^2+2\epsilon_x^{-2\Delta}[\epsilon_{\theta}(\sigma)]^{-2\Delta}\widetilde{\Corr}(\widetilde{\pi}_{n},\beta).
\end{align*}
Under condition \eqref{eqn:main_thm_condition} that
$$
  \epsilon_x\epsilon_{\theta}(\sigma)>\left(1-\frac{1}{16\Delta_{\mathcal{K}}\Delta^2} \right)^{\frac{1}{2\Delta}},
$$
  since $\Delta, \Delta_{\mathcal{K}}\geq 1$, we have 
$$\epsilon_x^{-2\Delta}[\epsilon_{\theta}(\sigma)]^{-2\Delta}< \frac{1}{1-\frac{1}{16\Delta_{\mathcal{K}}\Delta^2}}
\leq \frac{1}{1-1/16}.$$
  Hence, by Lemmas \ref{beta_property} and \ref{lemma:block_corr} in Appendix \ref{sec:bounding_bias_lemmas}, we have
\begin{align*}
\Corr(\widetilde{\pi}_{n},\beta)
\leq & (1-\epsilon_x^{2\Delta}[\epsilon_{\theta}(\sigma)]^{2\Delta})e^{2\beta r}\Delta^2+2\epsilon_x^{-2\Delta}[\epsilon_{\theta}(\sigma)]^{-2\Delta}\widetilde{\Corr}(\widetilde{\pi}_{n},\beta)\\
< & \frac{1}{16}+\frac{2}{8}\epsilon_x^{-2\Delta}[\epsilon_{\theta}(\sigma)]^{-2\Delta}\\
< & \frac{1}{16}+\frac{2}{8}\frac{1}{1-1/16}\\
< & \frac{1}{3}.
\end{align*}
\end{proof}

\begin{proposition}
\label{lemma:true_block_operatorerror}
Under Assumption \ref{assumption} in Section \ref{sec:beat_COD}, when condition \eqref{eqn:main_thm_condition} holds, for every $n\geq 1$, $K\in \mathcal{K}$ and $\K \subseteq K$, we have that
\begin{align*}
\|\mathsf{F}_n\widetilde{\pi}_{n-1}-\widetilde{\mathsf{F}}_n\widetilde{\pi}_{n-1} \|_{\K} \leq 4e^{-\beta}(1-\epsilon_x^{2\Delta}[\epsilon_{\theta}(\sigma)]^{2\Delta})e^{-\beta d(\K,\partial{K})}\card(\K),
\end{align*}
where $\beta$ is given in \eqref{eqn:beta_definition}.
\end{proposition}
\begin{proof}
Define $I=\{{n-1},n\}\times V$ and $\mathbb{S}=(\mathbb{X}\times \mathbb{\Theta})^2$. Fix $K\in \mathcal{K}$ and define
\begin{align*}
\rho(A)=&\dfrac{\splitdfrac{
\int\mathbbm{1}_A(x_{n-1},\theta_{n-1},x_n,\theta_n)\prod_{\omega\in V}f_{X_n^{\omega}|X_{n-1}}(x_n^{\omega}\mid x_{n-1}\,;\theta_n^{\omega})}{\splitdfrac{\times f_{\Theta_n^{\omega}|\Theta_{n-1}}(\theta_n^{\omega}\mid \theta_{n-1}\,; \sigma)f_{Y_n^{\omega}|X_n^{\omega}}(Y_n^{\omega}\mid x_n^{\omega}\,;\theta_n^{\omega})}{\times\psi(dx_n)\lambda(d\theta_n)\widetilde{\pi}_{n-1}(d x_{n-1},d\theta_{n-1})}}}
{\splitdfrac{
\int \prod_{\omega\in V}f_{X_n^{\omega}|X_{n-1}}(x_n^{\omega}\mid x_{n-1}\,;\theta_n^{\omega})}{\splitdfrac{\times f_{\Theta_n^{\omega}|\Theta_{n-1}}(\theta_n^{\omega}\mid \theta_{n-1}\,; \sigma)f_{Y_n^{\omega}|X_n^{\omega}}(Y_n^{\omega}\mid x_n^{\omega}\,;\theta_n^{\omega})}{\times\psi(dx_n)\lambda(d\theta_n)\widetilde{\pi}_{n-1}(d x_{n-1},d\theta_{n-1})}}},
\end{align*}
and
\begin{align*}
\widetilde{\rho}(A)=&\dfrac{\splitdfrac{
\int\mathbbm{1}_A(x_{n-1},\theta_{n-1},x_n,\theta_n)\prod_{v\in K}f_{X_n^v|X_{n-1}}(x_n^v\mid x_{n-1}\,;\theta_n^v)}{\splitdfrac{\times f_{\Theta_n^v|\Theta_{n-1}}(\theta_n^v\mid \theta_{n-1}\,; \sigma)\prod_{\omega\in V}f_{Y_n^{\omega}|X_n^{\omega}}(Y_n^{\omega}\mid x_n^{\omega}\,;\theta_n^{\omega})}{\times\psi(dx_n)\lambda(d\theta_n)\widetilde{\pi}_{n-1}(d x_{n-1},d\theta_{n-1})}}}
{\splitdfrac{
\int \prod_{v\in K}f_{X_n^v|X_{n-1}}(x_n^v\mid x_{n-1}\,;\theta_n^v)}{\splitdfrac{\times f_{\Theta_n^v|\Theta_{n-1}}(\theta_n^v\mid \theta_{n-1}\,; \sigma)\prod_{\omega\in V}f_{Y_n^{\omega}|X_n^{\omega}}(Y_n^{\omega}\mid x_n^{\omega}\,;\theta_n^{\omega})}{\times\psi(dx_n)\lambda(d\theta_n)\widetilde{\pi}_{n-1}(d x_{n-1},d\theta_{n-1})}}}.
\end{align*}
Then for any $\K \subseteq K$ and $K\subseteq V$, we have
$$\|\mathsf{F}_n\widetilde{\pi}_{n-1}-\widetilde{\mathsf{F}}_n\widetilde{\pi}_{n-1} \|_{\K}=\|\rho-\widetilde{\rho}\|_{\{n\}\times \K}.$$

In order to use Theorem \ref{thm:Dobrushin} (Dobrushin comparison theorem)  in Appendix \ref{sec:Existing_results} to bound $\|\rho-\widetilde{\rho}\|_{\{n\}\times \K}$, we need to bound $C_{ij}$ and $b_i$ with $i=(k,v)$ and $j=(k',v')$. 
Set 
$$\rho^i=\rho_{(x_{n-1},\theta_{n-1},x_n,\theta_n)}^i\quad\text{and}\quad\widetilde{\rho}^i=\widetilde{\rho}_{(x_{n-1},\theta_{n-1},x_n,\theta_n)}^i,$$
whose definitions are given in Theorem \ref{thm:Dobrushin} in Appendix \ref{sec:Existing_results}.
We display our discussions as follows:
\begin{itemize}
\item When $k={n-1}$, we have
\begin{align*}
\rho^i(A)
=\frac{\splitdfrac{\int\mathbbm{1}_A(x_{n-1}^v,\theta_{n-1}^v)\prod_{\omega\in N(v)}f_{X_n^{\omega}|X_{n-1}}(x_n^{\omega}\mid x_{n-1}\,;\theta_n^{\omega})}{\times f_{\Theta_n^{\omega}|\Theta_{n-1}}(\theta_n^{\omega}\mid \theta_{n-1}\,; \sigma)\widetilde{\pi}_{n-1}^v(d x_{n-1}^v,d\theta_{n-1}^v)}}{\splitdfrac{\int\prod_{\omega\in N(v)}f_{X_n^{\omega}|X_{n-1}}(x_n^{\omega}\mid x_{n-1}\,;\theta_n^{\omega})}{\times f_{\Theta_n^{\omega}|\Theta_{n-1}}(\theta_n^{\omega}\mid \theta_{n-1}\,; \sigma)\widetilde{\pi}_{n-1}^v(d x_{n-1}^v,d\theta_{n-1}^v)}}.
\end{align*}
We can see that $\rho^i=\widetilde{\pi}_{\chi_{n-1},\chi_{n}}^v$, according to the definition of $\mu_{\chi_{s-1},\chi_{s}}^v$ defined in \eqref{def:mu_s-1_s}. Therefore, if $k'={n-1}$, by the definition of $C_{vv'}^{\mu_{s-1}}$ in equation \eqref{eqn:C_vv'_mu}, 
                                                                                                                              we know that $C_{ij}\leq C_{vv'}^{\widetilde{\pi}_{n-1}}$. If $k'=n$, since
                                                                                                                              \begin{align*}
                                                                                                                              \rho^i(A)
                                                                                                                              &\geq \epsilon_x^2[\epsilon_{\theta}(\sigma)]^2\frac{\splitdfrac{\int\mathbbm{1}_A(x_{n-1}^v,\theta_{n-1}^v)\prod_{\omega\in N(v)\backslash \{v'\}}f_{X_n^{\omega}|X_{n-1}}(x_n^{\omega}\mid x_{n-1}\,;\theta_n^{\omega})}{\times f_{\Theta_n^{\omega}|\Theta_{n-1}}(\theta_n^{\omega}\mid \theta_{n-1}\,; \sigma)\widetilde{\pi}_{n-1}^v(d x_{n-1}^v,d\theta_{n-1}^v)}}{\splitdfrac{\int\prod_{\omega\in N(v)\backslash \{v'\}}f_{X_n^{\omega}|X_{n-1}}(x_n^{\omega}\mid x_{n-1}\,;\theta_n^{\omega})}{\times f_{\Theta_n^{\omega}|\Theta_{n-1}}(\theta_n^{\omega}\mid \theta_{n-1}\,; \sigma)\widetilde{\pi}_{n-1}^v(d x_{n-1}^v,d\theta_{n-1}^v)}},
                                                                                                                              \end{align*}
                                                                                                                              we have $C_{ij}\leq 1-\epsilon_x^2[\epsilon_{\theta}(\sigma)]^2$ if $v'\in N(v)$ by Theorem \ref{thm:VanHandel4.1} in Appendix \ref{sec:Existing_results}, and $C_{ij}=0$ otherwise. Hence, by the definition of $\Corr(\mu_{s-1},\beta)$ in \eqref{eqn:corr_definition}, we have
\begin{equation}
\label{eqn:true_block_operatorerror1}
\sum_{(k',v')\in I}e^{\beta |k'-k|}e^{\beta d(v,v')}C_{(k,v)(k',v')}\leq \Corr(\widetilde{\pi}_{n-1},\beta)+(1-\epsilon_x^2[\epsilon_{\theta}(\sigma)]^2)e^{\beta(r+1)}\Delta.
\end{equation}
Next we take care of $b_i$. 
When $k={n-1}$, we have
\begin{align*}
\widetilde{\rho}^i(A)=\frac{\splitdfrac{\int\mathbbm{1}_A(x_{n-1}^v,\theta_{n-1}^v)\prod_{\omega\in N(v)\cap K}f_{X_n^{\omega}|X_{n-1}}(x_n^{\omega}\mid x_{n-1}\,;\theta_n^{\omega})}{\times f_{\Theta_n^{\omega}|\Theta_{n-1}}(\theta_n^{\omega}\mid \theta_{n-1}\,; \sigma)\widetilde{\pi}_{n-1}^v(d x_{n-1}^v,d\theta_{n-1}^v)}}{\splitdfrac{\int\prod_{\omega\in N(v)\cap K}f_{X_n^{\omega}|X_{n-1}}(x_n^{\omega}\mid x_{n-1}\,;\theta_n^{\omega})}{\times f_{\Theta_n^{\omega}|\Theta_{n-1}}(\theta_n^{\omega}\mid \theta_{n-1}\,; \sigma)\widetilde{\pi}_{n-1}^v(d x_{n-1}^v,d\theta_{n-1}^v)}}.
\end{align*}
Note that, if $N(v)\subseteq K$, we have $\rho^i=\widetilde{\rho}^i$ which yields $b_i=0$ if $v\in \operatorname{int}(K)$. Further note that
\begin{align*}
\rho^i(A)
\geq \epsilon_x^{2\Delta}[\epsilon_{\theta}(\sigma)]^{2\Delta}\frac{\int\mathbbm{1}_A(x_{n-1}^v,\theta_{n-1}^v)\widetilde{\pi}_{n-1}^v(d x_{n-1}^v,d\theta_{n-1}^v)}{\int\widetilde{\pi}_{n-1}^v(d x_{n-1}^v,d\theta_{n-1}^v)},\\
\widetilde{\rho}^i(A)\geq \epsilon_x^{2\Delta}[\epsilon_{\theta}(\sigma)]^{2\Delta}\frac{\int\mathbbm{1}_A(x_{n-1}^v,\theta_{n-1}^v)\widetilde{\pi}_{n-1}^v(d x_{n-1}^v,d\theta_{n-1}^v)}{\int\widetilde{\pi}_{n-1}^v(d x_{n-1}^v,d\theta_{n-1}^v)},
\end{align*}
which by Theorem \ref{thm:VanHandel4.1} in Appendix \ref{sec:Existing_results} yields that $b_i=2(1-\epsilon_x^{2\Delta}[\epsilon_{\theta}(\sigma)]^{2\Delta})$ if $v\notin \operatorname{int}(K)$.
\smallskip

\item When $k=n$, we have
\begin{align*}
\rho^i(A)=&\dfrac{\splitdfrac{
\int\mathbbm{1}_A(x_n^v,\theta_n^v)f_{X_n^v|X_{n-1}}(x_n^v\mid x_{n-1}\,;\theta_n^v) f_{\Theta_n^v|\Theta_{n-1}}(\theta_n^v\mid \theta_{n-1}\,; \sigma)}{\times f_{Y_n^v|X_n^v}(Y_n^v\mid x_n^v\,;\theta_n^v)\psi^v(dx_n^v)\lambda^v(d\theta_n^v)}}
{\splitdfrac{
\int f_{X_n^v|X_{n-1}}(x_n^v\mid x_{n-1}\,;\theta_n^v) f_{\Theta_n^v|\Theta_{n-1}}(\theta_n^v\mid \theta_{n-1}\,; \sigma)}{\times f_{Y_n^v|X_n^v}(Y_n^v\mid x_n^v\,;\theta_n^v)\psi^v(dx_n^v)\lambda^v(d\theta_n^v)}}\\
&\geq \epsilon_x^2[\epsilon_{\theta}(\sigma)]^2\dfrac{
\int\mathbbm{1}_A(x_n^v,\theta_n^v) f_{Y_n^v|X_n^v}(Y_n^v\mid x_n^v\,;\theta_n^v)\psi^v(dx_n^v)\lambda^v(d\theta_n^v)}
{\int  f_{Y_n^v|X_n^v}(Y_n^v\mid x_n^v\,;\theta_n^v)\psi^v(dx_n^v)\lambda^v(d\theta_n^v)}.
\end{align*}
Hence, 
\begin{equation}
\label{eqn:true_block_operatorerror2}
\sum_{(k',v')\in I}e^{\beta |k-k'|}e^{\beta d(v,v')}C_{(k,v)(k',v')}\leq (1-\epsilon_x^2[\epsilon_{\theta}(\sigma)]^2)e^{\beta(r+1)}\Delta.
\end{equation}
Note that, if $v\in K$ we have $\rho^i=\widetilde{\rho}^i$ which yields $b_i=0$, otherwise given that
\begin{align*}
\widetilde{\rho}^i(A)=&\dfrac{
\int\mathbbm{1}_A(x_n^v,\theta_n^v) f_{Y_n^v|X_n^v}(Y_n^v\mid x_n^v\,;\theta_n^v)\psi^v(dx_n^v)\lambda^v(d\theta_n^v)}
{\int  f_{Y_n^v|X_n^v}(Y_n^v\mid x_n^v\,;\theta_n^v)\psi^v(dx_n^v)\lambda^v(d\theta_n^v)}\\
\geq &\epsilon_x^2[\epsilon_{\theta}(\sigma)]^2\dfrac{
\int\mathbbm{1}_A(x_n^v,\theta_n^v) f_{Y_n^v|X_n^v}(Y_n^v\mid x_n^v\,;\theta_n^v)\psi^v(dx_n^v)\lambda^v(d\theta_n^v)}
{\int  f_{Y_n^v|X_n^v}(Y_n^v\mid x_n^v\,;\theta_n^v)\psi^v(dx_n^v)\lambda^v(d\theta_n^v)}
\end{align*}
we have $b_i=2(1-\epsilon_x^{2}[\epsilon_{\theta}(\sigma)]^2)$  by Theorem \ref{thm:VanHandel4.1} in Appendix \ref{sec:Existing_results} .
\end{itemize}
\bigskip

Summing up \eqref{eqn:true_block_operatorerror1} for $k=n-1$ and \eqref{eqn:true_block_operatorerror2} for $k=n$, we have that
\begin{align*}
\max_{(k,v)\in I} \sum_{(k',v')\in I}e^{\beta |k-k'|}e^{\beta d(v,v')}C_{(k,v)(k',v')}\leq \Corr(\widetilde{\pi}_{n-1},\beta)+(1-\epsilon_x^2[\epsilon_{\theta}(\sigma)]^2)e^{\beta(r+1)}\Delta.
\end{align*}
Furthermore, by Lemma \ref{thm:onethird} in Appendix \ref{sec:bounding_bias_lemmas} we have that
\begin{align*}
\Corr(\widetilde{\pi}_{n},\beta)
< \frac{1}{3},
\end{align*}
and by Lemma \ref{beta_property} in Appendix \ref{sec:bounding_bias_lemmas} we have that
$$(1-\epsilon_x^2[\epsilon_{\theta}(\sigma)]^2)e^{\beta(r+1)}\Delta \leq \frac{1}{16}.$$
Hence,
\begin{align*}
\max_{(k,v)\in I} \sum_{(k',v')\in I}e^{\beta |k-k'|}e^{\beta d(v,v')}C_{(k,v)(k',v')}< 1/2.
\end{align*}
 Next, 
applying Theorem \ref{thm:Dobrushin} (Dobrushin comparison theorem)  in Appendix \ref{sec:Existing_results} and Theorem \ref{thm:VanHandel4.3}  in Appendix \ref{sec:Existing_results}, 
we obtain that
\begin{align*}
&\|\mathsf{F}_n\widetilde{\pi}_{n-1}-\widetilde{\mathsf{F}}_n\widetilde{\pi}_{n-1} \|_{\K}=\|\rho-\widetilde{\rho}\|_{\{n\}\times \K}\\
&=2(1-\epsilon_x^2[\epsilon_{\theta}(\sigma)]^2)\sum_{v\in \K}\left\{\sum_{v'\in V\backslash \operatorname{int}(K)}D_{(n,v)(n-1,v')}+ \sum_{v'\in V\backslash K}D_{(n,v)(n,v')}\right\}\\
&\leq  4e^{-\beta}(1-\epsilon_x^{2\Delta}[\epsilon_{\theta}(\sigma)]^{2\Delta})e^{-\beta d(\K,\partial{K})}\card(\K).
\end{align*}
\end{proof}

\begin{proposition}
\label{thm:error_prop_block}
Under Assumption \ref{assumption} in Section \ref{sec:beat_COD}, when condition \eqref{eqn:main_thm_condition} holds, we have that for every $\K \subseteq K$, $K \in \mathcal{K}$ and $s\in \{1,\ldots,n-1\}$,
\begin{align*}
 &\hspace*{-0.7cm}\| \mathsf{F}_n \cdots\mathsf{F}_{s+1}\mathsf{F}_s\widetilde{\pi}_{s-1}-\mathsf{F}_n \cdots\mathsf{F}_{s+1}\widetilde{\mathsf{F}}_s\widetilde{\pi}_{s-1} \|_{\K}\\
 \leq &\frac{48}{23}e^{-\beta(n-s)}\sum_{v\in \K}\max_{v'\in V}e^{-\beta d(v,v')} \sup_{x_s,x_{s+1}\in\mathbb{X}\atop
\theta_s,\theta_{s+1}\in\mathbb{\Theta} }\| (\mathsf{F}_s\widetilde{\pi}_{s-1})_{\chi_{s},\chi_{s+1}}^{v'}-(\widetilde{\mathsf{F}}_s\widetilde{\pi}_{s-1})_{\chi_{s},\chi_{s+1}}^{v'}\|,
\end{align*}
where $\beta$ is given in \eqref{eqn:beta_definition}.
\end{proposition}

\begin{proof}
Define $I=\{s,\ldots,n\}\times V$ and $\mathbb{S}=(\mathbb{X}\times \mathbb{\Theta})^{n-s+1}$.
Define 
$$\rho=\mathbf{P}^{\widetilde{\mathsf{F}}_s\widetilde{\pi}_{s-1}}[X_s,X_{s+1},\ldots,X_n\in \cdot \;, \Theta_s,\Theta_{s+1},\ldots,\Theta_n\in \cdot \mid Y_{s+1},\ldots, Y_n],$$
$$\widetilde{\rho}=\mathbf{P}^{\mathsf{F}_s\widetilde{\pi}_{s-1}}[X_s,X_{s+1},\cdots,X_n\in \cdot \;, \Theta_s,\Theta_{s+1},\cdots,\Theta_n\in \cdot \mid Y_{s+1},\ldots, Y_n].$$
Then we have 
$$\| \mathsf{F}_n \cdots\mathsf{F}_{s+1}\widetilde{\mathsf{F}}_s\widetilde{\pi}_{s-1}-\mathsf{F}_n \cdots\mathsf{F}_{s+1}\mathsf{F}_s\widetilde{\pi}_{s-1} \|_{\K}=\| \rho-\widetilde{\rho} \|_{\{n\}\times {\K}}.$$
In order to use Theorem \ref{thm:Dobrushin} (Dobrushin comparison theorem)  in Appendix \ref{sec:Existing_results} to bound $\| \rho-\widetilde{\rho} \|_{\{n\}\times {\K}}$, we need to bound $C_{ij}$ and $b_i$ with $i=(k,v)$ and $j=(k',v')$. 
Set 
$$\rho^i=\rho_{(x_{s},\ldots,x_n,\theta_{s},\ldots,\theta_n)}^i\quad\text{and}\quad\widetilde{\rho}^i=\widetilde{\rho}_{(x_{s},\ldots,x_n,\theta_{s},\ldots,\theta_n)}^i,$$
whose definitions are given in Theorem \ref{thm:Dobrushin} in Appendix \ref{sec:Existing_results}.
We display our discussions as follows:\\
\begin{itemize}
\item When $k=s$, we have 
\begin{align*}
\rho^i(A)=\widetilde{\pi}_{\chi_{s},\chi_{s+1}}^v(A)=&\dfrac{\splitdfrac{
\int\mathbbm{1}_A(x_s^v,\theta_s^v)\prod_{\omega\in N(v)}f_{X_{s+1}^{\omega}|X_s}(x_{s+1}^{\omega}\mid x_s\,;\theta_{s+1}^{\omega})}{\times f_{\Theta_{s+1}^{\omega}|\Theta_{s}}(\theta_{s+1}^{\omega}\mid \theta_s\,; \sigma)\widetilde{\pi}_{\chi_{s}}^v(d x_s^v,d\theta_s^v)}}
{\splitdfrac{
\int \prod_{\omega\in N(v)}f_{X_{s+1}^{\omega}|X_s}(x_{s+1}^{\omega}\mid x_s\,;\theta_{s+1}^{\omega})}{\times f_{\Theta_{s+1}^{\omega}|\Theta_{s}}(\theta_{s+1}^{\omega}\mid \theta_s\,; \sigma)\widetilde{\pi}_{\chi_{s}}^v(d x_s^v,d\theta_s^v)}}\\
\geq &\epsilon_x^{2\Delta}[\epsilon_{\theta}(\sigma)]^{2\Delta}\widetilde{\pi}_{\chi_{s}}^v(A)
\end{align*}
according to the definition of $\mu_{\chi_{s-1},\chi_{s}}^v$ given in \eqref{def:mu_s-1_s}.
Therefore, when $k'=s$, by the definition of $C_{vv'}^{\widetilde{\pi}_s}$ in equation \eqref{eqn:C_vv'_mu}, we know that $C_{ij}\leq C_{vv'}^{\widetilde{\pi}_s}$. 
When $k'=s+1$, 
we have $C_{ij}\leq 1-\epsilon_x^{2\Delta}[\epsilon_{\theta}(\sigma)]^{2\Delta}$ if $v'\in N(v)$ by Theorem \ref{thm:VanHandel4.1} in Appendix \ref{sec:Existing_results}, and $C_{ij}=0$ otherwise. Hence, by the definition of $\widetilde{\Corr}(\mu_{s-1},\beta)$ in \eqref{eqn:corr_definition},
\begin{align}
\label{eqn:error_prop_block1}
\sum_{(k',v')\in I}e^{\beta |k-k'|}e^{\beta d(v,v')}C_{(k,v)(k',v')}\leq \Corr(\widetilde{\pi}_s,\beta)+(1-\epsilon_x^{2\Delta}[\epsilon_{\theta}(\sigma)]^{2\Delta})e^{\beta(r+1)}\Delta.
\end{align}
\smallskip

\item When $k\in \{s+1,\ldots,n-1\}$, we have 
\begin{align*}
\rho^i(A)=&\dfrac{\splitdfrac{
\int\mathbbm{1}_A(x_k^v,\theta_k^v)f_{X_k^v|X_{k-1}}(x_k^v\mid x_{k-1}\,;\theta_k^v)f_{\Theta_{k}^{v}|\Theta_{k-1}}(\theta_k^v\mid \theta_{k-1}\,; \sigma)}{\splitdfrac{\times \prod_{\omega\in N(v)}f_{X_{k+1}^{\omega}|X_k}(x_{k+1}^{\omega}\mid x_k\,;\theta_{k+1}^{\omega})f_{\Theta_{k+1}^{\omega}|\Theta_{k}}(\theta_{k+1}^{\omega}\mid \theta_k\,; \sigma)}{\times f_{Y_k^v|X_k^v}(Y_k^v\mid x_k^v\,;\theta_k^v)\psi^v(dx_k^v)\lambda^v(d\theta_k^v)}}}
{\splitdfrac{
\int f_{X_k^v|X_{k-1}}(x_k^v\mid x_{k-1}\,;\theta_k^v)f_{\Theta_{k}^{v}|\Theta_{k-1}}(\theta_k^v\mid \theta_{k-1}\,; \sigma)}{\splitdfrac{\times \prod_{\omega\in N(v)}f_{X_{k+1}^{\omega}|X_k}(x_{k+1}^{\omega}\mid x_k\,;\theta_{k+1}^{\omega})f_{\Theta_{k+1}^{\omega}|\Theta_{k}}(\theta_{k+1}^{\omega}\mid \theta_k\,; \sigma)}{\times f_{Y_k^v|X_k^v}(Y_k^v\mid x_k^v\,;\theta_k^v)\psi^v(dx_k^v)\lambda^v(d\theta_k^v)}}}.
\end{align*}
Note that $x_k^{v}$ depends on $x_{k-1}^{v'}$ if $v'\in N(v)$, and
\begin{align*}
\rho^i(A)\geq \epsilon_x^2[\epsilon_{\theta}(\sigma)]^2\dfrac{\splitdfrac{
\int\mathbbm{1}_A(x_k^v,\theta_k^v)  \prod_{\omega\in N(v)}f_{X_{k+1}^{\omega}|X_k}(x_{k+1}^{\omega}\mid x_k\,;\theta_{k+1}^{\omega})}{\splitdfrac{\times f_{\Theta_{k+1}^{\omega}|\Theta_{k}}(\theta_{k+1}^{\omega}\mid \theta_k\,; \sigma)f_{Y_k^v|X_k^v}(Y_k^v\mid x_k^v\,;\theta_k^v)}{\times\psi^v(dx_k^v)\lambda^v(d\theta_k^v)}}}
{\splitdfrac{
\int  \prod_{\omega\in N(v)}f_{X_{k+1}^{\omega}|X_k}(x_{k+1}^{\omega}\mid x_k\,;\theta_{k+1}^{\omega})}{\splitdfrac{\times f_{\Theta_{k+1}^{\omega}|\Theta_{k}}(\theta_{k+1}^{\omega}\mid \theta_k\,; \sigma)f_{Y_k^v|X_k^v}(Y_k^v\mid x_k^v\,;\theta_k^v)}{\times\psi^v(dx_k^v)\lambda^v(d\theta_k^v)}}}.
\end{align*}
Further note that $x_k^{v}$ depends on $x_{k+1}^{v'}$ if $v'\in N(v)$, $x_k^{v}$ depends on $x_{k}^{v'}$ if $v'\in \cup_{\omega \in N(v)}N(\omega)$, and 
\begin{align*}
\rho^i(A)\geq \epsilon_x^{2\Delta}[\epsilon_{\theta}(\sigma)]^{2\Delta}\dfrac{\splitdfrac{
\int\mathbbm{1}_A(x_k^v,\theta_k^v)f_{X_k^v|X_{k-1}}(x_k^v\mid x_{k-1}\,;\theta_k^v)}{\splitdfrac{\times f_{\Theta_{k}^{v}|\Theta_{k-1}}(\theta_k^v\mid \theta_{k-1}\,; \sigma) f_{Y_k^v|X_k^v}(Y_k^v\mid x_k^v\,;\theta_k^v)}{\times\psi^v(dx_k^v)\lambda^v(d\theta_k^v)}}}
{\splitdfrac{
\int f_{X_k^v|X_{k-1}}(x_k^v\mid x_{k-1}\,;\theta_k^v)}{\splitdfrac{\times f_{\Theta_{k}^{v}|\Theta_{k-1}}(\theta_k^v\mid \theta_{k-1}\,; \sigma) f_{Y_k^v|X_k^v}(Y_k^v\mid x_k^v\,;\theta_k^v)}{\times\psi^v(dx_k^v)\lambda^v(d\theta_k^v)}}}.
\end{align*}

By Theorem \ref{thm:VanHandel4.1} in Appendix \ref{sec:Existing_results} we have $C_{ij}\leq 1-\epsilon_x^{2}[\epsilon_{\theta}(\sigma)]^{2}$, if $k'=k-1$ and $v'\in N(v)$; $C_{ij}\leq 1-\epsilon_x^{2\Delta}[\epsilon_{\theta}(\sigma)]^{2\Delta}$, if $k'=k+1$ and $v'\in N(v)$; $C_{ij}\leq 1-\epsilon_x^{2\Delta}[\epsilon_{\theta}(\sigma)]^{2\Delta}$, if $k'=k$ and $v'\in \cup_{\omega \in N(v)}N(\omega)$; $C_{ij}=0$, otherwise. Therefore
\begin{align}
\label{eqn:error_prop_block2}
\sum_{(k',v')\in I}e^{\beta |k-k'|}e^{\beta d(v,v')}C_{(k,v)(k',v')}
&\leq (1-\epsilon_x^{2\Delta}[\epsilon_{\theta}(\sigma)]^{2\Delta})(e^{2\beta r}\Delta^2+2e^{\beta (r+1)}\Delta)\nonumber\\
&\leq 3(1-\epsilon_x^{2\Delta}[\epsilon_{\theta}(\sigma)]^{2\Delta})e^{2\beta r}\Delta^2.
\end{align}
 \smallskip

\item When $k=n$, note that $x_n^{v}$ only depends on $x_{n-1}^{v'}$ where $v'\in N(v)$, and
\begin{align*}
\rho^i(A)=&\dfrac{\splitdfrac{
\int\mathbbm{1}_A(x_n^v,\theta_n^v)f_{X_n^v|X_{n-1}}(x_n^v\mid x_{n-1}\,;\theta_n^v)f_{\Theta_n^v|\Theta_{n-1}}(\theta_n^v\mid \theta_{n-1}\,; \sigma)}{\times f_{Y_n^v|X_n^v}(Y_n^v\mid x_n^v\,;\theta_n^v)\psi^v(dx_n^v)\lambda^v(d\theta_n^v)}}
{\splitdfrac{
\int f_{X_n^v|X_{n-1}}(x_n^v\mid x_{n-1}\,;\theta_n^v)f_{\Theta_n^v|\Theta_{n-1}}(\theta_n^v\mid \theta_{n-1}\,; \sigma)}{\times f_{Y_n^v|X_n^v}(Y_n^v\mid x_n^v\,;\theta_n^v)\psi^v(dx_n^v)\lambda^v(d\theta_n^v)}}\\
\geq &\epsilon_x^2[\epsilon_{\theta}(\sigma)]^2\dfrac{
\int\mathbbm{1}_A(x_n^v,\theta_n^v) f_{Y_n^v|X_n^v}(Y_n^v\mid x_n^v\,;\theta_n^v)\psi^v(dx_n^v)\lambda^v(d\theta_n^v)}
{\int f_{Y_n^v|X_n^v}(Y_n^v\mid x_n^v\,;\theta_n^v)\psi^v(dx_n^v)\lambda^v(d\theta_n^v)}.
\end{align*}
By Theorem \ref{thm:VanHandel4.1} in Appendix \ref{sec:Existing_results} we have $C_{ij}\leq 1-\epsilon_x^{2}[\epsilon_{\theta}(\sigma)]^2$ if $k'=n-1$ and $v'\in N(v)$; $C_{ij}=0$, otherwise.
Therefore
\begin{align}
\label{eqn:error_prop_block3}
\sum_{(k',v')\in I}e^{\beta |k-k'|}e^{\beta d(v,v')}C_{(k,v)(k',v')}&\leq (1-\epsilon_x^{2}[\epsilon_{\theta}(\sigma)]^{2})e^{\beta (r+1)}\Delta.
\end{align}
\end{itemize}
\bigskip

Summing up \eqref{eqn:error_prop_block1} for $k=s$, \eqref{eqn:error_prop_block2} for $k\in \{s+1,\ldots,n-1\}$, and \eqref{eqn:error_prop_block3} for $k=n$, by Lemma \ref{beta_property} in Appendix \ref{sec:bounding_bias_lemmas}, we have
\begin{align*}
&\hspace*{-2cm}\max_{(k,v)\in I}\sum_{(k',v')\in I}e^{\beta |k-k'|}e^{\beta d(v,v')}C_{(k,v)(k',v')}\\
&\leq \Corr(\widetilde{\pi}_s,\beta)+3(1-\epsilon_x^{2\Delta}[\epsilon_{\theta}(\sigma)]^{2\Delta})e^{2\beta r}\Delta^2.
\end{align*}
Furthermore, by Lemma \ref{thm:onethird} in Appendix \ref{sec:bounding_bias_lemmas} we have that
\begin{align*}
\Corr(\widetilde{\pi}_{s},\beta)
< \frac{1}{3},
\end{align*}
and by Lemma \ref{beta_property} in Appendix \ref{sec:bounding_bias_lemmas} we have that
$$(1-\epsilon_x^{2\Delta}[\epsilon_{\theta}(\sigma)]^{2\Delta})e^{2\beta r}\Delta^2\leq \frac{1}{16}.$$
Hence,
\begin{align*}
\max_{(k,v)\in I}\sum_{(k',v')\in I}e^{\beta |k-k'|}e^{\beta d(v,v')}C_{(k,v)(k',v')}<\frac{1}{3} +\frac{3}{16}=\frac{25}{48}.
\end{align*}
By Theorem \ref{thm:VanHandel4.3}  in Appendix \ref{sec:Existing_results} we have 
\begin{align}
\label{eqn:bound2}
\max_{(k,v)\in I}\sum_{(k',v')\in I}e^{\beta |k-k'|}e^{\beta d(v,v')}D_{(k,v)(k',v')}< \frac{48}{23}.
\end{align}
Note that, $\rho^i=\widetilde{\rho}^i$ when $i=(k,v)$ with $k>s$. 
For $k=s$, we have $\widetilde{\rho}^i=(\mathsf{F}_s\widetilde{\pi}_{s-1})_{\chi_s,\chi_{s+1}}^{v}$ and $\rho^i=(\widetilde{\mathsf{F}}_s\widetilde{\pi}_{s-1})_{\chi_s,\chi_{s+1}}^{v}$. By Theorem \ref{thm:VanHandel4.1} in Appendix \ref{sec:Existing_results}, we have
\begin{align*}
&\hspace*{-0.5cm}\| \mathsf{F}_n \cdots\mathsf{F}_{s+1}\widetilde{\mathsf{F}}_s\widetilde{\pi}_{s-1}-\mathsf{F}_n \cdots\mathsf{F}_{s+1}\mathsf{F}_s\widetilde{\pi}_{s-1}\|_{\K}\\
=&\| \rho-\widetilde{\rho} \|_{\{n\}\times {\K}}\\
\leq & \sum_{v\in \K}\sum_{v'\in V}D_{(n,v)(s,v')}\sup_{x_s,x_{s+1}\in \mathbb{X}\atop
\theta_s,\theta_{s+1}\in \mathbb{\Theta}}\|(\widetilde{\mathsf{F}}_s\widetilde{\pi}_{s-1})_{\chi_{s},\chi_{s+1}}^{v'}- (\mathsf{F}_s\widetilde{\pi}_{s-1})_{\chi_{s},\chi_{s+1}}^{v'}\|\\
=& \sum_{v\in \K} e^{-\beta (n-s)}  \sum_{v'\in V} e^{\beta ((n-s)+d(v,v'))}D_{(n,v)(s,v')}\\
&\quad\quad\quad\quad\quad\quad\quad\quad\times e^{-\beta d(v,v')}\sup_{x_s,x_{s+1}\in \mathbb{X}\atop
\theta_s,\theta_{s+1}\in \mathbb{\Theta}}\| (\widetilde{\mathsf{F}}_s\widetilde{\pi}_{s-1})_{\chi_{s},\chi_{s+1}}^{v'}-(\mathsf{F}_s\widetilde{\pi}_{s-1})_{\chi_{s},\chi_{s+1}}^{v'}\|\\
\leq &\frac{48}{23} e^{-\beta (n-s)} \sum_{v\in \K} \max_{v'\in V} e^{-\beta d(v,v')} \sup_{x_s,x_{s+1}\in \mathbb{X}\atop
\theta_s,\theta_{s+1}\in \mathbb{\Theta}}\| (\widetilde{\mathsf{F}}_s\widetilde{\pi}_{s-1})_{\chi_{s},\chi_{s+1}}^{v'}-(\mathsf{F}_s\widetilde{\pi}_{s-1})_{\chi_{s},\chi_{s+1}}^{v'}\|,
\end{align*}
where we used \eqref{eqn:bound2} in the last inequality.
\end{proof}

\begin{proposition}
\label{thm:one_step_block_bound}
Under Assumption \ref{assumption} in Section \ref{sec:beat_COD}, when condition \eqref{eqn:main_thm_condition} holds, we have that for every $s\geq 1$, $K'\in \mathcal{K}$, and $v'\in K'$,
  $$\sup_{x_s,x_{s+1}\in \mathbb{X}\atop
    \theta_s,\theta_{s+1}\in \mathbb{\Theta}}\| (\mathsf{F}_s\widetilde{\pi}_{s-1})_{\chi_{s},\chi_{s+1}}^{v'}-(\widetilde{\mathsf{F}}_s\widetilde{\pi}_{s-1})_{\chi_{s},\chi_{s+1}}^{v'}\|
    \leq \frac{96}{29}e^{-\beta}(1-\epsilon_x^{2\Delta}[\epsilon_{\theta}(\sigma)]^{2\Delta})e^{-\beta d(v',\partial K')},$$
    where $\beta$ is given in \eqref{eqn:beta_definition}.
  \end{proposition}
  
  \begin{proof}
  According to the expressions of
  $(\mathsf{F}_s\nu_{s-1})_{\chi_s,\chi_{s+1}}^{v}$  given in \eqref{def:F_s_rho_s1}
  and $(\widetilde{\mathsf{F}}_s\nu_{s-1})_{\chi_s,\chi_{s+1}}^{v}$ given in \eqref{def:widetildeF_s_rho_s1}, we have
  \begin{align*}
  &\hspace*{-0.2cm}(\mathsf{F}_s\widetilde{\pi}_{s-1})_{\chi_{s},\chi_{s+1}}^{v'}(A)\\
&=\dfrac{\splitdfrac{
\int\mathbbm{1}_A(x_s^{v'},\theta_s^{v'})\prod_{\omega\in V}f_{X_s^{\omega}|X_{s-1}}(x_s^{\omega}\mid x_{s-1}\,;\theta_s^{\omega})f_{\Theta_s^{\omega}|\Theta_{s-1}}(\theta_s^{\omega}\mid \theta_{s-1}\,; \sigma)}{\splitdfrac{\times f_{Y_s^{v'}|X_s^{v'}}(Y_s^{v'}\mid x_s^{v'}\,;\theta_s^{v'})\prod_{u\in N(v')}f_{X_{s+1}^u|X_s}(x_{s+1}^u\mid x_s\,;\theta_{s+1}^u)}{\times f_{\Theta_{s+1}^{u}|\Theta_{s}}(\theta_{s+1}^{u}\mid \theta_{s}\,; \sigma)\psi^{v'}(dx_s^{v'})\lambda^{v'}(d\theta_s^{v'})\widetilde{\pi}_{s-1}(dx_{s-1},d\theta_{s-1})}}}
{\splitdfrac{
\int \prod_{\omega\in V}f_{X_s^{\omega}|X_{s-1}}(x_s^{\omega}\mid x_{s-1}\,;\theta_s^{\omega})f_{\Theta_s^{\omega}|\Theta_{s-1}}(\theta_s^{\omega}\mid \theta_{s-1}\,; \sigma)}{\splitdfrac{\times f_{Y_s^{v'}|X_s^{v'}}(Y_s^{v'}\mid x_s^{v'}\,;\theta_s^{v'})\prod_{u\in N(v')}f_{X_{s+1}^u|X_s}(x_{s+1}^u\mid x_s\,;\theta_{s+1}^u)}{\times f_{\Theta_{s+1}^{u}|\Theta_{s}}(\theta_{s+1}^{u}\mid \theta_{s}\,; \sigma)\psi^{v'}(dx_s^{v'})\lambda^{v'}(d\theta_s^{v'})\widetilde{\pi}_{s-1}(dx_{s-1},d\theta_{s-1})}}}
\end{align*}
and
\begin{align*}
&\hspace*{-0.2cm}(\widetilde{\mathsf{F}}_s\widetilde{\pi}_{s-1})_{\chi_{s},\chi_{s+1}}^{v'}(A)\\
                                                                                                                                                                                                                                                                                                                                                                                                                                                          &=\dfrac{\splitdfrac{
                                                                                                                                                                                                                                                                                                                                                                                                                                                            \int\mathbbm{1}_A(x_s^{v'},\theta_s^{v'})\prod_{\omega\in K'}f_{X_s^{\omega}|X_{s-1}}(x_s^{\omega}\mid x_{s-1}\,;\theta_s^{\omega})f_{\Theta_s^{\omega}|\Theta_{s-1}}(\theta_s^{\omega}\mid \theta_{s-1}\,; \sigma)}{\splitdfrac{\times f_{Y_s^{v'}|X_s^{v'}}(Y_s^{v'}\mid x_s^{v'}\,;\theta_s^{v'})\prod_{u\in N(v')}f_{X_{s+1}^u|X_s}(x_{s+1}^u\mid x_s\,;\theta_{s+1}^u)}{\times f_{\Theta_{s+1}^{u}|\Theta_{s}}(\theta_{s+1}^{u}\mid \theta_{s}\,; \sigma)\psi^{v'}(dx_s^{v'})\lambda^{v'}(d\theta_s^{v'})\widetilde{\pi}_{s-1}(dx_{s-1},d\theta_{s-1})}}}
{\splitdfrac{
\int \prod_{\omega\in K'}f_{X_s^{\omega}|X_{s-1}}(x_s^{\omega}\mid x_{s-1}\,;\theta_s^{\omega})f_{\Theta_s^{\omega}|\Theta_{s-1}}(\theta_s^{\omega}\mid \theta_{s-1}\,; \sigma)}{\splitdfrac{\times f_{Y_s^{v'}|X_s^{v'}}(Y_s^{v'}\mid x_s^{v'}\,;\theta_s^{v'})\prod_{u\in N(v')}f_{X_{s+1}^u|X_s}(x_{s+1}^u\mid x_s\,;\theta_{s+1}^u)}{\times f_{\Theta_{s+1}^{u}|\Theta_{s}}(\theta_{s+1}^{u}\mid \theta_{s}\,; \sigma)\psi^{v'}(dx_s^{v'})\lambda^{v'}(d\theta_s^{v'})\widetilde{\pi}_{s-1}(dx_{s-1},d\theta_{s-1})}}}.
                                                                                                                                                                                                                                                                                                                                                                                                                                                          \end{align*}
                                                                                                                                                                                                                                                                                                                                                                                                                                                          
                                                                                                                                                                                                                                                                                                                                                                                                                                                          Define $I=(\{s-1\}\times V)\cup (s,{v'})$ and $\mathbb{S}=(\mathbb{X}\times \mathbb{\Theta})\times
(\mathbb{X}^{v'}\times \mathbb{\Theta}^{v'})$, and the probability measures on $\mathbb{S}$ as follows:
\begin{align*}
\rho(A)
&=\dfrac{\splitdfrac{
\int\mathbbm{1}_A(x_{s-1}, \theta_{s-1}, x_s^{v'},\theta_s^{v'})\prod_{\omega\in V}f_{X_s^{\omega}|X_{s-1}}(x_s^{\omega}\mid x_{s-1}\,;\theta_s^{\omega})}{\splitdfrac{\times f_{\Theta_s^{\omega}|\Theta_{s-1}}(\theta_s^{\omega}\mid \theta_{s-1}\,; \sigma) f_{Y_s^{v'}|X_s^{v'}}(Y_s^{v'}\mid x_s^{v'}\,;\theta_s^{v'})}{\splitdfrac{\times \prod_{u\in N(v')}f_{X_{s+1}^u|X_s}(x_{s+1}^u\mid x_s\,;\theta_{s+1}^u) f_{\Theta_{s+1}^{u}|\Theta_{s}}(\theta_{s+1}^{u}\mid \theta_{s}\,; \sigma)}{\times\psi^{v'}(dx_s^{v'})\lambda^{v'}(d\theta_s^{v'})\widetilde{\pi}_{s-1}(dx_{s-1},d\theta_{s-1})}}}}
{\splitdfrac{
\int\prod_{\omega\in V}f_{X_s^{\omega}|X_{s-1}}(x_s^{\omega}\mid x_{s-1}\,;\theta_s^{\omega})}{\splitdfrac{\times f_{\Theta_s^{\omega}|\Theta_{s-1}}(\theta_s^{\omega}\mid \theta_{s-1}\,; \sigma) f_{Y_s^{v'}|X_s^{v'}}(Y_s^{v'}\mid x_s^{v'}\,;\theta_s^{v'})}{\splitdfrac{\times \prod_{u\in N(v')}f_{X_{s+1}^u|X_s}(x_{s+1}^u\mid x_s\,;\theta_{s+1}^u) f_{\Theta_{s+1}^{u}|\Theta_{s}}(\theta_{s+1}^{u}\mid \theta_{s}\,; \sigma)}{\times\psi^{v'}(dx_s^{v'})\lambda^{v'}(d\theta_s^{v'})\widetilde{\pi}_{s-1}(dx_{s-1},d\theta_{s-1})}}}}
\end{align*}
and
\begin{align*}
\widetilde{\rho}(A)&=\dfrac{\splitdfrac{
\int\mathbbm{1}_A(x_{s-1}, \theta_{s-1}, x_s^{v'},\theta_s^{v'})\prod_{\omega\in K'}f_{X_s^{\omega}|X_{s-1}}(x_s^{\omega}\mid x_{s-1}\,;\theta_s^{\omega})}{\splitdfrac{\times f_{\Theta_s^{\omega}|\Theta_{s-1}}(\theta_s^{\omega}\mid \theta_{s-1}\,; \sigma) f_{Y_s^{v'}|X_s^{v'}}(Y_s^{v'}\mid x_s^{v'}\,;\theta_s^{v'})}{\splitdfrac{\times \prod_{u\in N(v')}f_{X_{s+1}^u|X_s}(x_{s+1}^u\mid x_s\,;\theta_{s+1}^u) f_{\Theta_{s+1}^{u}|\Theta_{s}}(\theta_{s+1}^{u}\mid \theta_{s}\,; \sigma)}{\times\psi^{v'}(dx_s^{v'})\lambda^{v'}(d\theta_s^{v'})\widetilde{\pi}_{s-1}(dx_{s-1},d\theta_{s-1})}}}}
{\splitdfrac{
  \int\prod_{\omega\in K'}f_{X_s^{\omega}|X_{s-1}}(x_s^{\omega}\mid x_{s-1}\,;\theta_s^{\omega})}{\splitdfrac{\times f_{\Theta_s^{\omega}|\Theta_{s-1}}(\theta_s^{\omega}\mid \theta_{s-1}\,; \sigma) f_{Y_s^{v'}|X_s^{v'}}(Y_s^{v'}\mid x_s^{v'}\,;\theta_s^{v'})}{\splitdfrac{\times \prod_{u\in N(v')}f_{X_{s+1}^u|X_s}(x_{s+1}^u\mid x_s\,;\theta_{s+1}^u) f_{\Theta_{s+1}^{u}|\Theta_{s}}(\theta_{s+1}^{u}\mid \theta_{s}\,; \sigma)}{\times\psi^{v'}(dx_s^{v'})\lambda^{v'}(d\theta_s^{v'})\widetilde{\pi}_{s-1}(dx_{s-1},d\theta_{s-1})}}}}.
\end{align*}
Then we have
\begin{align*}
\|(\mathsf{F}_s\widetilde{\pi}_{s-1})_{\chi_{s},\chi_{s+1}}^{v'}-(\widetilde{\mathsf{F}}_s\widetilde{\pi}_{s-1})_{\chi_{s},\chi_{s+1}}^{v'}\|=\|\rho-\widetilde{\rho}\|_{(s,v')}.
                                                                                                                                                                                                                                                                                                                                                                                                                                                           \end{align*}
                                                                                                                                                                                                                                                                                                                                                                                                                                                           
                                                                                                                                                                                                                                                                                                                                                                                                                                                           In order to use Theorem \ref{thm:Dobrushin} (Dobrushin comparison theorem)  in Appendix \ref{sec:Existing_results} to bound $\|\rho-\widetilde{\rho}\|_{(s,v')}$, we need to bound $C_{ij}$ and $b_i$ with $i=(\overline{k},\overline{v})$ and $j=(\ooverline{k},\ooverline{v})$. 
Set 
$$\rho^i=\rho_{(x_{s-1}, \theta_{s-1}, x_s^{v'},\theta_s^{v'})}^i\quad\text{and}\quad\widetilde{\rho}^i=\widetilde{\rho}_{(x_{s-1}, \theta_{s-1}, x_s^{v'},\theta_s^{v'})}^i,$$
whose definitions are given in Theorem \ref{thm:Dobrushin} in Appendix \ref{sec:Existing_results}.
We display our discussions as follows:
\begin{itemize}
\item When $\overline{k}=s-1$, we have
\begin{align*}
\rho^i(A)
=&\dfrac{\splitdfrac{
\int\mathbbm{1}_A(x_{s-1}^{\overline{v}}, \theta_{s-1}^{\overline{v}})\prod_{\omega\in N(\overline{v})}f_{X_s^{\omega}|X_{s-1}}(x_s^{\omega}\mid x_{s-1}\,;\theta_s^{\omega})}{\times f_{\Theta_s^{\omega}|\Theta_{s-1}}(\theta_s^{\omega}\mid \theta_{s-1}\,; \sigma)\widetilde{\pi}_{\chi_{s-1}}^{\overline{v}}(dx_{s-1}^{\overline{v}},d\theta_{s-1}^{\overline{v}})}}
{\splitdfrac{
\int \prod_{\omega\in N(\overline{v})}f_{X_s^{\omega}|X_{s-1}}(x_s^{\omega}\mid x_{s-1}\,;\theta_s^{\omega})}{\times f_{\Theta_s^{\omega}|\Theta_{s-1}}(\theta_s^{\omega}\mid \theta_{s-1}\,; \sigma)\widetilde{\pi}_{\chi_{s-1}}^{\overline{v}}(dx_{s-1}^{\overline{v}},d\theta_{s-1}^{\overline{v}})}}
\end{align*}
and
\begin{align*}
\widetilde{\rho}^i(A)
=&\dfrac{\splitdfrac{
\int\mathbbm{1}_A(x_{s-1}^{\overline{v}}, \theta_{s-1}^{\overline{v}})\prod_{\omega\in N(\overline{v})\cap K'}f_{X_s^{\omega}|X_{s-1}}(x_s^{\omega}\mid x_{s-1}\,;\theta_s^{\omega})}{\times f_{\Theta_s^{\omega}|\Theta_{s-1}}(\theta_s^{\omega}\mid \theta_{s-1}\,; \sigma)\widetilde{\pi}_{\chi_{s-1}}^{\overline{v}}(dx_{s-1}^{\overline{v}},d\theta_{s-1}^{\overline{v}})}}
  {\splitdfrac{
    \int \prod_{\omega\in N(\overline{v})\cap K'}f_{X_s^{\omega}|X_{s-1}}(x_s^{\omega}\mid x_{s-1}\,;\theta_s^{\omega})}{\times f_{\Theta_s^{\omega}|\Theta_{s-1}}(\theta_s^{\omega}\mid \theta_{s-1}\,; \sigma)\widetilde{\pi}_{\chi_{s-1}}^{\overline{v}}(dx_{s-1}^{\overline{v}},d\theta_{s-1}^{\overline{v}})}},
\end{align*}
where $\widetilde{\pi}_{\chi_{s-1}}^{\overline{v}}$ is defined according to the definition of $\mu_{\chi_{s-1}}^v$ given in \eqref{def:mu_s-1_v}.
Therefore, by the definition of $\mu_{\chi_{s-1},\chi_{s}}^v$ given in \eqref{def:mu_s-1_s}, we have 
$\rho^i=\widetilde{\pi}_{\chi_{s-1},\chi_{s}}^{\overline{v}}$. Furthermore, by the definition of $C_{vv'}^{\mu_{s-1}}$ in equation \eqref{eqn:C_vv'_mu}, we know that $C_{ij}\leq C_{\overline{v}\,\ooverline{v}}^{\widetilde{\pi}_{s-1}}$ if $\ooverline{k}=s-1$. 
If $\ooverline{k}=s$, since
\begin{align*}
\rho^i(A)
\geq \epsilon_x^2[\epsilon_{\theta}(\sigma)]^2\dfrac{\splitdfrac{
\int\mathbbm{1}_A(x_{s-1}^{\overline{v}}, \theta_{s-1}^{\overline{v}})\prod_{\omega\in N(\overline{v})\backslash \{v'\}}f_{X_s^{\omega}|X_{s-1}}(x_s^{\omega}\mid x_{s-1}\,;\theta_s^{\omega})}{\times f_{\Theta_s^{\omega}|\Theta_{s-1}}(\theta_s^{\omega}\mid \theta_{s-1}\,; \sigma)\widetilde{\pi}_{\chi_{s-1}}^{\overline{v}}(dx_{s-1}^{\overline{v}},d\theta_{s-1}^{\overline{v}})}}
{\splitdfrac{
  \int \prod_{\omega\in N(\overline{v})\backslash \{v'\}}f_{X_s^{\omega}|X_{s-1}}(x_s^{\omega}\mid x_{s-1}\,;\theta_s^{\omega})}{\times f_{\Theta_s^{\omega}|\Theta_{s-1}}(\theta_s^{\omega}\mid \theta_{s-1}\,; \sigma)\widetilde{\pi}_{\chi_{s-1}}^{\overline{v}}(dx_{s-1}^{\overline{v}},d\theta_{s-1}^{\overline{v}})}},
\end{align*}
we have $C_{ij}\leq 1-\epsilon_x^{2}[\epsilon_{\theta}(\sigma)]^2$ if $\ooverline{v}=v'\in N(\overline{v})$ by Theorem \ref{thm:VanHandel4.1} in Appendix \ref{sec:Existing_results}, and $C_{ij}=0$ otherwise. Hence, by the definition of $\Corr(\mu_{s-1},\beta)$ in \eqref{eqn:corr_definition},
\begin{align}
\label{eqn:one_step_block_bound1}
\sum_{(\ooverline{k},\ooverline{v})\in I}e^{\beta |\overline{k}-\ooverline{k}|}e^{\beta d(\overline{v},\ooverline{v})}C_{(\overline{k},\overline{v})(\ooverline{k},\ooverline{v})}\leq \Corr(\widetilde{\pi}_{s-1},\beta)+(1-\epsilon_x^{2}[\epsilon_{\theta}(\sigma)]^2)e^{\beta(r+1)}\Delta.
\end{align}
To handle $b_i$, note that if $N(\overline{v})\subseteq K'$ we have $\rho^i=\widetilde{\rho}^i$, and note that
\begin{align*}
\rho^i(A)\geq \epsilon_x^{2\Delta}[\epsilon_{\theta}(\sigma)]^{2\Delta}\widetilde{\pi}_{\chi_{s-1}}^{\overline{v}},\quad\quad
\widetilde{\rho}^i(A)\geq \epsilon_x^{2\Delta}[\epsilon_{\theta}(\sigma)]^{2\Delta}\widetilde{\pi}_{\chi_{s-1}}^{\overline{v}}.
\end{align*}
Therefore, we have that $b_i=0$ if $\overline{v}\in \operatorname{int}(K')$, and by Theorem \ref{thm:VanHandel4.1} in Appendix \ref{sec:Existing_results} we have $b_i=2(1-\epsilon_x^{2\Delta}[\epsilon_{\theta}(\sigma)]^{2\Delta})$ otherwise.
\smallskip

\item When $\overline{k}=s$, we have
\begin{align*}
&\hspace*{-0.2cm}\rho^i(A)=\widetilde{\rho}^i(A)\\
&=\dfrac{\splitdfrac{
\int\mathbbm{1}_A(x_s^{v'},\theta_s^{v'})f_{X_s^{v'}|X_{s-1}}(x_s^{v'}\mid x_{s-1}\,;\theta_s^{v'})f_{\Theta_s^{v'}|\Theta_{s-1}}\,; \sigma)}{\splitdfrac{\times f_{Y_s^{v'}|X_s^{v'}}(Y_s^{v'}\mid x_s^{v'}\,;\theta_s^{v'})\prod_{u\in N(v')}f_{X_{s+1}^u|X_s}(x_{s+1}^u\mid x_s\,;\theta_{s+1}^u)}{\times f_{\Theta_{s+1}^{u}|\Theta_{s}}(\theta_{s+1}^{u}\mid \theta_{s}\,; \sigma)\psi^{v'}(dx_s^{v'})\lambda^{v'}(d\theta_s^{v'})}}}
{\splitdfrac{
\int f_{X_s^{v'}|X_{s-1}}(x_s^{v'}\mid x_{s-1}\,;\theta_s^{v'})f_{\Theta_s^{v'}|\Theta_{s-1}}\,; \sigma)}{\splitdfrac{\times f_{Y_s^{v'}|X_s^{v'}}(Y_s^{v'}\mid x_s^{v'}\,;\theta_s^{v'})\prod_{u\in N(v')}f_{X_{s+1}^u|X_s}(x_{s+1}^u\mid x_s\,;\theta_{s+1}^u)}{\times f_{\Theta_{s+1}^{u}|\Theta_{s}}(\theta_{s+1}^{u}\mid \theta_{s}\,; \sigma)\psi^{v'}(dx_s^{v'})\lambda^{v'}(d\theta_s^{v'})}}}.
\end{align*}
Therefore, we have $b_i=0$, and $C_{ij}\leq 1-\epsilon_x^{2}[\epsilon_{\theta}(\sigma)]^{2}$ if $\ooverline{k}=s-1$ and $\ooverline{v}\in N(v')$ and $C_{ij}=0$ otherwise. Hence, 
                                                                                                                                                                                                                                                                                                                                                             \begin{align}
                                                                                                                                                                                                                                                                                                                                                             \label{eqn:one_step_block_bound2}
                                                                                                                                                                                                                                                                                                                                                             \sum_{(\ooverline{k},\ooverline{v})\in I}e^{\beta |\overline{k}-\ooverline{k}|}e^{\beta d(v',\ooverline{v})}C_{(\overline{k},v')(\ooverline{k},\ooverline{v})}\leq (1-\epsilon_x^{2}[\epsilon_{\theta}(\sigma)]^{2})e^{\beta(r+1)}\Delta.
                                                                                                                                                                                                                                                                                                                                                             \end{align}
                                                                                                                                                                                                                                                                                                                                                             \end{itemize}
                                                                                                                                                                                                                                                                                                                                                             
                                                                                                                                                                                                                                                                                                                                                             Summing up \eqref{eqn:one_step_block_bound1} for $\overline{k}=s-1$ and \eqref{eqn:one_step_block_bound2} for $\overline{k}=s$, we have
                                                                                                                                                                                                                                                                                                                                                             \begin{align*}
                                                                                                                                                                                                                                                                                                                                                             &\hspace*{-1cm}\max_{(\overline{k},\overline{v})\in I}\sum_{(\ooverline{k},\ooverline{v})\in I}e^{\beta  |\overline{k}-\ooverline{k}|}e^{\beta d(\overline{v},\ooverline{v})}C_{(\overline{k},\overline{v})(\ooverline{k},\ooverline{v})}\\
                                                                                                                                                                                                                                                                                                                                                             &\leq \Corr(\widetilde{\pi}_{s-1},\beta)+(1-\epsilon_x^{2}[\epsilon_{\theta}(\sigma)]^2)e^{\beta(r+1)}\Delta.
                                                                                                                                                                                                                                                                                                                                                   \end{align*}
                                                                                                                                                                                                                                                                                                                                                             Furthermore, by Lemma \ref{thm:onethird} in Appendix \ref{sec:bounding_bias_lemmas} we have that
                                                                                                                                                                                                                                                                                                                                                             \begin{align*}
                                                                                                                                                                                                                                                                                                                                                             \Corr(\widetilde{\pi}_{s-1},\beta)
                                                                                                                                                                                                                                                                                                                                                             < \frac{1}{3},
                                                                                                                                                                                                                                                                                                                                                             \end{align*}
                                                                                                                                                                                                                                                                                                                                                             and by Lemma \ref{beta_property} in Appendix \ref{sec:bounding_bias_lemmas} we have that
                                                                                                                                                                                                                                                                                                                                                             $$(1-\epsilon_x^2[\epsilon_{\theta}(\sigma)]^2)e^{\beta(r+1)}\Delta \leq \frac{1}{16}.$$
                                                                                                                                                                                                                                                                                                                                                               Hence,
                                                                                                                                                                                                                                                                                                                                                             \begin{align*}
                                                                                                                                                                                                                                                                                                                                                             \max_{(\overline{k},\overline{v})\in I}\sum_{(\ooverline{k},\ooverline{v})\in I}e^{\beta  |\overline{k}-\ooverline{k}|}e^{\beta d(\overline{v},\ooverline{v})}C_{(\overline{k},\overline{v})(\ooverline{k},\ooverline{v})}=\frac{19}{48}.
                                                                                                                                                                                                                                                                                                                                                             \end{align*}
                                                                                                                                                                                                                                                                                                                                                             By Theorem \ref{thm:Dobrushin} (Dobrushin comparison theorem)  in Appendix \ref{sec:Existing_results} and Theorem \ref{thm:VanHandel4.3}  in Appendix \ref{sec:Existing_results}, we have
                                                                                                                                                                                                                                                                                                                                                             \begin{align*}
                                                                                                                                                                                                                                                                                                                                                             &\hspace*{-2cm}\| (\mathsf{F}_s\widetilde{\pi}_{s-1})_{\chi_{s},\chi_{s+1}}^{v'}-(\widetilde{\mathsf{F}}_s\widetilde{\pi}_{s-1})_{\chi_{s},\chi_{s+1}}^{v'}\|=\|\rho-\widetilde{\rho}\|_{(s,v')}\\
\leq & 2 (1-\epsilon_x^{2\Delta}[\epsilon_{\theta}(\sigma)]^{2\Delta}) \sum_{\overline{v}\in V\backslash \operatorname{int}(K')}D_{(s,v')(s-1,\overline{v})}\\
                                                                                                                                                                                                                                                                                                                                                             \leq & \frac{96}{29}e^{-\beta}(1-\epsilon_x^{2\Delta}[\epsilon_{\theta}(\sigma)]^{2\Delta})e^{-\beta d(v',\partial K')},
                                                                                                                                                                                                                                                                                                                                                             \end{align*}
                                                                                                                                                                                                                                                                                                                                                             which is uniform for all $x_s,x_{s+1}\in \mathbb{X}$
                                                                                                                                                                                                                                                                                                                                                               and all $\theta_s,\theta_{s+1}\in \mathbb{\Theta}$.
                                                                                                                                                                                                                                                                                                                                                             \end{proof}
                                                                                                                                                                                                                                                                                                                                                             
                                                                                                                                                                                                                                                                                                                                                             \section{Proofs for bounding variance}
                                                                                                                                                                                                                                                                                                                                                             \label{sec:bounding_variance_lemmas}
                                                                                                                                                                                                                                                                                                                                                             \begin{proposition}
                                                                                                                                                                                                                                                                                                                                                             \label{thm:variance_propagation}
                                                                                                                                                                                                                                                                                                                                                             Under Assumption \ref{assumption} in Section \ref{sec:beat_COD}, for integer $s\geq 1$, one has
                                                                                                                                                                                                                                                                                                                                                             \begin{align*}
                                                                                                                                                                                                                                                                                                                                                             &\hspace*{-1.5cm}\max_{K\in\mathcal{K}}\bigg[\mathbf{E}\big\|\widetilde{\mathsf{F}}_{s+1}\widetilde{\mathsf{F}}_s\widehat{\pi}_{s-1}-\widetilde{\mathsf{F}}_{s+1}\widehat{\mathsf{F}}_s\widehat{\pi}_{s-1} \big\|_K^2\bigg]^{1/2}\\
                                                                                                                                                                                                                                                                                                                                                             &\hspace*{1.7cm}\leq  \frac{16}{\sqrt{J}}[\epsilon_{\theta}(\sigma)]^{-2|\mathcal{K}|_{\infty}}\epsilon_x^{-2|\mathcal{K}|_{\infty}}\epsilon_y^{-2|\mathcal{K}|_{\infty}(\Delta_{\mathcal{K}}+1)}\Delta_{\mathcal{K}},
                                                                                                                                                                                                                                                                                                                                                             \end{align*}
                                                                                                                                                                                                                                                                                                                                                             where $|\mathcal{K}|_{\infty}$ is the maximal size of a block in $\mathcal{K}$ defined in 
                                                                                                                                                                                                                                                                                                                                                             \eqref{eqn:maxsize_block}. 
                                                                                                                                                                                                                                                                                                                                                               \end{proposition}
                                                                                                                                                                                                                                                                                                                                                             
                                                                                                                                                                                                                                                                                                                                                             \begin{proof} 
                                                                                                                                                                                                                                                                                                                                                             For any $K\in\mathcal{K}$, by Theorem \ref{thm:VanHandel4.2} in Appendix \ref{sec:Existing_results}, we have
                                                                                                                                                                                                                                                                                                                                                             \begin{align}
                                                                                                                                                                                                                                                                                                                                                             \label{eqn:variance_propagation_eq1}
                                                                                                                                                                                                                                                                                                                                                             &\hspace*{-2.5cm}\big\|\widetilde{\mathsf{F}}_{s+1}\widetilde{\mathsf{F}}_s\widehat{\pi}_{s-1}-\widetilde{\mathsf{F}}_{s+1}\widehat{\mathsf{F}}_s\widehat{\pi}_{s-1} \big\|_K\nonumber\\
                                                                                                                                                                                                                                                                                                                                                             =&\big\|\mathsf{C}_{s+1}\mathsf{B}\mathsf{P}_{s+1}\widetilde{\mathsf{F}}_s\widehat{\pi}_{s-1}-\mathsf{C}_{s+1}\mathsf{B}\mathsf{P}_{s+1}\widehat{\mathsf{F}}_s\widehat{\pi}_{s-1} \big\|_K\nonumber\\
                                                                                                                                                                                                                                                                                                                                                             =&\big\|\mathsf{C}_{s+1}^K\mathsf{B}^K\mathsf{P}_{s+1}\widetilde{\mathsf{F}}_s \widehat{\pi}_{s-1}-\mathsf{C}_{s+1}^K\mathsf{B}^K\mathsf{P}_{s+1}\widehat{\mathsf{F}}_s\widehat{\pi}_{s-1} \big\|\nonumber\\
                                                                                                                                                                                                                                                                                                                                                             \leq & 2 \epsilon_y^{-2 |\mathcal{K}|_{\infty}}\|\mathsf{B}^K\mathsf{P}_{s+1}\widetilde{\mathsf{F}}_s \widehat{\pi}_{s-1}-\mathsf{B}^K\mathsf{P}_{s+1}\widehat{\mathsf{F}}_s\widehat{\pi}_{s-1} \big\|.
                                                                                                                                                                                                                                                                                                                                                             \end{align}
                                                                                                                                                                                                                                                                                                                                                             For $\psi^K(dx_{\cdot}^K)$ and $\lambda^K(d\theta_{\cdot}^K)$ defined in \eqref{eqn:block_measure}, we have
                                                                                                                                                                                                                                                                                                                                                             \begin{align*}
                                                                                                                                                                                                                                                                                                                                                             &\hspace*{-0.3cm}\frac{\bigg(\mathsf{B}^K\mathsf{P}_{s+1}\widetilde{\mathsf{F}}_s\widehat{\pi}_{s-1}\bigg)(dx_{s+1}^K, d\theta_{s+1}^K)}{\psi^K(dx_{s+1}^K)\lambda^K(d\theta_{s+1}^K)}\\
                                                                                                                                                                                                                                                                                                                                                             =&\dfrac{\splitdfrac{
                                                                                                                                                                                                                                                                                                                                                               \int \prod_{v \in K}f_{X_{s+1}^v|X_s}(x_{s+1}^v\mid x_s\,;\theta_{s+1}^v) f_{\Theta_{s+1}^v|\Theta_s}(\theta_{s+1}^v\mid \theta_{s}\,; \sigma)}{\times\prod_{K'\in N(K)}\prod_{v'\in K'}f_{Y_s^{v'}|X_s^{v'}}(Y_s^{v'}\mid x_s^{v'}\,;\theta_s^{v'})\left[\mathsf{B}^{K'}\mathsf{P}_s \widehat{\pi}_{s-1}\right](dx_s^{K'},d\theta_s^{K'})}}
{\int \prod_{K'\in N(K)}\prod_{v'\in K'}f_{Y_s^{v'}|X_s^{v'}}(Y_s^{v'}\mid x_s^{v'}\,;\theta_s^{v'})\left[\mathsf{B}^{K'}\mathsf{P}_s \widehat{\pi}_{s-1}\right](dx_s^{K'},d\theta_s^{K'})},
                                                                                                                                                                                                                                                                                                                                                               \end{align*}
                                                                                                                                                                                                                                                                                                                                                               and 
                                                                                                                                                                                                                                                                                                                                                               \begin{align*}
                                                                                                                                                                                                                                                                                                                                                               &\hspace*{-0.3cm}\frac{\bigg(\mathsf{B}^K\mathsf{P}_{s+1}\widehat{\mathsf{F}}_s\widehat{\pi}_{s-1}\bigg)(dx_{s+1}^K, d\theta_{s+1}^K)}{\psi^K(dx_{s+1}^K)\lambda^K(d\theta_{s+1}^K)}\\
                                                                                                                                                                                                                                                                                                                                                               =&\dfrac{\splitdfrac{
                                                                                                                                                                                                                                                                                                                                                                 \int \prod_{v \in K}f_{X_{s+1}^v|X_s}(x_{s+1}^v\mid x_s\,;\theta_{s+1}^v) f_{\Theta_{s+1}^v|\Theta_s}(\theta_{s+1}^v\mid \theta_{s}\,; \sigma)}{\times\prod_{K'\in N(K)}\prod_{v'\in K'}f_{Y_s^{v'}|X_s^{v'}}(Y_s^{v'}\mid x_s^{v'}\,;\theta_s^{v'})\left[\mathsf{B}^{K'}\mathsf{S}^{J}\mathsf{P}_s \widehat{\pi}_{s-1}\right](dx_s^{K'},d\theta_s^{K'})}}
{\int \prod_{K'\in N(K)}\prod_{v'\in K'}f_{Y_s^{v'}|X_s^{v'}}(Y_s^{v'}\mid x_s^{v'}\,;\theta_s^{v'})\left[\mathsf{B}^{K'}\mathsf{S}^{J}\mathsf{P}_s \widehat{\pi}_{s-1}\right](dx_s^{K'},d\theta_s^{K'})},
                                                                                                                                                                                                                                                                                                                                                                 \end{align*}
                                                                                                                                                                                                                                                                                                                                                                 where $N(K)$ is defined in \eqref{eqn:block_neighborhood} as the collection of blocks that interact with the block $K$.
                                                                                                                                                                                                                                                                                                                                                                 By equation $(8.1)$ in \cite{georgii2011gibbs}, we have
                                                                                                                                                                                                                                                                                                                                                                 \begin{align*}
                                                                                                                                                                                                                                                                                                                                                                 &\hspace*{-0.4cm}\big\|\mathsf{B}^K\mathsf{P}_{s+1}\widetilde{\mathsf{F}}_s\widehat{\pi}_{s-1}-\mathsf{B}^K\mathsf{P}_{s+1}\widehat{\mathsf{F}}_s\widehat{\pi}_{s-1}\big\|\\
                                                                                                                                                                                                                                                                                                                                                                 =&\scaleint{10.5ex}\, \left|\frac{\bigg(\mathsf{B}^K\mathsf{P}_{s+1}\widetilde{\mathsf{F}}_s\widehat{\pi}_{s-1}\bigg)(dx_{s+1}^K, d\theta_{s+1}^K)}{\psi^K(dx_{s+1}^K)\lambda^K(d\theta_{s+1}^K)}\right.\\
                                                                                                                                                                                                                                                                                                                                                                 &\hspace*{1.5cm} \left.-\frac{\bigg(\mathsf{B}^K\mathsf{P}_{s+1}\widehat{\mathsf{F}}_s\widehat{\pi}_{s-1}\bigg)(dx_{s+1}^K, d\theta_{s+1}^K)}{\psi^K(dx_{s+1}^K)\lambda^K(d\theta_{s+1}^K)}\right|\psi^K(dx_{s+1}^K)\lambda^K(d\theta_{s+1}^K).
                                                                                                                                                                                                                                                                                                                                                                 \end{align*}
                                                                                                                                                                                                                                                                                                                                                                 Therefore, by Minkowski's integral inequality,
\begin{align*}
&\hspace*{-0.45cm}\bigg[\mathbf{E}\big\|\mathsf{B}^K\mathsf{P}_{s+1}\widetilde{\mathsf{F}}_s\widehat{\pi}_{s-1}-\mathsf{B}^K\mathsf{P}_{s+1}\widehat{\mathsf{F}}_s\widehat{\pi}_{s-1}\big\|^2\bigg]^{1/2}\\
\leq & \scaleint{11ex}\, \left[\mathbf{E}\left|\frac{\bigg(\mathsf{B}^K\mathsf{P}_{s+1}\widetilde{\mathsf{F}}_s\widehat{\pi}_{s-1}\bigg)(dx_{s+1}^K, d\theta_{s+1}^K)}{\psi^K(dx_{s+1}^K)\lambda^K(d\theta_{s+1}^K)}\right.\right.\\
                                &\quad\quad\quad\left.\left.\quad-\frac{\bigg(\mathsf{B}^K\mathsf{P}_{s+1}\widehat{\mathsf{F}}_s\widehat{\pi}_{s-1}\bigg)(dx_{s+1}^K, d\theta_{s+1}^K)}{\psi^K(dx_{s+1}^K)\lambda^K(d\theta_{s+1}^K)}\right|^2\right]^{1/2}\psi^K(dx_{s+1}^K)\lambda^K(d\theta_{s+1}^K)\\
\leq &\psi^K(\mathbb{X}^K)\lambda^K(\mathbb{\Theta}^K)\sup_{x^K\in\mathbb{X}^K\atop
  \theta^K\in\mathbb{\Theta}^K}\left[\mathbf{E}\left|\frac{\bigg(\mathsf{B}^K\mathsf{P}_{s+1}\widetilde{\mathsf{F}}_s\widehat{\pi}_{s-1}\bigg)(dx_{s+1}^K, d\theta_{s+1}^K)}{\psi^K(dx_{s+1}^K)\lambda^K(d\theta_{s+1}^K)}\right.\right.\\
                                      &\quad\quad\quad\quad\quad\quad\quad\quad\quad\quad\quad\quad\quad-\left.\left.\frac{\bigg(\mathsf{B}^K\mathsf{P}_{s+1}\widehat{\mathsf{F}}_s\widehat{\pi}_{s-1}\bigg)(dx_{s+1}^K, d\theta_{s+1}^K)}{\psi^K(dx_{s+1}^K)\lambda^K(d\theta_{s+1}^K)}\right|^2\right]^{1/2}.
\end{align*}

By Assumption \ref{assumption} in Section \ref{sec:beat_COD} and the fact that $f_{X_{s+1}^v|X_s}(x_{s+1}^v\mid x_s\,;\theta_{s+1}^v)$ is a transition density, we have that
$$\epsilon_x \psi^v(\mathbb{X}^v)\leq \int f_{X_{s+1}^v|X_s}(x_{s+1}^v\mid x_s\,;\theta_{s+1}^v) \psi^v(dx_{s+1}^v)=1,$$
and then
$$\psi^v(\mathbb{X}^v) \leq \epsilon_x^{-1}, \quad \quad \psi^K(\mathbb{X}^K) \leq \epsilon_x^{-|\mathcal{K}|_{\infty}}.$$
Similarly, by Assumption \ref{assumption} in Section \ref{sec:beat_COD} and the fact that $f_{\Theta_{s+1}^v|\Theta_s}(\theta_{s+1}^v\mid \theta_{s}\,; \sigma)$ is a transition density, we have that
$$\epsilon_{\theta}(\sigma) \lambda^v(\mathbb{\Theta}^v)\leq \int f_{\Theta_{s+1}^v|\Theta_s}(\theta_{s+1}^v\mid \theta_{s}\,; \sigma) \lambda^v(d\theta_{s+1}^v)=1,$$
and then
$$\lambda^v(\mathbb{\Theta}^v) \leq [\epsilon_{\theta}(\sigma)]^{-1}, \quad \quad \lambda^K(\mathbb{\Theta}^K) \leq [\epsilon_{\theta}(\sigma)]^{-|\mathcal{K}|_{\infty}}.$$
Furthermore, by Assumption \ref{assumption} in Section \ref{sec:beat_COD}, we have
$$\prod_{v\in K}f_{X_{s+1}^v|X_s}(x_{s+1}^v\mid x_s\,;\theta_{s+1}^v)f_{\Theta_{s+1}^v|\Theta_s}(\theta_{s+1}^v\mid \theta_{s}\,; \sigma)\leq \epsilon_x^{-|\mathcal{K}|_{\infty}}[\epsilon_{\theta}(\sigma)]^{-|\mathcal{K}|_{\infty}}$$
 and
$$\epsilon_y^{|\mathcal{K}|_{\infty}\Delta_{\mathcal{K}}}\leq \prod_{K'\in N(K)}\prod_{v'\in K'}f_{Y_s^{v'}|X_s^{v'}}(Y_s^{v'}\mid x_s^{v'}\,;\theta_s^{v'}) \leq \epsilon_y^{-|\mathcal{K}|_{\infty}\Delta_{\mathcal{K}}}.$$
Hence, by Theorem \ref{thm:VanHandel4.2} in Appendix \ref{sec:Existing_results} and Assumption \ref{assumption} in Section \ref{sec:beat_COD},
\begin{align*}
&\bigg[\mathbf{E}\big\|\mathsf{B}^K\mathsf{P}_{s+1}\widetilde{\mathsf{F}}_s\widehat{\pi}_{s-1}-\mathsf{B}^K\mathsf{P}_{s+1}\widehat{\mathsf{F}}_s\widehat{\pi}_{s-1}\big\|^2\bigg]^{1/2}\\
&\leq  2[\epsilon_{\theta}(\sigma)]^{-2|\mathcal{K}|_{\infty}}\epsilon_x^{-2|\mathcal{K}|_{\infty}}\epsilon_y^{-2|\mathcal{K}|_{\infty}\Delta_{\mathcal{K}}}\vertiii{\bigotimes_{K'\in N(K)}\mathsf{B}^{K'}\mathsf{P}_s \widehat{\pi}_{s-1}-\bigotimes_{K'\in N(K)}\mathsf{B}^{K'}\mathsf{S}^{J}\mathsf{P}_s\widehat{\pi}_{s-1}}\\
&\leq  8[\epsilon_{\theta}(\sigma)]^{-2|\mathcal{K}|_{\infty}}\epsilon_x^{-2|\mathcal{K}|_{\infty}}\epsilon_y^{-2|\mathcal{K}|_{\infty}\Delta_{\mathcal{K}}}\frac{\Delta_{\mathcal{K}}}{\sqrt{J}},
\end{align*}
where the last inequality is by Theorem \ref{thm:VanHandel4.21} in Appendix \ref{sec:Existing_results}.
Plugging in equation \eqref{eqn:variance_propagation_eq1}, the proof is complete.
\end{proof} 

\begin{lemma}
\label{thm:error_prop_pointmass}
Under Assumption \ref{assumption} in Section \ref{sec:beat_COD}, when condition \eqref{eqn:main_thm_condition} holds, one has that for every  $K\in \mathcal{K}$, $\K \subseteq K$ and $s \in \{1,\ldots,n\}$,
 $$\|\widetilde{\mathsf{F}}_n \cdots\widetilde{\mathsf{F}}_{s}\delta_{x}\delta_{\theta}-\widetilde{\mathsf{F}}_n \cdots\widetilde{\mathsf{F}}_{s}\delta_{\overline{x}}\delta_{\overline{\theta}} \|_{\K} \leq 
\frac{32}{13}e^{-\tilde{\beta} (n-s+1)}\card(\K),$$
where  
\begin{align}
\label{eqn:tilde_beta_definition}
\tilde{\beta}=\frac{1}{2r}\log\left(\frac{1}{16\Delta^2(1-\epsilon_x^{2\Delta}[\epsilon_{\theta}(\sigma)]^{2\Delta})} \right).
\end{align} 
\end{lemma}

\begin{proof}
We first tackle the $s=1$ case that 
$$\|\widetilde{\mathsf{F}}_n \cdots\widetilde{\mathsf{F}}_{1}\delta_{x}\delta_{\theta}-\widetilde{\mathsf{F}}_n \cdots\widetilde{\mathsf{F}}_{1}\delta_{\overline{x}}\delta_{\overline{\theta}}\|_{\K}$$
and then generalize to all $s \in \{1,\ldots,n\}$.

Recalling that $N(K)$ is defined in \eqref{eqn:block_neighborhood} as the collection of blocks that interact with the block $K$ in one time step,
we define a ``block'' tree $T$ as
\begin{align*}
T=\big\{[K_u\cdots K_{n}]: 0\leq u\leq n,\,  K_l\in N(K_{l+1}), \, u\leq l\leq n \big\},
\end{align*} 
where $K_n=K$. That is, $[K_u\cdots K_{n}]$ is the block $K_u$ at time $u$ that has interactions with the block $K$ at time $n$ after $n-u$ time steps; note that this type of representation describes the block-wised interaction trace from time $u$ up to time $n$. By the effect of block operator $\mathsf{B}^K$ on $\widetilde{\mathsf{F}}_n\rho$ for any $n\in \mathbb{N}$ and any measure $\rho$ on $\mathbb{X}\times \mathbb{\Theta}$, given in 
\eqref{eqn:B_K_effect},
we can write
\begin{align*}
\mathsf{B}^K \widetilde{\mathsf{F}}_n \cdots \widetilde{\mathsf{F}}_1 \delta_{x}\delta_{\theta}
=&\mathsf{C}_n^{K_n} \mathsf{P}_n^{K_n}\bigotimes_{K_{n-1}\in N(K_n)}
\bigg[ \mathsf{C}_{n-1}^{K_{n-1}} \mathsf{P}_{n-1}^{K_{n-1}}\\
&\hspace*{3cm}\bigotimes_{K_{n-2}\in N(K_{n-1})}
\bigg[ \mathsf{C}_{n-2}^{K_{n-2}} \mathsf{P}_{n-2}^{K_{n-2}}\cdots\\
&\hspace*{5.5cm}\bigotimes_{K_1\in N(K_2)}
\bigg[ \mathsf{C}_1^{K_1} \mathsf{P}_s^{K_1}\\
&\hspace*{7.5cm}\bigotimes_{K_0\in N(K_1)}\delta_{x}^{K_0}\delta_{\theta}^{K_0}\bigg]\cdots\bigg]\bigg],
\end{align*} 
where $\mathsf{P}_n^K$ is defined in 
\eqref{eqn:block_prediction_operator_definition}
and $\mathsf{C}_n^K$  is defined in 
\eqref{eqn:block_correction_operator_definition}.

The vertex set of the tree is defined as 
$$I=\{ [K_u\cdots K_{n}]v: [K_u\cdots K_{n}]\in T,\, v\in K_u \}.$$
Clearly, the following equivalence holds:
$$[K_u\cdots K_{n}]=\{ [K_u\cdots K_{n}]v: v\in K_u \}.$$
Define the index set of leaves of the tree $T$ as
$$T_0=\{ [K_0\cdots K_{n}]: K_l \in N(K_{l+1}),\,  0\leq l<n \}.$$
Further 
define the children $c(i)$ of  $i\in I$ as
$$c([K_u\cdots K_{n}]v)=\{ [K_{u-1}\cdots K_{n}]v': K_{u-1}\in N(K_{u}),\,  v'\in N(v)\},$$
define the location $v(i)$ of $i\in I$ as
$$v([K_u\cdots K_{n}]v)=v,$$
define the depth $d(i)$ of $i\in I$ as
$$d([K_u\cdots K_{n}]v)=u,$$
and then define the set of non-leaf vertices of the tree as 
$$ I_{+}=\{ i\in I: 0<d(i)\leq n\}.$$
Define
$$\mathbb{S}=\prod_{i\in I}\mathbb{X}^i\times \mathbb{\Theta}^i,$$
 for $[t]v\in I$ define
$$\mathbb{Y}^{[t]v}=\mathbb{Y}^v,\quad\mathbb{X}^{[t]v}=\mathbb{X}^v\quad\text{and}\quad \mathbb{\Theta}^{[t]v}=\mathbb{\Theta}^v,$$
and for any $[K_u\cdots K_{n}]\in T$ and any measure $\rho$ define
$$\rho^{[K_u\cdots K_{n}]}=\rho^{K_u}.$$ 
We define two probability measures on $\mathbb{S}$ as
\begin{align*}
\rho(A)
=&\dfrac{\splitdfrac{
\int \mathbbm{1}_A(x,\theta)\prod_{i\in I_{+}}
f_{X_{d(i)}^i|X_{d(i)-1}}(x_{d(i)}^{i}\mid x_{d(i)-1}\,;\theta_{d(i)}^{i})}{\splitdfrac{
\times f_{\Theta_{d(i)}^{i}\mid \Theta_{d(i)-1}}(\theta_{d(i)}^{i}\mid \theta_{d(i)-1}\,; \sigma)
f_{Y_{d(i)}^i|X_{d(i)}^i}(Y_{d(i)}^i\mid x_{d(i)}^i\,;\theta_{d(i)}^i)
}{\times \psi^{v(i)}(dx_{d(i)}^i)  \lambda^{v(i)}(d\theta_{d(i)}^i)\prod_{t\in T_0} \delta^{t}(dx_0^{t},d\theta_0^{t})}}}
{\splitdfrac{
\int \prod_{i\in I_{+}}
f_{X_{d(i)}^i|X_{d(i)-1}}(x_{d(i)}^{i}\mid x_{d(i)-1}\,;\theta_{d(i)}^{i})}{\splitdfrac{
\times f_{\Theta_{d(i)}^{i}\mid \Theta_{d(i)-1}}(\theta_{d(i)}^{i}\mid \theta_{d(i)-1}\,; \sigma)
f_{Y_{d(i)}^i|X_{d(i)}^i}(Y_{d(i)}^i\mid x_{d(i)}^i\,;\theta_{d(i)}^i)
}{\times \psi^{v(i)}(dx_{d(i)}^i)  \lambda^{v(i)}(d\theta_{d(i)}^i)\prod_{t\in T_0} \delta^{t}(dx_0^{t},d\theta_0^{t})}}}
\end{align*}
and
\begin{align*}
\overline{\rho}(A)
=&\dfrac{\splitdfrac{
\int \mathbbm{1}_A(x,\theta)\prod_{i\in I_{+}}
f_{X_{d(i)}^i|X_{d(i)-1}}(x_{d(i)}^{i}\mid x_{d(i)-1}\,;\theta_{d(i)}^{i})}{\splitdfrac{
\times f_{\Theta_{d(i)}^{i}\mid \Theta_{d(i)-1}}(\theta_{d(i)}^{i}\mid \theta_{d(i)-1}\,; \sigma)
f_{Y_{d(i)}^i|X_{d(i)}^i}(Y_{d(i)}^i\mid x_{d(i)}^i\,;\theta_{d(i)}^i)
}{\times \psi^{v(i)}(dx_{d(i)}^i)  \lambda^{v(i)}(d\theta_{d(i)}^i)\prod_{t\in T_0} \delta^{t}(d\overline{x}_0^{t},d\overline{\theta}_0^{t})}}}
{\splitdfrac{
\int \prod_{i\in I_{+}}
f_{X_{d(i)}^i|X_{d(i)-1}}(x_{d(i)}^{i}\mid x_{d(i)-1}\,;\theta_{d(i)}^{i})}{\splitdfrac{
\times f_{\Theta_{d(i)}^{i}\mid \Theta_{d(i)-1}}(\theta_{d(i)}^{i}\mid \theta_{d(i)-1}\,; \sigma)
f_{Y_{d(i)}^i|X_{d(i)}^i}(Y_{d(i)}^i\mid x_{d(i)}^i\,;\theta_{d(i)}^i)
}{\times \psi^{v(i)}(dx_{d(i)}^i)  \lambda^{v(i)}(d\theta_{d(i)}^i)\prod_{t\in T_0} \delta^{t}(d\overline{x}_0^{t},d\overline{\theta}_0^{t})}}},
\end{align*}
where $\delta(x,\theta)$ is defined in \eqref{eqn:pi_recursion}.
Therefore,
$$ \|\widetilde{\mathsf{F}}_n \cdots\widetilde{\mathsf{F}}_{1}\delta_{x}\delta_{\theta}-\widetilde{\mathsf{F}}_n \cdots\widetilde{\mathsf{F}}_{1}\delta_{\overline{x}}\delta_{\overline{\theta}}\|_{\K}=\|\rho-\overline{\rho}\|_{[K_n]\K}.$$

In the following, we are going to use Theorem \ref{thm:Dobrushin} (Dobrushin comparison theorem)  in Appendix \ref{sec:Existing_results} to bound $\|\rho-\overline{\rho}\|_{[K_n]\K}$. We will bound $C_{ij}$ and $b_i$ with $i=[K_u\cdots K_{n}]v$ and $j=[K'_{u'}\cdots K'_{n}]v'$ where $K_{n}=K'_{n}=K$ and $0\leq u,u'\leq n$. 
Set 
$$\rho^i=\rho_{(x,\theta)}^i\quad\text{and}\quad\overline{\rho}^i=\overline{\rho}_{(x,\theta)}^i,$$
whose definitions are given in Theorem \ref{thm:Dobrushin} in Appendix \ref{sec:Existing_results}.
We display our discussions as follows:
\begin{itemize}
\item When $u=0$, we have $\rho^i=\delta_{x_0^v}\delta_{\theta_0^v}$ and $\overline{\rho}^i=\delta_{\overline{x}_0^v}\delta_{\overline{\theta}_0^v}$. Hence 
\begin{align}
\label{eqn:error_prop_pointmass0}
C_{ij}=0
\end{align}
 and $b_i\leq 2$.\\

\item When $u\in \{1,\ldots,n-1\}$, we have 
\begin{align*}
\rho^i(A)
=\dfrac{\splitdfrac{
\int \mathbbm{1}_A(x_u^i,\theta_u^i)f_{X_u^i|X_{u-1}}(x_u^i\mid x_{u-1}\,;\theta_u^{i})f_{\Theta_u^i\mid \Theta_{u-1}}(\theta_u^i\mid \theta_{u-1}\,; \sigma)
}{\splitdfrac{\times \prod_{l\in I_{+}:i\in c(l)} 
f_{X_{d(l)}^l|X_{d(l)-1}}(x_{d(l)}^l\mid x_{d(l)-1}\,;\theta_{d(l)}^{l})
}{\splitdfrac{\times f_{\Theta_{d(l)}^l\mid \Theta_{d(l)-1}}(\theta_{d(l)}^l\mid \theta_{d(l)-1}\,; \sigma)f_{Y_{u}^{v}|X_{u}^{v}}(Y_{u}^{v}\mid x_{u}^{v}\,;\theta_{u}^{v})}{\times\psi^{v}(dx_u^i) \lambda^{v}(d\theta_u^i)}}}}
{\splitdfrac{
\int  f_{X_u^i|X_{u-1}}(x_u^i\mid x_{u-1}\,;\theta_u^{i})f_{\Theta_u^i\mid \Theta_{u-1}}(\theta_u^i\mid \theta_{u-1}\,; \sigma)
}{\splitdfrac{\times \prod_{l\in I_{+}:i\in c(l)} 
f_{X_{d(l)}^l|X_{d(l)-1}}(x_{d(l)}^l\mid x_{d(l)-1}\,;\theta_{d(l)}^{l})
}{\splitdfrac{\times f_{\Theta_{d(l)}^l\mid \Theta_{d(l)-1}}(\theta_{d(l)}^l\mid \theta_{d(l)-1}\,; \sigma)f_{Y_{u}^{v}|X_{u}^{v}}(Y_{u}^{v}\mid x_{u}^{v}\,;\theta_{u}^{v})}{\times\psi^{v}(dx_u^i) \lambda^{v}(d\theta_u^i)}}}}.
\end{align*}
We can see that $\rho^i(A)
=\overline{\rho}^i(A)$ and then $b_i=0$. Next we take care of $C_{ij}$. Note that when $j\in c(i)$ we have
\begin{align*}
\rho^i(A)\geq \epsilon_x^2[\epsilon_{\theta}(\sigma)]^2\dfrac{\splitdfrac{
\int \mathbbm{1}_A(x_u^i,\theta_u^i)\prod_{l\in I_{+}:i\in c(l)} 
f_{X_{d(l)}^l|X_{d(l)-1}}(x_{d(l)}^l\mid x_{d(l)-1}\,;\theta_{d(l)}^{l})
}{\splitdfrac{\times f_{\Theta_{d(l)}^l\mid \Theta_{d(l)-1}}(\theta_{d(l)}^l\mid \theta_{d(l)-1}\,; \sigma)f_{Y_{u}^{v}|X_{u}^{v}}(Y_{u}^{v}\mid x_{u}^{v}\,;\theta_{u}^{v})}{\times\psi^{v}(dx_u^i) \lambda^{v}(d\theta_u^i)}}}
{\splitdfrac{
\int  \prod_{l\in I_{+}:i\in c(l)} 
f_{X_{d(l)}^l|X_{d(l)-1}}(x_{d(l)}^l\mid x_{d(l)-1}\,;\theta_{d(l)}^{l})
}{\splitdfrac{\times f_{\Theta_{d(l)}^l\mid \Theta_{d(l)-1}}(\theta_{d(l)}^l\mid \theta_{d(l)-1}\,; \sigma)f_{Y_{u}^{v}|X_{u}^{v}}(Y_{u}^{v}\mid x_{u}^{v}\,;\theta_{u}^{v})}{\times\psi^{v}(dx_u^i) \lambda^{v}(d\theta_u^i)}}},
\end{align*}
when $i\in c(j)$ we have 
\begin{align*}
\rho^i(A)
\geq \epsilon_x^2[\epsilon_{\theta}(\sigma)]^2\dfrac{\splitdfrac{
\int \mathbbm{1}_A(x_u^i,\theta_u^i)f_{X_u^i|X_{u-1}}(x_u^i\mid x_{u-1}\,;\theta_u^{i})f_{\Theta_u^i\mid \Theta_{u-1}}(\theta_u^i\mid \theta_{u-1}\,; \sigma)
}{\splitdfrac{\times \prod_{\{l\in I_{+}:i\in c(l)\}\backslash \{j\}}
f_{X_{d(l)}^l|X_{d(l)-1}}(x_{d(l)}^l\mid x_{d(l)-1}\,;\theta_{d(l)}^{l})
}{\splitdfrac{\times f_{\Theta_{d(l)}^l\mid \Theta_{d(l)-1}}(\theta_{d(l)}^l\mid \theta_{d(l)-1}\,; \sigma)f_{Y_{u}^{v}|X_{u}^{v}}(Y_{u}^{v}\mid x_{u}^{v}\,;\theta_{u}^{v})}{\times\psi^{v}(dx_u^i) \lambda^{v}(d\theta_u^i)}}}}
{\splitdfrac{
\int  f_{X_u^i|X_{u-1}}(x_u^i\mid x_{u-1}\,;\theta_u^{i})f_{\Theta_u^i\mid \Theta_{u-1}}(\theta_u^i\mid \theta_{u-1}\,; \sigma)
}{\splitdfrac{\times \prod_{\{l\in I_{+}:i\in c(l)\}\backslash \{j\}}
f_{X_{d(l)}^l|X_{d(l)-1}}(x_{d(l)}^l\mid x_{d(l)-1}\,;\theta_{d(l)}^{l})
}{\splitdfrac{\times f_{\Theta_{d(l)}^l\mid \Theta_{d(l)-1}}(\theta_{d(l)}^l\mid \theta_{d(l)-1}\,; \sigma)f_{Y_{u}^{v}|X_{u}^{v}}(Y_{u}^{v}\mid x_{u}^{v}\,;\theta_{u}^{v})}{\times\psi^{v}(dx_u^i) \lambda^{v}(d\theta_u^i)}}}},
\end{align*}
and when $j\in \cup_{l\in I_{+}:i\in c(l)}c(l)$ we have
\begin{align*}
\rho^i(A)
\geq \epsilon_x^{2\Delta}[\epsilon_{\theta}(\sigma)]^{2\Delta}\dfrac{\splitdfrac{
\int \mathbbm{1}_A(x_u^i,\theta_u^i)f_{X_u^i|X_{u-1}}(x_u^i\mid x_{u-1}\,;\theta_u^{i})f_{\Theta_u^i\mid \Theta_{u-1}}(\theta_u^i\mid \theta_{u-1}\,; \sigma)
}{\times  f_{Y_{u}^{v}|X_{u}^{v}}(Y_{u}^{v}\mid x_{u}^{v}\,;\theta_{u}^{v})\psi^{v}(dx_u^i) \lambda^{v}(d\theta_u^i)}}
{\splitdfrac{
\int  f_{X_u^i|X_{u-1}}(x_u^i\mid x_{u-1}\,;\theta_u^{i})f_{\Theta_u^i\mid \Theta_{u-1}}(\theta_u^i\mid \theta_{u-1}\,; \sigma)
}{\times f_{Y_{u}^{v}|X_{u}^{v}}(Y_{u}^{v}\mid x_{u}^{v}\,;\theta_{u}^{v})\psi^{v}(dx_u^i) \lambda^{v}(d\theta_u^i)}}.
\end{align*}
By Theorem \ref{thm:VanHandel4.1} in Appendix \ref{sec:Existing_results},
we have $C_{ij}\leq 1-\epsilon_x^{2}[\epsilon_{\theta}(\sigma)]^2$ if $j\in c(i)$ or $i\in c(j)$,
$C_{ij}\leq 1-\epsilon_x^{2\Delta}[\epsilon_{\theta}(\sigma)]^{2\Delta}$ if $j\in \cup_{l\in I_{+}:i\in c(l)}c(l)$, and $C_{ij}=0$ otherwise. Hence,
\begin{align}
\label{eqn:error_prop_pointmass1}
\sum_{j\in I}e^{\tilde{\beta}|d(i)-d(j)|}C_{ij}\leq & 2(1-\epsilon_x^{2}[\epsilon_{\theta}(\sigma)]^2)e^{\tilde{\beta}}\Delta
+(1-\epsilon_x^{2\Delta}[\epsilon_{\theta}(\sigma)]^{2\Delta})\Delta^2.
\end{align}

\item When $u=n$, we have
\begin{align*}
\rho^i(A)=\dfrac{\splitdfrac{
\int \mathbbm{1}_A(x_n^i,\theta_n^i)f_{X_n^i|X_{n-1}}(x_n^i\mid x_{n-1}\,;\theta_n^{i})f_{\Theta_n^{i}|\Theta_{n-1}}(\theta_n^{i}\mid \theta_{n-1}\,; \sigma)
}{\times f_{Y_{n}^{v}|X_{n}^{v}}(Y_{n}^{v}\mid x_{n}^{v}\,;\theta_{n}^{v})\psi^{v}(dx_n^i) \lambda^{v}(d\theta_n^i)}}
{\splitdfrac{
\int f_{X_n^i|X_{n-1}}(x_n^i\mid x_{n-1}\,;\theta_n^{i})f_{\Theta_n^{i}|\Theta_{n-1}}(\theta_n^{i}\mid \theta_{n-1}\,; \sigma)
}{ \times f_{Y_{n}^{v}|X_{n}^{v}}(Y_{n}^{v}\mid x_{n}^{v}\,;\theta_{n}^{v})\psi^{v}(dx_n^i) \lambda^{v}(d\theta_n^i)}}.
\end{align*}
We can see that $\rho^i(A)=\overline{\rho}^i(A)$ and then $b_i=0$. When $j\in c(i)$ we have
\begin{align*}
\rho^i(A)\geq \epsilon_x^2[\epsilon_{\theta}(\sigma)]^2\dfrac{
\int \mathbbm{1}_A(x_n^i,\theta_n^i)f_{Y_{n}^{v}|X_{n}^{v}}(Y_{n}^{v}\mid x_{n}^{v}\,;\theta_{n}^{v})\psi^{v}(dx_n^i) \lambda^{v}(d\theta_n^i)}
{\int  f_{Y_{n}^{v}|X_{n}^{v}}(Y_{n}^{v}\mid x_{n}^{v}\,;\theta_{n}^{v})\psi^{v}(dx_n^i) \lambda^{v}(d\theta_n^i)}.
\end{align*}
By Theorem \ref{thm:VanHandel4.1} in Appendix \ref{sec:Existing_results}, we have $C_{ij}\leq 1-\epsilon_x^{2}[\epsilon_{\theta}(\sigma)]^2$ if $j\in c(i)$, and $C_{ij}=0$ otherwise. Hence,
\begin{align}
\label{eqn:error_prop_pointmass2}
\sum_{j\in I}e^{\tilde{\beta}|d(i)-d(j)|}C_{ij}
\leq (1-\epsilon_x^{2}[\epsilon_{\theta}(\sigma)]^2)e^{\tilde{\beta}}\Delta.
\end{align}
\end{itemize}

Summing up \eqref{eqn:error_prop_pointmass0} for $u=0$, \eqref{eqn:error_prop_pointmass1} for $u\in \{1,\ldots,n-1\}$, and  \eqref{eqn:error_prop_pointmass2} for $u=n$, we have
\begin{align*}
\max_{i\in I}\sum_{j\in I}e^{\tilde{\beta}|d(i)-d(j)|}C_{ij}
\leq 2(1-\epsilon_x^{2}[\epsilon_{\theta}(\sigma)]^2)e^{\tilde{\beta}}\Delta
+(1-\epsilon_x^{2\Delta}[\epsilon_{\theta}(\sigma)]^{2\Delta})\Delta^2.
\end{align*}
Under Assumption \ref{assumption} in Section \ref{sec:beat_COD}, when condition \eqref{eqn:main_thm_condition} holds,  since $\Delta_{\mathcal{K}}\geq 1$ and $\Delta \geq 1$, we have
$$\frac{1}{16 \Delta^2}
> \frac{1}{16\Delta_{\mathcal{K}}\Delta^2}
>1-\epsilon_x^{2\Delta}[\epsilon_{\theta}(\sigma)]^{2\Delta}>0,$$
therefore, for $\tilde{\beta}$ defined in \eqref{eqn:tilde_beta_definition}, we have that
$$\tilde{\beta}=\frac{1}{2r}\log\left(\frac{1}{16\Delta^2(1-\epsilon_x^{2\Delta}[\epsilon_{\theta}(\sigma)]^{2\Delta})} \right)>0,$$
and
\begin{align*}
\max_{i\in I}\sum_{j\in I}e^{\tilde{\beta}|d(i)-d(j)|}C_{ij}
\leq 3(1-\epsilon_x^{2\Delta}[\epsilon_{\theta}(\sigma)]^{2\Delta})e^{\tilde{\beta}}\Delta^2\leq \frac{3}{16}.
\end{align*}
By Theorem \ref{thm:Dobrushin} (Dobrushin comparison theorem)  in Appendix \ref{sec:Existing_results} and Theorem \ref{thm:VanHandel4.3}  in Appendix \ref{sec:Existing_results}, we obtain
\begin{align*}
\|\widetilde{\mathsf{F}}_n \cdots\widetilde{\mathsf{F}}_{1}\delta_{x}\delta_{\theta}-\widetilde{\mathsf{F}}_n \cdots\widetilde{\mathsf{F}}_{1}\delta_{\overline{x}}\delta_{\overline{\theta}}\|_{\K}&=\|\rho-\widetilde{\rho}\|_{[K_n]\K}\\
&\leq  2\times\frac{1}{1-\frac{3}{16}}e^{-\tilde{\beta} n}\card(\K)\\
&=  \frac{32}{13}e^{-\tilde{\beta} n}\card(\K),
\end{align*}
which is a special case of $\|\widetilde{\mathsf{F}}_n \cdots\widetilde{\mathsf{F}}_{s}\delta_{x}\delta_{\theta}-\widetilde{\mathsf{F}}_n \cdots\widetilde{\mathsf{F}}_{s}\delta_{\overline{x}}\delta_{\overline{\theta}} \|_{\K}$ for $s=1$.
Note that the above bound holds uniformly in the sequence of $Y$, we can generalize to all $s \in \{1,\ldots,n\}$,
$$\|\widetilde{\mathsf{F}}_n \cdots\widetilde{\mathsf{F}}_{s}\delta_{x}\delta_{\theta}-\widetilde{\mathsf{F}}_n \cdots\widetilde{\mathsf{F}}_{s}\delta_{\overline{x}}\delta_{\overline{\theta}} \|_{\K} \leq 
\frac{32}{13}e^{-\tilde{\beta} (n-s+1)}\card(\K).$$

\end{proof}

\begin{proposition}
\label{thm:error_prop_blockbound}
Under Assumption \ref{assumption} in Section \ref{sec:beat_COD}, when condition \eqref{eqn:main_thm_condition} holds, for any two product measures
$$\mu=\bigotimes_{K\in \mathcal{K}}\mu^K \quad\text{and}\quad \nu=\bigotimes_{K\in \mathcal{K}}\nu^K,$$
one has that for every $s \in \{1,\ldots,n-2\}$, $\K \subseteq K$, and $K\in \mathcal{K}$,
\begin{align*}
&\hspace*{-0.5cm}\mathbf{E}\left[\|\widetilde{\mathsf{F}}_n \cdots\widetilde{\mathsf{F}}_{s+2}\mu-\widetilde{\mathsf{F}}_n \cdots\widetilde{\mathsf{F}}_{s+2}\nu\|_{\K}^2\right]^{1/2} \\
 &\leq 
\frac{32}{13}\epsilon_x^{-2|\mathcal{K}|_{\infty}}[\epsilon_{\theta}(\sigma)]^{-2|\mathcal{K}|_{\infty}}e^{-\beta (n-s-1)}
\card(\K)\max_{K \in \mathcal{K}}\mathbf{E}[\|\mu-\nu\|_K^2]^{1/2},
\end{align*}
where  $\beta$ is given in \eqref{eqn:beta_definition}.
\end{proposition}

\begin{proof}
Define functions
\begin{align*}
\varrho_A(x_0^{T_0},\theta_0^{T_0})
&=\int \mathbbm{1}_A(x_n^{\K},\theta_n^{\K})\prod_{i\in I_{+}}f_{X_{d(i)}^i|X_{d(i)-1}}(x_{d(i)}^{i}\mid x_{d(i)-1}\,;\theta_{d(i)}^{i})
\\
&\hspace{4cm}\times f_{\Theta_{d(i)}^{i}\mid \Theta_{d(i)-1}}(\theta_{d(i)}^{i}\mid \theta_{d(i)-1}\,; \sigma)\\
&\hspace{4cm}\times f_{Y_{d(i)}^i|X_{d(i)}^i}(Y_{d(i)}^i\mid x_{d(i)}^i\,;\theta_{d(i)}^i)\\
&\hspace{4cm}\times
\psi^{v(i)}(dx_{d(i)}^i) \lambda^{v(i)}(d\theta_{d(i)}^i),\\
\varrho(x_0^{T_0},\theta_0^{T_0})
&=\int \prod_{i\in I_{+}}f_{X_{d(i)}^i|X_{d(i)-1}}(x_{d(i)}^{i}\mid x_{d(i)-1}\,;\theta_{d(i)}^{i})
\\
&\hspace{1.3cm}\times f_{\Theta_{d(i)}^{i}\mid \Theta_{d(i)-1}}(\theta_{d(i)}^{i}\mid \theta_{d(i)-1}\,; \sigma)
\\
&\hspace{1.3cm}\times f_{Y_{d(i)}^i|X_{d(i)}^i}(Y_{d(i)}^i\mid x_{d(i)}^i\,;\theta_{d(i)}^i)\\
&\hspace{1.3cm}\times  \psi^{v(i)}(dx_{d(i)}^i) \lambda^{v(i)}(d\theta_{d(i)}^i),
\end{align*}
and then we can write
\begin{align*}
(\widetilde{\mathsf{F}}_n \cdots\widetilde{\mathsf{F}}_{1}\mu)(A)=
\frac{\int\varrho_A(x_0^{T_0},\theta_0^{T_0})\prod_{t\in T_0} \mu^{t}(dx_0^{t},d\theta_0^{t})}{\int \varrho(x_0^{T_0},\theta_0^{T_0})\prod_{t\in T_0} \mu^{t}(dx_0^{t},d\theta_0^{t})}.
\end{align*}
Further define the measure
\begin{align*}
\zeta (A)=\frac{\int \mathbbm{1}_{A}(x_0^{T_0},\theta_0^{T_0})\varrho(x_0^{T_0},\theta_0^{T_0})\prod_{t\in T_0} \mu^{t}(dx_0^{t},d\theta_0^{t})}{\int \varrho(x_0^{T_0},\theta_0^{T_0})\prod_{t\in T_0} \mu^{t}(dx_0^{t},d\theta_0^{t})},
\end{align*}
and then we can write
\begin{align}
\label{eqn:error_prop_blockbound_mu}
(\widetilde{\mathsf{F}}_n \cdots\widetilde{\mathsf{F}}_{1}\mu)(A)=
  \mathop{\mathlarger{\int}}\frac{\varrho_A(x_0^{T_0},\theta_0^{T_0})}{\varrho(x_0^{T_0},\theta_0^{T_0})}\zeta (d x_0^{T_0},d\theta_0^{T_0}).
\end{align}
Analogously define the measure
\begin{align*}
\varsigma (A)=\frac{\int \mathbbm{1}_{A}(x_0^{T_0},\theta_0^{T_0})\varrho(x_0^{T_0},\theta_0^{T_0})\prod_{t\in T_0} \nu^{t}(dx_0^{t},d\theta_0^{t})}{\int \varrho(x_0^{T_0},\theta_0^{T_0})\prod_{t\in T_0} \nu^{t}(dx_0^{t},d\theta_0^{t})}
\end{align*}
and then we can write
\begin{align}
\label{eqn:error_prop_blockbound_nu}
(\widetilde{\mathsf{F}}_n \cdots\widetilde{\mathsf{F}}_{1}\nu)(A)=
  \mathop{\mathlarger{\int}}\frac{\varrho_A(x_0^{T_0},\theta_0^{T_0})}{\varrho(x_0^{T_0},\theta_0^{T_0})}\varsigma (d x_0^{T_0},d\theta_0^{T_0}).
\end{align}

Recall that the local total
variation distance is given by
$$\|\rho-\rho'\|_{\K}:=\sup_{g\in \mathcal{S}^{\K}:|g|\leq 1}|\rho(g)-\rho'(g)|,
$$
  and by equation $(8.1)$ in \cite{georgii2011gibbs} that
$$\sup_A | \rho(A)-\rho'(A) |=\sup_g | \rho(g)-\rho'(g) |/\text{osc}(g)$$
  where $g$ is a bounded function with oscillation
$$\text{osc}(g)=\sup_{x,y}|g(x)-g(y)|=\sup_x g(x)-\inf_x g(x).$$
  Since $\text{osc}(g)=2$ for $|g|\leq 1$, by \eqref{eqn:error_prop_blockbound_mu} and \eqref{eqn:error_prop_blockbound_nu}, we have
\begin{align}
\label{eqn:error_prop_blockbound}
\left\|\widetilde{\mathsf{F}}_n \cdots\widetilde{\mathsf{F}}_{1}\mu -\widetilde{\mathsf{F}}_n \cdots\widetilde{\mathsf{F}}_{1}\nu \right\|_{\K}=&2\sup_A \left | \mathop{\mathlarger{\int}}\frac{\varrho_A}{\varrho}d\zeta - \mathop{\mathlarger{\int}}\frac{\varrho_A}{\varrho}d\varsigma \right|.
\end{align}
Since $\frac{\varrho_A}{\varrho}$ is the filter obtained when the initial condition is a point mass on the leaves of the computation tree, by Lemma \ref{thm:error_prop_pointmass} we have
\begin{align*}
2 \sup_{z,z'\in (\mathbb{X}\times \mathbb{\Theta})^{T_0}}\sup_{A\in (\mathbb{X}\times \mathbb{\Theta})^{\K}} \left |\frac{\varrho_A(z)}{\varrho(z)} - \frac{\varrho_A(z')}{\varrho(z')} \right|\leq \frac{32}{13}e^{-\tilde{\beta} n}\card(\K),
  \end{align*}
  i.e.,
  \begin{align*}
  \text{osc}\left(\frac{\varrho_A}{\varrho}\right)\leq \frac{16}{13}e^{-\tilde{\beta} n}\card(\K).
  \end{align*}
  Also by the local total
  variation distance definition of this paper 
  and by equation $(8.1)$ in \cite{georgii2011gibbs}, we have
  \begin{align*}
  \left|  \mathop{\mathlarger{\int}}\frac{\varrho_A}{\varrho}d\zeta - \mathop{\mathlarger{\int}}\frac{\varrho_A}{\varrho}d\varsigma \right|\leq \frac{1}{2}\text{osc}\left(\frac{\varrho_A}{\varrho}\right) \|\zeta- \varsigma\|
    \leq  \frac{8}{13}e^{-\tilde{\beta} n}\card(\K)\|\zeta- \varsigma\|.
  \end{align*}
  Plugging in equation \eqref{eqn:error_prop_blockbound}, we have
  \begin{align}
  \label{eqn:error_prop_blockbound2}
  \left\|\widetilde{\mathsf{F}}_n \cdots\widetilde{\mathsf{F}}_{1}\mu -\widetilde{\mathsf{F}}_n \cdots\widetilde{\mathsf{F}}_{1}\nu \right\|_{\K}\leq & \frac{16}{13}e^{-\tilde{\beta} n}\card(\K)\|\zeta- \varsigma\|.
  \end{align}
  
  Note that $\varrho(x_0^{T_0},\theta_0^{T_0})$ depends on $(x_0^{t},\theta_0^{t})$  for $t\in T_0$ through the terms $$f_{X_{d(i)}^i|X_{d(i)-1}}(x_{d(i)}^{i}\mid x_{d(i)-1}\,;\theta_{d(i)}^{i})f_{\Theta_{d(i)}^{i}\mid \Theta_{d(i)-1}}(\theta_{d(i)}^{i}\mid \theta_{d(i)-1}\,; \sigma)$$ when $c(i)\cap t\neq \emptyset$. Write $t=[K_0\cdots K_n]$, and then $c(i)\cap t\neq \emptyset$ requires $i \in [K_1\cdots K_n]$, therefore
  $$\text{card}\big\{ i\in I_+: c(i)\cap t\neq \emptyset \big\}\leq \text{card}(K_1)\leq |\mathcal{K}|_{\infty}.$$
    Define 
  \begin{align*}
  \varrho^{t}(x_0^{T_0},\theta_0^{T_0})
  &=\int \prod_{i\in I_{+}:c(i)\cap t=\emptyset} f_{X_{d(i)}^i|X_{d(i)-1}}(x_{d(i)}^{i}\mid x_{d(i)-1}\,;\theta_{d(i)}^{i})
f_{\Theta_{d(i)}^{i}\mid \Theta_{d(i)-1}}(\theta_{d(i)}^{i}\mid \theta_{d(i)-1}\,; \sigma)\\
  &\hspace{2.5cm}\times f_{Y_{d(i)}^i|X_{d(i)}^i}(Y_{d(i)}^i\mid x_{d(i)}^i\,;\theta_{d(i)}^i)
  \psi^{v(i)}(dx_{d(i)}^i) \lambda^{v(i)}(d\theta_{d(i)}^i),
  \end{align*}
  and then we have
  $$\epsilon_x^{|\mathcal{K}|_{\infty}}[\epsilon_{\theta}(\sigma)]^{|\mathcal{K}|_{\infty}}\varrho^{t}(x_0^{T_0},\theta_0^{T_0})\leq \varrho(x_0^{T_0},\theta_0^{T_0})\leq
  \epsilon_x^{-|\mathcal{K}|_{\infty}}[\epsilon_{\theta}(\sigma)]^{-|\mathcal{K}|_{\infty}}\varrho^{t}(x_0^{T_0},\theta_0^{T_0}).$$
    By Theorem \ref{thm:VanHandel4.16} in Appendix \ref{sec:Existing_results}, we have
  $$\|\zeta- \varsigma\|\leq 2\epsilon_x^{-2|\mathcal{K}|_{\infty}}[\epsilon_{\theta}(\sigma)]^{-2|\mathcal{K}|_{\infty}} \sum_{t\in T_0}\|\mu^{t}-\nu^{t}\|.$$
    Plugging into equation \eqref{eqn:error_prop_blockbound2}, we have
  \begin{align*}
  \left\|\widetilde{\mathsf{F}}_n \cdots\widetilde{\mathsf{F}}_{1}\mu -\widetilde{\mathsf{F}}_n \cdots\widetilde{\mathsf{F}}_{1}\nu \right\|_{\K}\leq  \frac{32}{13}e^{-\tilde{\beta} n}\card(\K)\epsilon_x^{-2|\mathcal{K}|_{\infty}}[\epsilon_{\theta}(\sigma)]^{-2|\mathcal{K}|_{\infty}} \sum_{t\in T_0}\|\mu^{t}-\nu^{t}\|.
  \end{align*}
  Since the branching factor of $T_0$ is at most $\Delta_{\mathcal{K}}$ for each layer of $n$ layers, we have
  \begin{align*}
  &\hspace*{-0.5cm}\mathbf{E}\left[\|\widetilde{\mathsf{F}}_n \cdots\widetilde{\mathsf{F}}_{1}\mu-\widetilde{\mathsf{F}}_n \cdots\widetilde{\mathsf{F}}_{1}\nu\|_{\K}^2\right]^{1/2} \\
  &\leq 
  \frac{32}{13}\epsilon_x^{-2|\mathcal{K}|_{\infty}}[\epsilon_{\theta}(\sigma)]^{-2|\mathcal{K}|_{\infty}}e^{-\tilde{\beta} n}\Delta_{\mathcal{K}}^{n}
  \card(\K)\max_{K \in \mathcal{K}}\mathbf{E}[\|\mu^K-\nu^K\|^2]^{1/2}\\
  &= 
    \frac{32}{13}\epsilon_x^{-2|\mathcal{K}|_{\infty}}[\epsilon_{\theta}(\sigma)]^{-2|\mathcal{K}|_{\infty}}e^{-\beta n}
  \card(\K)\max_{K \in \mathcal{K}}\mathbf{E}[\|\mu^K-\nu^K\|^2]^{1/2}\\
  &=\frac{32}{13}\epsilon_x^{-2|\mathcal{K}|_{\infty}}[\epsilon_{\theta}(\sigma)]^{-2|\mathcal{K}|_{\infty}}e^{-\beta n}
  \card(\K)\max_{K \in \mathcal{K}}\mathbf{E}[\|\mu-\nu\|_K^2]^{1/2},
  \end{align*}
  where  $\beta$ is given in \eqref{eqn:beta_definition}.
  Note that the above bound holds uniformly in the sequence of $Y$, we can generalize to
  \begin{align*}
  &\hspace*{-0.5cm}\mathbf{E}\left[\|\widetilde{\mathsf{F}}_n \cdots\widetilde{\mathsf{F}}_{s+2}\mu-\widetilde{\mathsf{F}}_n \cdots\widetilde{\mathsf{F}}_{s+2}\nu\|_{\K}^2\right]^{1/2} \\
  &\leq 
  \frac{32}{13}\epsilon_x^{-2|\mathcal{K}|_{\infty}}[\epsilon_{\theta}(\sigma)]^{-2|\mathcal{K}|_{\infty}}e^{-\beta (n-s-1)}
  \card(\K)\max_{K \in \mathcal{K}}\mathbf{E}[\|\mu-\nu\|_K^2]^{1/2}.
  \end{align*}
  \end{proof}

\section{Proof for Theorem \ref{thm:main_theorem}}
\label{sec:proof_main_theorem}

\begin{proof} \textbf{of Theorem \ref{thm:main_theorem}} 
With $\vertiii{\;\cdot\;}_{\K}$ defined in \eqref{eqn:vertiii_K_definition}, by the triangle inequality, we have
\begin{align}
\label{eqn:triangle_inequality_CCOD}
\vertiii{\widehat{\pi}_n-\pi_n}_{\K}
\leq & \vertiii{\widetilde{\pi}_n-\pi_n}_{\K}+\vertiii{\widetilde{\pi}_n-\widehat{\pi}_n}_{\K}.
\end{align} 
In the following, we are going to bound $\vertiii{\widetilde{\pi}_n-\pi_n}_{\K}$ in Step $1$, bound $\vertiii{\widetilde{\pi}_n-\widehat{\pi}_n}_{\K}$ in Step $2$, and sum them up in Step $3$.
\bigskip

\noindent\textbf{Step $1$. Bounding} $\bm{\vertiii{\widetilde{\pi}_n-\pi_n}_{\K}}$.\\
Let us firstly use the local total
variation distance $\|\,\cdot\,\|_{\K}$ defined in \eqref{eqn:ltvd} to bound 
$\|\widetilde{\pi}_n-\pi_n\|_{\K}$.
By \eqref{eqn:pi_recursion}, we have that
\begin{align*}
\pi_n= \mathsf{F}_n\mathsf{F}_{n-1}\cdots\mathsf{F}_{k+1}\mathsf{F}_k\mathsf{F}_{k-1}\cdots\mathsf{F}_1\pi_0.
\end{align*}
By \eqref{eqn:widetilde_pi_recursion}, we have that
\begin{align*}
\widetilde{\pi}_n=\widetilde{\mathsf{F}}_n\widetilde{\mathsf{F}}_{n-1} \cdots\widetilde{\mathsf{F}}_{k+1}\widetilde{\mathsf{F}}_k\widetilde{\mathsf{F}}_{k-1}\cdots\widetilde{\mathsf{F}}_1\widetilde{\pi}_0.
\end{align*}
Given that $\pi_0=\widetilde{\pi}_0$, we can bound $\|\widetilde{\pi}_n-\pi_n\|_{\K}$ by means of error decomposition
\begin{align}
\label{eqn:step1_eq1}
\|\widetilde{\pi}_n-\pi_n\|_{\K}\leq \sum_{s=1}^{n} \|\mathsf{F}_n \cdots\mathsf{F}_{s+1}\widetilde{\mathsf{F}}_s\widetilde{\pi}_{s-1} - \mathsf{F}_n \cdots\mathsf{F}_{s+1}\mathsf{F}_s\widetilde{\pi}_{s-1}\|_{\K}.
\end{align}
We note that \eqref{eqn:step1_eq1} is in the sense of the error decomposition in Chapter $7$ of \citep{del2004feynman} with the corresponding diagram on page $215$ therein as a great illustration. 
We display our discussions as follows:
\begin{itemize}
	\item When $s=n$ in \eqref{eqn:step1_eq1},
 by Proposition \ref{lemma:true_block_operatorerror} in Appendix \ref{sec:bounding_bias_lemmas}, we have
\begin{align}
\label{eqn:mainthm_s=n_bias}
 \| \widetilde{\mathsf{F}}_n\widetilde{\pi}_{n-1}-\mathsf{F}_n\widetilde{\pi}_{n-1} \|_{\K}\leq 4e^{-\beta}(1-\epsilon_x^{2\Delta}[\epsilon_{\theta}(\sigma)]^{2\Delta})e^{-\beta d(\K,\partial{K})}\card(\K).
\end{align}

\item 
When $s\in \{1,\ldots,n-1\}$ in \eqref{eqn:step1_eq1}, by Proposition \ref{thm:error_prop_block} in Appendix \ref{sec:bounding_bias_lemmas} we have that
\begin{align}
\label{eqn:error_prop_block_inmainthm}
 &\hspace*{-0.4cm}\|\mathsf{F}_n \cdots\mathsf{F}_{s+1}\widetilde{\mathsf{F}}_s\widetilde{\pi}_{s-1}-
  \mathsf{F}_n \cdots\mathsf{F}_{s+1}\mathsf{F}_s\widetilde{\pi}_{s-1} \|_{\K}\\
 \leq &\frac{48}{23}e^{-\beta(n-s)}\sum_{v\in {\K}}\max_{v'\in V}e^{-\beta d(v,v')} \sup_{x_s,x_{s+1}\in\mathbb{X}\atop
\theta_s,\theta_{s+1}\in\mathbb{\Theta} }\| (\mathsf{F}_s\widetilde{\pi}_{s-1})_{\chi_{s},\chi_{s+1}}^{v'}-(\widetilde{\mathsf{F}}_s\widetilde{\pi}_{s-1})_{\chi_{s},\chi_{s+1}}^{v'}\|.\nonumber
\end{align}
By Proposition \ref{thm:one_step_block_bound} in Appendix \ref{sec:bounding_bias_lemmas}, 
we have that for every $s\geq 1$, $K'\in \mathcal{K}$, and $v'\in K'$,
\begin{align}
\label{eqn:one_step_block_bound_inmainthm}
\sup_{x_s,x_{s+1}\in \mathbb{X}\atop
\theta_s,\theta_{s+1}\in \mathbb{\Theta}}\| (\mathsf{F}_s\widetilde{\pi}_{s-1})_{\chi_{s},\chi_{s+1}}^{v'}-(\widetilde{\mathsf{F}}_s\widetilde{\pi}_{s-1})_{\chi_{s},\chi_{s+1}}^{v'}\|\leq \frac{96}{29}e^{-\beta}(1-\epsilon_x^{2\Delta}[\epsilon_{\theta}(\sigma)]^{2\Delta})e^{-\beta d(v',\partial K')}.
\end{align}
With the condition 
$$\epsilon_x\epsilon_{\theta}(\sigma)>\left(1-\frac{1}{16\Delta_{\mathcal{K}}\Delta^2} \right)^{\frac{1}{2\Delta}},$$
by Assumption \ref{assumption} in Section \ref{sec:beat_COD} and the fact that $\Delta,\Delta_{\mathcal{K}}\geq 1$, we have 
$$0<16\Delta_{\mathcal{K}}\Delta^2(1-\epsilon_x^{2\Delta}[\epsilon_{\theta}(\sigma)]^{2\Delta})<1.$$
Then the definition of $\beta$ given in \eqref{eqn:beta_definition} and the fact that $r\geq 1$ yield
\begin{equation}
\label{eqn:positive_beta}
\beta=\log\left(\frac{1}{16\Delta_{\mathcal{K}}\Delta^2(1-\epsilon_x^{2\Delta}[\epsilon_{\theta}(\sigma)]^{2\Delta})} \right)^{\frac{1}{2r}}>0.
\end{equation}
By \eqref{eqn:positive_beta} and the fact that 
$$d(v,v')+d(v',\partial K')\geq d(v,\partial K'),$$
we have
$$e^{-\beta d(v,v')}e^{-\beta d(v',\partial K')}\leq e^{-\beta d(v,\partial K')}.$$
Hence, plugging \eqref{eqn:one_step_block_bound_inmainthm} into \eqref{eqn:error_prop_block_inmainthm}, we have
\begin{align}
\label{eqn:error_prop_block_inmainthm2}
 &\hspace*{-0.4cm}\|\mathsf{F}_n \cdots\mathsf{F}_{s+1}\widetilde{\mathsf{F}}_s\widetilde{\pi}_{s-1}-
  \mathsf{F}_n \cdots\mathsf{F}_{s+1}\mathsf{F}_s\widetilde{\pi}_{s-1} \|_{\K}\nonumber\\
 \leq &\frac{48}{23}e^{-\beta(n-s)}\sum_{v\in {\K}}\max_{v'\in K', K'\in \mathcal{K}}e^{-\beta d(v,v')} \frac{96}{29}e^{-\beta}(1-\epsilon_x^{2\Delta}[\epsilon_{\theta}(\sigma)]^{2\Delta})e^{-\beta d(v',\partial K')}\nonumber\\
 \leq & \frac{48}{23} e^{-\beta(n-s)}\sum_{v\in {\K}}\max_{K'\in \mathcal{K}} \frac{96}{29}e^{-\beta}(1-\epsilon_x^{2\Delta}[\epsilon_{\theta}(\sigma)]^{2\Delta})e^{-\beta d(v,\partial K')}\nonumber\\
 \leq & 7 e^{-\beta(n-s)}\sum_{v\in {\K}}\max_{K'\in \mathcal{K}} e^{-\beta}(1-\epsilon_x^{2\Delta}[\epsilon_{\theta}(\sigma)]^{2\Delta})e^{-\beta d(v,\partial K')}.
\end{align}
By the definition of $\partial K'$ given in \eqref{eqn:partial_definition} which is the subset of vertices in $K'$ that can interact with vertices outside $K'$, we have $\partial K'\subseteq K'$ which together with the fact that 
$v\in K$ and $\partial K\subseteq K$ yield
\begin{equation}
\label{eqn:v_partialK_distance}
d(v,\partial K')\geq d(v,\partial K).
\end{equation}
Then \eqref{eqn:error_prop_block_inmainthm2} and \eqref{eqn:v_partialK_distance} give
\begin{align*}
 &\hspace*{-1.4cm}\|\mathsf{F}_n \cdots\mathsf{F}_{s+1}\widetilde{\mathsf{F}}_s\widetilde{\pi}_{s-1}-
  \mathsf{F}_n \cdots\mathsf{F}_{s+1}\mathsf{F}_s\widetilde{\pi}_{s-1} \|_{\K}\nonumber\\
 \leq &7e^{-\beta(n-s)}\sum_{v\in {\K}} e^{-\beta}(1-\epsilon_x^{2\Delta}[\epsilon_{\theta}(\sigma)]^{2\Delta})e^{-\beta d(v,\partial K)}.
\end{align*}
According to the definition of $d(\K,\K')$ in \eqref{eqn:distance_blocks_def}, we know that for any $v\in \K$ 
$$\min_{v'\in \partial{K}} d(v,v')=d(v,\partial{K}) \geq d(\K,\partial{K})=\min_{v\in \K}\min_{v'\in \partial{K}} d(v,v').$$
Hence,
\begin{align}
\label{eqn:error_prop_block_inmainthm3}
 &\hspace*{-1.4cm}\|\mathsf{F}_n \cdots\mathsf{F}_{s+1}\widetilde{\mathsf{F}}_s\widetilde{\pi}_{s-1}-
  \mathsf{F}_n \cdots\mathsf{F}_{s+1}\mathsf{F}_s\widetilde{\pi}_{s-1} \|_{\K}\nonumber\\
\leq &  7e^{-\beta(n-s+1)}(1-\epsilon_x^{2\Delta}[\epsilon_{\theta}(\sigma)]^{2\Delta})e^{-\beta d(\K,\partial{K})}\card(\K).
\end{align}
\end{itemize}

Plugging \eqref{eqn:mainthm_s=n_bias} for the $s=n$ case and \eqref{eqn:error_prop_block_inmainthm3} for the $s\in \{1,\ldots,n-1\}$ case, into the error decomposition equation \eqref{eqn:step1_eq1}, we have
\begin{align*}
\|\widetilde{\pi}_n-\pi_n\|_{\K}
&\leq  4e^{-\beta}(1-\epsilon_x^{2\Delta}[\epsilon_{\theta}(\sigma)]^{2\Delta})e^{-\beta d(\K,\partial{K})}\card(\K)\\
&\hspace*{0.4cm}+\sum_{s=1}^{n-1}7e^{-\beta(n-s+1)}(1-\epsilon_x^{2\Delta}[\epsilon_{\theta}(\sigma)]^{2\Delta})e^{-\beta d(\K,\partial{K})}\card(\K),
\end{align*}
which can be simplified, using the sum of geometric series due to the fact that $e^{-\beta}<1$ from \eqref{eqn:positive_beta}, as follows:
\begin{align*}
\|\widetilde{\pi}_n-\pi_n\|_{\K}
\leq  \frac{7e^{-\beta}}{1-e^{-\beta}}(1-\epsilon_x^{2\Delta}[\epsilon_{\theta}(\sigma)]^{2\Delta})e^{-\beta d(\K,\partial{K})}\card(\K). 
\end{align*}
Since there is no random sampling in $\|\widetilde{\pi}_n-\pi_n\|_{\K}$, we can replace it with $\vertiii{\widetilde{\pi}_n-\pi_n}_{\K}$ in the above inequality and then obtain
\begin{align}
\vertiii{\widetilde{\pi}_n-\pi_n}_{\K}
\leq  \frac{7e^{-\beta}}{1-e^{-\beta}}(1-\epsilon_x^{2\Delta}[\epsilon_{\theta}(\sigma)]^{2\Delta})e^{-\beta d(\K,\partial{K})}\card(\K). \label{eqn:CCOD_bias}
\end{align}
\smallskip

\noindent\textbf{Step $2$. Bounding} $\bm{\vertiii{\widetilde{\pi}_n-\widehat{\pi}_n}_{\K}}$.\\
By \eqref{eqn:widehat_pi_recursion}, we have that 
$$\widehat{\pi}_n=\widehat{\mathsf{F}}_n \widehat{\mathsf{F}}_{n-1}\cdots \widehat{\mathsf{F}}_{k+1}\widehat{\mathsf{F}}_k\widehat{\mathsf{F}}_{k-1}\cdots \widehat{\mathsf{F}}_1 \widehat{\pi}_0.$$
Given that $\widehat{\pi}_0=\widetilde{\pi}_0$, we can bound $\vertiii{\widetilde{\pi}_n-\widehat{\pi}_n}_{\K}$ by means of error decomposition
\begin{align}
\label{eqn:step2_eq1}
\vertiii{\widetilde{\pi}_n-\widehat{\pi}_n}_{\K}\leq \sum_{s=1}^n \vertiii{ \widetilde{\mathsf{F}}_n \cdots\widetilde{\mathsf{F}}_{s+1}\widetilde{\mathsf{F}}_s\widehat{\pi}_{s-1}-\widetilde{\mathsf{F}}_n \cdots\widetilde{\mathsf{F}}_{s+1}\widehat{\mathsf{F}}_s\widehat{\pi}_{s-1}}_{\K}.
\end{align}
We display our discussions as follows:
\begin{itemize}
	\item When $s=n$ in \eqref{eqn:step2_eq1}, 
	by the definition of $\vertiii{\;\cdot\;}_{\K}$ in \eqref{eqn:vertiii_K_definition},
we can see that for $\K\subseteq K$,
\begin{align*}
\vertiii{ \widetilde{\mathsf{F}}_n\widehat{\pi}_{n-1}-\widehat{\mathsf{F}}_n\widehat{\pi}_{n-1} }_{\K}
\leq &\vertiii{ \widetilde{\mathsf{F}}_n\widehat{\pi}_{n-1}-\widehat{\mathsf{F}}_n\widehat{\pi}_{n-1} }_{K}.
\end{align*}
Recalling that $\mathsf{P}_n$ is the prediction operator defined in \eqref{eqn:prediction_operator_definition},
$\mathsf{C}_n^K$ is the block correction operator defined in \eqref{eqn:block_correction_operator_definition},
$\mathsf{B}^K$ is the blocking operator defined in \eqref{eqn:blockingoperator},
and $\mathsf{S}^{J}$ is the sampling operator defined in \eqref{eqn:samplingoperator}, by the expressions of $\widetilde{\mathsf{F}}_n$ given in \eqref{eqn:widetilde_pi_def} and $\widehat{\mathsf{F}}_n$ given in \eqref{eqn:widehat_pi_def}, we have
\begin{align*}
\vertiii{ \widetilde{\mathsf{F}}_n\widehat{\pi}_{n-1}-\widehat{\mathsf{F}}_n\widehat{\pi}_{n-1} }_{K}= &\vertiii{\mathsf{C}_n\mathsf{B}\mathsf{P}_n \widehat{\pi}_{n-1}-\mathsf{C}_n\mathsf{B}\mathsf{S}^{J}\mathsf{P}_n \widehat{\pi}_{n-1}}_{K}\\
= &\vertiii{\mathsf{C}_n^K\mathsf{B}^K\mathsf{P}_n \widehat{\pi}_{n-1}-\mathsf{C}_n^K\mathsf{B}^K\mathsf{S}^{J}\mathsf{P}_n \widehat{\pi}_{n-1}},
\end{align*}
where the last inequality is obtained by the effect of $\mathsf{B}^K$ on $\widetilde{\mathsf{F}}_n$ provided in \eqref{eqn:B_K_effect} and the definition of $\vertiii{\;\cdot\;}_{K}$ given in \eqref{eqn:vertiii_K_definition}.
Furthermore, by Theorem \ref{thm:VanHandel4.2} in Appendix \ref{sec:Existing_results}, we have
\begin{align*}
\vertiii{ \widetilde{\mathsf{F}}_n\widehat{\pi}_{n-1}-\widehat{\mathsf{F}}_n\widehat{\pi}_{n-1} }_{K}
\leq & 2 \epsilon_y^{-2 \card(K)}\vertiii{\mathsf{P}_n \widehat{\pi}_{n-1}-\mathsf{S}^{J}\mathsf{P}_n \widehat{\pi}_{n-1}}.
\end{align*}
Noting that by Lemma $4.7$ in \cite{VanHandel2008} for $\rho$ being any probability measure
\begin{equation*}
\vertiii{\rho- \mathsf{S}^{J}\rho}\leq \frac{1}{\sqrt{J}},
\end{equation*}
and then we have
\begin{align}
\label{eqn:variance_n_mainthm}
\vertiii{ \widetilde{\mathsf{F}}_n\widehat{\pi}_{n-1}-\widehat{\mathsf{F}}_n\widehat{\pi}_{n-1} }_{\K}
\leq \frac{2 \epsilon_y^{-2 \card(K)}}{\sqrt{J}}.
\end{align}

\item When $s=n-1$ in \eqref{eqn:step2_eq1}, by the definition of $\|\cdot\|_{\K}$ given in \eqref{eqn:ltvd} and the definition of $\vertiii{\;\cdot\;}_{\K}$ given in \eqref{eqn:vertiii_K_definition}, for $\K\subseteq K$ 
we have
\begin{align}
\label{eqn:variance_n-1_mainthm}
\vertiii{ \widetilde{\mathsf{F}}_n\widetilde{\mathsf{F}}_s\widehat{\pi}_{s-1}-\widetilde{\mathsf{F}}_n\widehat{\mathsf{F}}_s\widehat{\pi}_{s-1} }_{\K}
&\leq  \max_{K\in\mathcal{K}}\bigg[\mathbf{E}\big\|\widetilde{\mathsf{F}}_{s+1}\widetilde{\mathsf{F}}_s\widehat{\pi}_{s-1}-\widetilde{\mathsf{F}}_{s+1}\widehat{\mathsf{F}}_s\widehat{\pi}_{s-1} \big\|_K^2\bigg]^{1/2}\nonumber\\
&\leq  \frac{16}{\sqrt{J}}[\epsilon_{\theta}(\sigma)]^{-2|\mathcal{K}|_{\infty}}\epsilon_x^{-2|\mathcal{K}|_{\infty}}\epsilon_y^{-2|\mathcal{K}|_{\infty}(\Delta_{\mathcal{K}}+1)}\Delta_{\mathcal{K}},
\end{align}
where the last inequality follows by Proposition \ref{thm:variance_propagation} in Appendix \ref{sec:bounding_variance_lemmas}.\\

\item When $s\in \{1,\ldots,n-2\}$ in \eqref{eqn:step2_eq1}, 
by the definitions of $\|\cdot\|_{\K}$ and $\vertiii{\;\cdot\;}_{\K}$, we have
\begin{align*}
&\hspace*{-0.1cm}\vertiii{ \widetilde{\mathsf{F}}_n \cdots\widetilde{\mathsf{F}}_{s+2}\widetilde{\mathsf{F}}_{s+1}\widetilde{\mathsf{F}}_s\widehat{\pi}_{s-1}-\widetilde{\mathsf{F}}_n \cdots\widetilde{\mathsf{F}}_{s+2}\widetilde{\mathsf{F}}_{s+1}\widehat{\mathsf{F}}_s\widehat{\pi}_{s-1}}_{\K}\\
& \leq 
\mathbf{E}\left[\|\widetilde{\mathsf{F}}_n \cdots\widetilde{\mathsf{F}}_{s+2}(\widetilde{\mathsf{F}}_{s+1}\widetilde{\mathsf{F}}_s\widehat{\pi}_{s-1})-\widetilde{\mathsf{F}}_n \cdots\widetilde{\mathsf{F}}_{s+2}(\widetilde{\mathsf{F}}_{s+1}\widehat{\mathsf{F}}_s\widehat{\pi}_{s-1})\|_{\K}^2\right]^{1/2}.
\end{align*}
Then firstly 
by Proposition \ref{thm:error_prop_blockbound} in Appendix \ref{sec:bounding_variance_lemmas} we have
\begin{align*}
&\hspace*{-1.4cm}\vertiii{ \widetilde{\mathsf{F}}_n \cdots\widetilde{\mathsf{F}}_{s+2}\widetilde{\mathsf{F}}_{s+1}\widetilde{\mathsf{F}}_s\widehat{\pi}_{s-1}-\widetilde{\mathsf{F}}_n \cdots\widetilde{\mathsf{F}}_{s+2}\widetilde{\mathsf{F}}_{s+1}\widehat{\mathsf{F}}_s\widehat{\pi}_{s-1}}_{\K}\\
& \leq 
\frac{32}{13}\epsilon_x^{-2|\mathcal{K}|_{\infty}}[\epsilon_{\theta}(\sigma)]^{-2|\mathcal{K}|_{\infty}}e^{-\beta (n-s-1)}
\card(\K)\\
&\hspace*{0.7cm}\times\max_{K \in \mathcal{K}}\mathbf{E}\left[\|\widetilde{\mathsf{F}}_{s+1}\widetilde{\mathsf{F}}_s\widehat{\pi}_{s-1}-\widetilde{\mathsf{F}}_{s+1}\widehat{\mathsf{F}}_s\widehat{\pi}_{s-1}\|_K^2\right]^{1/2},
\end{align*}
and then by Proposition \ref{thm:variance_propagation} in Appendix \ref{sec:bounding_variance_lemmas} we have
\begin{align}
\label{eqn:variance_s_other_mainthm}
&\hspace*{-0.4cm}\vertiii{ \widetilde{\mathsf{F}}_n \cdots\widetilde{\mathsf{F}}_{s+2}\widetilde{\mathsf{F}}_{s+1}\widetilde{\mathsf{F}}_s\widehat{\pi}_{s-1}-\widetilde{\mathsf{F}}_n \cdots\widetilde{\mathsf{F}}_{s+2}\widetilde{\mathsf{F}}_{s+1}\widehat{\mathsf{F}}_s\widehat{\pi}_{s-1}}_{\K}\nonumber\\
\leq&  40\frac{1}{\sqrt{J}}[\epsilon_{\theta}(\sigma)]^{-4|\mathcal{K}|_{\infty}}\epsilon_x^{-4|\mathcal{K}|_{\infty}}\epsilon_y^{-2|\mathcal{K}|_{\infty}(\Delta_{\mathcal{K}}+1)}e^{-\beta (n-s-1)}\Delta_{\mathcal{K}}\card(\K).
\end{align}
\end{itemize}

Plugging \eqref{eqn:variance_n_mainthm} for the $s=n$ case, \eqref{eqn:variance_n-1_mainthm} for the $s=n-1$ case, 
and \eqref{eqn:variance_s_other_mainthm} for the $s\in \{1,\ldots,n-2\}$ case, 
into the error decomposition equation \eqref{eqn:step2_eq1}, we have
\begin{align*}
\vertiii{\widetilde{\pi}_n-\widehat{\pi}_n}_{\K}
\leq & \frac{2 \epsilon_y^{-2 \card(K)}}{\sqrt{J}}+\frac{16}{\sqrt{J}}[\epsilon_{\theta}(\sigma)]^{-2|\mathcal{K}|_{\infty}}\epsilon_x^{-2|\mathcal{K}|_{\infty}}\epsilon_y^{-2|\mathcal{K}|_{\infty}(\Delta_{\mathcal{K}}+1)}\Delta_{\mathcal{K}}\nonumber\\
&+\sum_{s=1}^{n-2}40\frac{1}{\sqrt{J}}[\epsilon_{\theta}(\sigma)]^{-4|\mathcal{K}|_{\infty}}\epsilon_x^{-4|\mathcal{K}|_{\infty}}\epsilon_y^{-2|\mathcal{K}|_{\infty}(\Delta_{\mathcal{K}}+1)}e^{-\beta (n-s-1)}\Delta_{\mathcal{K}}\card(\K).
\end{align*}
Since $\card(K)\leq |\mathcal{K}|_{\infty}$ for any $K\in \mathcal{K}$, using the sum of geometric series due to the fact that $e^{-\beta}<1$ from \eqref{eqn:positive_beta}, the above expression can be simplified as follows:
\begin{align}
\vertiii{\widetilde{\pi}_n-\widehat{\pi}_n}_{\K} 
\leq 40\frac{1}{1-e^{-\beta}}\frac{1}{\sqrt{J}}[\epsilon_{\theta}(\sigma)]^{-4|\mathcal{K}|_{\infty}}\epsilon_x^{-4|\mathcal{K}|_{\infty}}\epsilon_y^{-2|\mathcal{K}|_{\infty}(\Delta_{\mathcal{K}}+1)}\Delta_{\mathcal{K}}\card(\K). \label{eqn:CCOD_variance}
\end{align}
\smallskip

\noindent\textbf{Step $3$. Summing up.}\\
Now plugging \eqref{eqn:CCOD_bias} and \eqref{eqn:CCOD_variance} into \eqref{eqn:triangle_inequality_CCOD}, we have
\begin{align*}
\vertiii{\widehat{\pi}_n-\pi_n}_{\K}
\leq &
\frac{7e^{-\beta}}{1-e^{-\beta}}(1-\epsilon_x^{2\Delta}[\epsilon_{\theta}(\sigma)]^{2\Delta})e^{-\beta d(\K,\partial{K})}\card(\K)\\
&+40\frac{1}{1-e^{-\beta}}\frac{1}{\sqrt{J}}[\epsilon_{\theta}(\sigma)]^{-4|\mathcal{K}|_{\infty}}\epsilon_x^{-4|\mathcal{K}|_{\infty}}\epsilon_y^{-2|\mathcal{K}|_{\infty}(\Delta_{\mathcal{K}}+1)}\Delta_{\mathcal{K}}\card(\K)\\
\leq & \frac{\card(\K)}{1-e^{-\beta}}\bigg[ 7e^{-\beta}(1-\epsilon_x^{2\Delta}[\epsilon_{\theta}(\sigma)]^{2\Delta})e^{-\beta d(\K,\partial{K})}\\
&\hspace*{2cm}+\frac{40}{\sqrt{J}}[\epsilon_{\theta}(\sigma)]^{-4|\mathcal{K}|_{\infty}}\epsilon_x^{-4|\mathcal{K}|_{\infty}}\epsilon_y^{-2|\mathcal{K}|_{\infty}(\Delta_{\mathcal{K}}+1)}\Delta_{\mathcal{K}}\bigg].
\end{align*}
Since the above bound is uniform on all $\K$, we complete the proof.
\end{proof}

\section{Further technical discussions}
\label{appendix:discussion}
Theorem \ref{thm:main_theorem} involves generalizing the important result of \cite{rebeschini2015can} (Theorem $2.1$ therein) to a time-inhomogeneous setting. Many practical applications require time-inhomogeneity; for example,  stochastic epidemic models may have a time-varying population size, or other covariate, leading to time-inhomogeneity (see, e.g. \cite{breto2009time}).  In this paper, the transition densities of $X_n$, $Y_n$, and $\Theta_n$ are all time-inhomogeneous which means they are different in each time step $n$. 
Furthermore, we consider these time-inhomogeneous transition densities in a general form where only standard conditions are required. When our transition densities are the same for each time step $n$ (i.e. being time-homogeneous), our results covers the situation of \cite{rebeschini2015can} as a special case.

Our proof broadly follows the approach of \cite{rebeschini2015can} while differing in some details that enable us to obtain stronger and more explicit bounds.
We acknowledge that the proof from  \cite{rebeschini2015can} can be adapted to the time-inhomogeneous case, but we take the opportunity to make other adjustments while adding this extension.
Specifically, we follow \cite{rebeschini2015can} by controlling the filtering error $\vertiii{\widehat{\pi}_n-\pi_n}_{\K}$ using the algorithmic bias and the algorithmic variance with the help of the triangle inequality, thus resulting in the two terms in the upper bound. An intermediate filter without sampling and resampling, $\widetilde{\pi}_n$ defined in \eqref{eqn:widetilde_pi_recursion}, was used to separate the bias $\vertiii{\widetilde{\pi}_n-\pi_n}_{\K}$ and the variance $\vertiii{\widetilde{\pi}_n-\widehat{\pi}_n}_{\K}$.
To control the bias generated by blocking, the decay of correlations (DOCs) property was established. The DOCs property arises in statistical physics, in regards to investigations of high-dimensional networks (see, e.g., \cite{liu2018tightness,liu2019information,liu2019large,liu2021phase}). In the current context, it means that the effect on the distribution on block $K$ of a perturbation made in another block $K'$ decays rapidly in the distance $d(K,K')$ defined in \eqref{eqn:distance_blocks_def}. 

Utilizing the Dobrushin comparison theorem (Theorem \ref{thm:Dobrushin}), we extended the mechanism to the parameter space by showing that the DOCs property of the underlying model is inherited by the IBPF algorithmic filter $\widehat{\pi}_n$, which is the joint conditional distribution of state $X_n$ and $\Theta_n$ given the observations $Y_1, \ldots, Y_n$. We showed that the influence of blocking on the marginal distribution
at a vertex $v\in K$ should decay exponentially in the distance from $v$ to the boundary of the block $\partial{K}$. This idea is revealed in the $e^{-\beta d(\K,\partial{K})}$ factor in the first term of the bound. To control the variance, a major issue is that $\widetilde{\pi}_n$ cannot be interpreted as a regular marginal or conditional distribution, given that it is only defined recursively in \eqref{eqn:widetilde_pi_recursion}. \cite{rebeschini2015can} solved this issue by constructing a ``computation tree", which is in analogy with a similar notion that arises in the analysis of the well-known belief propagation algorithms \citep{tatikonda2012loopy}. It is to introduce independent duplicates of the blocks in the previous time step and have each block interact with its own set of duplicates, which hence unravels the dependence graph to a tree without blockwise interactions. Then one can interpret $\widetilde{\pi}_n$ as the marginal distribution on this tree.

\cite{rebeschini2015can} focused on establishing scalability. Hence their  error bound (Theorem $2.1$ therein) targets to reveal that property while being ambiguous in some other regards, in the form as follows:
\begin{align}
	\label{eqn:rebeschini_main}
	\vertiii{\widehat{\pi}_n-\pi_n}_{\K}
	\leq & \alpha \, \card(\K)\big[ e^{-\beta_1 d(\K,\partial{K})}+e^{\beta_2|\mathcal{K}|_{\infty}}/\sqrt{J}\big],
\end{align}    
with $\K$ being a subset of the block $K$ in partition $\mathcal{K}$,
where $\alpha$, $\beta_1$, and $\beta_2$ are positive finite constants. 
Although they did not provide a precise form of the error bound in their main result, they provided a precise error bound for variance in Theorem $4.23$ (page $2864$ therein), as follows:
\begin{align}
	\label{eqn:rebeschini_variance}
	\vertiii{\widetilde{\pi}_n-\widehat{\pi}_n}_{\K}
	\leq & \card(\K)\frac{64}{\sqrt{J}}\frac{e^{\beta}}{1-e^{-\beta}}
	\epsilon_x^{-4|\mathcal{K}|_{\infty}}\epsilon_y^{-4|\mathcal{K}|_{\infty}\Delta_{\mathcal{K}}}\Delta_{\mathcal{K}},
\end{align}
where $\beta=-\log 6\Delta_{\mathcal{K}}\Delta^2(1-\epsilon^{2\Delta})>0$.
Although \eqref{eqn:rebeschini_variance} shares similarites with the second term of our bound in Theorem \ref{thm:main_theorem}, with a close look we can see that \eqref{eqn:rebeschini_variance} has an additional factor $e^{\beta}$.
$\beta$ becomes large when $\epsilon$ is close to 1, which is how this mathematical framework describes the situation where spatiotemporal mixing is fast.
We expect the resulting bound on the error of the block filter to be tighter in this case, and our bound has that property whereas the bound of \cite{rebeschini2015can} does not.
Furthermore, we can see that the exponent of $\epsilon_y^{-1}(>1)$ differs.
Whereas $\epsilon$, $\epsilon_x$ and $\epsilon_{\theta}$ are required to be close to 1, $\epsilon_y$ can be close to $0$.
Thus, an improvement from $\epsilon_y^{-4|\mathcal{K}|_{\infty}\Delta_{\mathcal{K}}}$ to  $\epsilon_y^{-2|\mathcal{K}|_{\infty}(\Delta_{\mathcal{K}}+1)}$ can substantially tighten the bound, especially when the maximal block size  $(|\mathcal{K}|_{\infty})$ is not small.


Throughout the proofs in \cite{rebeschini2015can} and ours, a positive constant $\beta$ containing those quantities is used. We provide a precise definition of $\beta$ in \eqref{eqn:beta_definition} for the first time, which is used consistently in all the proofs.
Our precise constant $\beta$ is able to rigorously reveal the influences of those crucial quantities on the error bound which have been open problems: 
when $r$ (the range of interacting neighborhoods) increases with other quantities fixed, the error bound increases;
when $\Delta$ (the maximal number of vertices that interact with one single vertex in its $r$-neighborhood) increases  with other quantities fixed, the error bound increases;
when $\Delta_{\mathcal{K}}$ (the maximal number of blocks that interact with a single block) increases with other quantities fixed, the error bound increases;
when $|\mathcal{K}|_{\infty}$ (the maximal size of one single block in the partition) increases  with other quantities fixed, the error bound increases.
%
Furthermore, we provided a precise sufficient condition for the first time, and used it throughout all the proofs. That is, the product of the assumed lower bound of $X$'s local transition density ($\epsilon_x$) and the assumed lower bound of $\Theta$'s local transition density ($\epsilon_{\theta}(\sigma)$) is larger than $\left(1-\frac{1}{16\Delta_{\mathcal{K}}\Delta^2} \right)^{\frac{1}{2\Delta}}$.

Although we followed the strategy in \cite{rebeschini2015can} in general, to adapt to the time-inhomogeneous setting some proof strategies need to be adjusted correspondingly. For example, their Proposition $4.4$ (Page $2839$) achieves local filter stability by bounding the term
$$\| \mathsf{F}_n \cdots\mathsf{F}_{s+1}\mu-\mathsf{F}_n \cdots\mathsf{F}_{s+1}\nu \|_{\K},$$
where they used the Dobrushin comparison theorem on the distributions
$$\rho=\mathbf{P}^{\mu}[X_0,\ldots,X_n\in \cdot \mid Y_{1},\ldots, Y_n]\quad\text{and}\quad\widetilde{\rho}=\mathbf{P}^{\nu}[X_0,\cdots,X_n\in \cdot \mid Y_{1},\ldots, Y_n].$$
Our local filter stability is established in Proposition \ref{thm:error_prop_block} by bounding the term
$$\| \mathsf{F}_n \cdots\mathsf{F}_{s+1}\mathsf{F}_s\widetilde{\pi}_{s-1}-\mathsf{F}_n \cdots\mathsf{F}_{s+1}\widetilde{\mathsf{F}}_s\widetilde{\pi}_{s-1} \|_{\K}$$
where we used the Dobrushin comparison theorem on the distributions
$$\rho=\mathbf{P}^{\widetilde{\mathsf{F}}_s\widetilde{\pi}_{s-1}}[X_s,X_{s+1},\ldots,X_n\in \cdot \;, \Theta_s,\Theta_{s+1},\ldots,\Theta_n\in \cdot \mid Y_{s+1},\ldots, Y_n],$$
$$\widetilde{\rho}=\mathbf{P}^{\mathsf{F}_s\widetilde{\pi}_{s-1}}[X_s,X_{s+1},\cdots,X_n\in \cdot \;, \Theta_s,\Theta_{s+1},\cdots,\Theta_n\in \cdot \mid Y_{s+1},\ldots, Y_n].$$
Note that, with time-inhomogeneous, to quantify the effect of $\mathsf{F}_{s+1}$ on $\mu$, it is appropriate to use distributions on latent states starting from $s$ instead of $0$.

\section{Proof of Theorem \ref{thm:ionidesA2_modified}}
\label{appendix:block_MLE}

\begin{proof}
	Let the inital particle swarm $\{\Theta_j^0,\,1\leq j\leq J\}$ consist of independent draws from the density $g$. For $T_{\sigma}$ in the fractional form defined in \eqref{recursion:if}, we write 
	$T_{\sigma}g(\theta)=S_{\sigma}g(\theta)/\|S_{\sigma}g\|_1$. Then $S_{\sigma}^m$, as the $m$-th iteration of $S_{\sigma}$ can be written as
	$$S_{\sigma}^m g(\theta)=\int s_{\sigma}^m(\vartheta,\theta) g(\vartheta)d\vartheta.$$
	Under conditions (B$2$) and  (B$4$), and under conditions imposed in Theorem \ref{thm:main_theorem}, there exist $m_0\geq 1$ and $0<\delta_m<\infty$ such that for any $m\geq m_0$, any measurable set $A\subset \mathbb{\Theta}$, and any $\theta\in \mathbb{\Theta}$,
	\begin{align}
		\label{eqn:S_mixing}
		\delta_m \lambda(A)\leq \int_A s_{\sigma}^m(\vartheta,\theta) d\vartheta\leq \delta_m^{-1} \lambda(A).
	\end{align}
	That is, $S_{\sigma}^{m_0}$ is mixing according to Definition $3.2$ of \citep{le2004stability}.
	
	Consider $M$ can be written as $M=q m_0+r$ for some $r\in \{0, 1,\ldots, m_0-1\}$. Define the empirical measures 
	$$\mu^{(0)}=\frac{1}{J}\sum_{j=1}^{J}\delta_{\Theta_j^r} \qquad \text{and}\qquad \mu^{(k)}=\frac{1}{J}\sum_{j=1}^{J}\delta_{\Theta_j^{km_0+r}}$$
	where $k=1,\ldots,q$.
	 For any bounded measurable function $\breve{g}:\mathbb{\Theta}\rightarrow \mathbb{R}$, we have
	\begin{align*}
	  &\frac{1}{J}\sum_{j=1}^{J}\breve{g}(\Theta_j^{M})-[T_{\sigma}^M g](\breve{g})\\
	  &=\mu^{(q)}(\breve{g})-[T_{\sigma}^M g](\breve{g})\\
	  &=\mu^{(q)}(\breve{g})-[T_{\sigma}^{m_0 q} \mu^{(0)}](\breve{g})+[T_{\sigma}^{m_0 q} \mu^{(0)}](\breve{g})-[T_{\sigma}^{m_0 q} T_{\sigma}^{r} g](\breve{g})\\
	  &=\sum_{i=1}^{q}\Big\{[T_{\sigma}^{m_0(i-1)}\mu^{(q-i+1)}](\breve{g})-[T_{\sigma}^{m_0 i}\mu^{(q-i)}](\breve{g}) \Big\}
	  +[T_{\sigma}^{m_0 q} \mu^{(0)}](\breve{g})-[T_{\sigma}^{m_0 q} T_{\sigma}^{r} g](\breve{g}).
	\end{align*}
	Denote 
	\begin{align*}
		\rho=\sup_{\breve{g}: \|\breve{g}\|_{\infty}=1}\mathbf{E}\left| \mu^{(k)}(\breve{g})-[T_{\sigma}^{m_0} \mu^{(k-1)}](\breve{g})\right|.
	\end{align*}
Then by the definition of $\vertiii{\cdot}_{\K}$ given in \eqref{eqn:vertiii_K_definition} and by Theorem \ref{thm:main_theorem}, we have that  for every $n\geq 0$, $K\in \mathcal{K}$ and $\K \subseteq K$,
	\begin{align}
		\label{eqn:rho_upperbound}
	\rho\leq C_{\alpha} \card(\K)\left[e^{-C_{\beta_1} d(\K,\partial{K})}+\frac{e^{C_{\beta_2} |\mathcal{K}|_{\infty}}}{\sqrt{J}}\right]
\end{align}
where $C_{\alpha}$, $C_{\beta_1}$, and $C_{\beta_2}$ are some positive finite constants.

Let $H(\cdot,\cdot)$ be the Hilbert metric on nonnegative measures 
$$H(\rho,\rho'):=\left\{ \begin{array}{lcl}
	\log\frac{\sup_{A:\rho'(A)>0}\rho(A)/\rho'(A)}{\inf_{A:\rho'(A)>0}\rho(A)/\rho'(A)} && \mbox{if }
	\rho \text{ and } \rho' \text{ are comparable}, \\ 0 && \mbox{if }  \rho=\rho'\equiv 0,\\
	+\infty && \mbox{otherwise.} 
\end{array}\right.$$
Here, the measures $\rho$ and $\rho'$ are said to be comparable if they are both nonzero and there exist constants $0<a\leq b$ such that $a\,\rho'(A)\leq \rho(A)\leq b\,\rho'(A)$ for all measurable subsets $A$.
Since $S_{\sigma}^{m_0}$ is mixing and \eqref{eqn:S_mixing} holds, applying Lemmas $3.4$, $3.5$, and $3.8$ and equation $(7)$ in \citep{le2004stability}, we have 
\begin{align*}
\mathbf{E}\Big |	[T_{\sigma}^{m_0 q} \mu^{(0)}](\breve{g})-[T_{\sigma}^{m_0 q} T_{\sigma}^{r} g](\breve{g})\Big |&\leq  \|\breve{g}\|_{\infty}\mathbf{E}\Big\|T_{\sigma}^{m_0 q} \mu^{(0)}-T_{\sigma}^{m_0 q} T_{\sigma}^{r} g\Big\|\\
&\leq \frac{2\|\breve{g}\|_{\infty}}{\log 3}\mathbf{E}\Big [H(S_{\sigma}^{m_0 q} \mu^{(0)},S_{\sigma}^{m_0 q} T_{\sigma}^{r} g) \Big ]\\
&\leq \frac{2\|\breve{g}\|_{\infty}}{\log 3} \left(\frac{1-\delta_{m_0}^2}{1+\delta_{m_0}^2}\right)^{q-2}\frac{1}{\delta_{m_0}^2}\mathbf{E}\Big\|T_{\sigma}^{m_0 } \mu^{(0)}-T_{\sigma}^{m_0} T_{\sigma}^{r} g\Big\|\\
&\leq \frac{4\|\breve{g}\|_{\infty}}{\log 3} \left(\frac{1-\delta_{m_0}^2}{1+\delta_{m_0}^2}\right)^{q-2}\frac{\rho}{\delta_{m_0}^4}.
\end{align*}
Similarly, we have
\begin{align*}
	\mathbf{E}\Big |	[T_{\sigma}^{m_0} \mu^{(q-1)}](\breve{g})-[T_{\sigma}^{2m_0} \mu^{(q-2)}](\breve{g})\Big |
	&\leq \frac{2\|\breve{g}\|_{\infty}\rho}{\delta_{m_0}^2},
\end{align*}
and for $i=3,\ldots,q$
\begin{align*}
	\mathbf{E}\Big |	[T_{\sigma}^{m_0 (i-1)} \mu^{(q-i+1)}](\breve{g})-[T_{\sigma}^{m_0 i} \mu^{(q-i)}](\breve{g})\Big |
	&\leq \frac{4\|\breve{g}\|_{\infty}}{\log 3} \left(\frac{1-\delta_{m_0}^2}{1+\delta_{m_0}^2}\right)^{i-3}\frac{\rho}{\delta_{m_0}^4}.
\end{align*}
Now plugging in the upper bound of $\rho$ in \eqref{eqn:rho_upperbound} yields
	\begin{align*}
	&\mathbf{E}\Bigg |\frac{1}{J}\sum_{j=1}^{J}\breve{g}(\Theta_j^{M})-[T_{\sigma}^M g](\breve{g})\Bigg |\\
	&\leq C_{\alpha} \|\breve{g}\|_{\infty}\card(\K)\left[e^{-C_{\beta_1} d(\K,\partial{K})}+\frac{e^{C_{\beta_2} |\mathcal{K}|_{\infty}}}{\sqrt{J}}\right] \left( 1 + \frac{2}{\delta_{m_0}^2}+\frac{4}{\log 3}\frac{1}{\delta_{m_0}^4}\sum_{j=0}^{q-2}\left(\frac{1-\delta_{m_0}^2}{1+\delta_{m_0}^2}\right)^j\right)\\
	&\leq \widetilde{C}_{\alpha} \|\breve{g}\|_{\infty}\card(\K)\left[e^{-C_{\beta_1} d(\K,\partial{K})}+\frac{e^{C_{\beta_2} |\mathcal{K}|_{\infty}}}{\sqrt{J}}\right], 
\end{align*}
which complete the proof by Theorem \ref{thm:ionides2015inference1}.
\end{proof}

\newpage
  \section{Original results for $4$ cases}
  \label{Sec:training_result_data}
In Tables \ref{table:parameter_learning_results1}--\ref{table:parameter_learning_results4}, we provide the original 
parameter learning results in terms of log-likelihood for cases $1$-$4$ respectively.
  \begin{table}[hbt!]
  \centering  
  \renewcommand{\arraystretch}{1.5}
  \begin{tabular}{|c|c|c|c|c|c|c|c|}
  \hline
  \multirow{2}{*}{Cities}  &\multirow{2}{*}{Parameters}  & \multicolumn{3}{c|}{IEnKF} & %
  \multicolumn{3}{c|}{IF2} \\
  \cline{3-8}
  & & $\le$ & $\lp$ & $\lb$ & $\le$ & $\lp$ & $\lb$ \\
  \hline
  2 & 8 & -4802  & \bf{-4382}  & \bf{-4382} & \bf{-4733} & -4385 & -4383\\
  4 & 16& -9142 & -8408 & \bf{-8278} & -9116 &-8462 & \bf{-8278}\\
  6 & 24 & -13089 & \bf{-13027} & \bf{-11942} & \bf{-13073} & -15643 & -11948\\
  8 & 32 & -16901 & -70338 & -15304 & -16843 & -71233 & \bf{-15299}\\
  10 & 40 &  -20644 & \bf{-118655} & \bf{-18596} &\bf{-20623}  &-182148 & \bf{-18596}\\
		  12 & 48 & -24277 & \bf{-202817} & \bf{-21829} & -24248 & -248841 & -21834 \\
  14 & 56 & -27761 & -311222 & \bf{-25077} & -27721 &-383853 & -25081\\
  16 & 64 & -30699 & -389299 & \bf{-27756} & \bf{-30667} &-394666 & -27758\\
  18 & 72 & -33562 & -421416 & \bf{-30329} & -33545 & -422739 & -30330\\
  20 & 80 & -36656 & -458302 & \bf{-33279} & -36645  & -536362 & \bf{-33279}\\
  \hline
  \end{tabular}
  \begin{tabular}{|c|c|c|c|c|c|c|c|}
  \hline
  \multirow{2}{*}{Cities}  &\multirow{2}{*}{Parameters}  & \multicolumn{3}{c|}{IBPF} & %
  \multicolumn{3}{c|}{$\theta$} \\
  \cline{3-8}
  & & $\le$ & $\lp$ & $\lb$ & $\le$ & $\lp$ & $\lb$ \\
  \hline
  2 & 8 & \bf{-4733} & -4383 & -4385 & -4812 &  -4393 & -4393\\
  4 & 16& \bf{-9114} & \bf{-8368} & -8283 &-9234 &  -8825 & -8290\\
  6 & 24 & -13092 & -13095 & -11945 & -13267 & -39422& -11962\\
  8 & 32 & \bf{-16842} & \bf{-63041} & -15306 & -16949& -131278& -15318\\
  10 & 40 & -20632 & -147454 & -18606 &-20699 &-225806 &-18953\\
  12 & 48 & \bf{-24226} & -206977 & -21836 & -24294& -311348& -21923\\
  14 & 56 & \bf{-27716} & \bf{-307622} & \bf{-25077} & -27789& -438979 &-25248\\
  16 & 64 & -30682 & \bf{-351668} & -27767 & -30726& -522343& -27978\\
  18 & 72 & \bf{-33540} & \bf{-382098} & -30338 & -33590&-596846 &-30423\\
  20 & 80 & \bf{-36623} & \bf{-435034} & -33293 & -36679& -624310& -33379\\
  \hline
  \end{tabular}
  \caption{Parameter learning results in terms of log-likelihood for case $1$. Three log-likelihood metrics: $\le$ representing the EnKF metric, $\lp$   representing the PF metric, and $\lb$ representing the BPF metric, were applied to the best parameters learned from IEnKF, IF$2$, and IBPF as well as the true parameter $\theta$. The highest log-likelihood values in each metric are highlighted.}
  \label{table:parameter_learning_results1}
  \end{table}
  

\begin{table}[hbt!]
\centering
\renewcommand{\arraystretch}{1.5}
\begin{tabular}{|c|c|c|c|c|c|c|c|}
\hline
\multirow{2}{*}{Cities}  &\multirow{2}{*}{Parameters}  & \multicolumn{3}{c|}{IEnKF} & %
\multicolumn{3}{c|}{IF2} \\
\cline{3-8}
& & $\le$ & $\lp$ & $\lb$ & $\le$ & $\lp$ & $\lb$ \\
\hline
2 & 10 & \bf{-4759} &  -4378 & -4379 & -4803 & \bf{-4368} & -4372 \\
4 & 20& \bf{-9127} & -8408 & -8280 & -9446 & -8355 & -8265\\
6 & 30 & \bf{-13132} & -15388 & -11946 & -14351 & -23058 & -12069 \\
8 & 40 & \bf{-16871} & -59757 & -15303 &-20638  & -127207 & -16566\\
10 & 50 & \bf{-20570} & -132791  & -18584 & -27783 &-184975 & -20324\\
12 & 60 & \bf{-24226} & -209244 & -21841 & -33111 &-252187 & -24180\\
14 & 70 & \bf{-27683} & \bf{-306092} & -25097 & -43236 & -332286  & -32121\\
\hline
\end{tabular}
\begin{tabular}{|c|c|c|c|c|c|c|c|}
\hline
\multirow{2}{*}{Cities}  &\multirow{2}{*}{Parameters}  & \multicolumn{3}{c|}{IBPF} & %
\multicolumn{3}{c|}{$\theta$} \\
\cline{3-8}
& & $\le$ & $\lp$ & $\lb$ & $\le$ & $\lp$ & $\lb$ \\
\hline
2 & 10 & -4846 & -4372 & \bf{-4367} &  -4812 &  -4393 & -4393\\
4 & 20& -9297 & \bf{-8313} & \bf{-8254} & -9234 &  -8825 & -8290\\
6 & 30 & -13348 & \bf{-12774} & \bf{-11921} & -13267 & -39422& -11962\\
8 & 40 & -17197 & \bf{-38177} & \bf{-15269} & -16949& -131278& -15318\\
10 & 50 & -20864 & \bf{-94941} & \bf{-18552} & -20699 &-225806 &-18953\\
12 & 60 & -24571 & \bf{-201864} & \bf{-21782} &  -24294& -311348& -21923\\
14 & 70 & -28082 & -315804  & \bf{-25034} & -27789& -438979 &-25248\\
\hline
\end{tabular}
\caption{Parameter learning results in terms of log-likelihood for case $2$. Three log-likelihood metrics: $\le$ representing the EnKF metric, $\lp$   representing the PF metric, and $\lb$ representing the BPF metric, were applied to the best parameters learned from IEnKF, IF$2$, and IBPF as well as the true parameter $\theta$. The highest log-likelihood values in each metric are highlighted.}
\label{table:parameter_learning_results2}
\end{table}

\begin{table}[hbt!]
\centering
\renewcommand{\arraystretch}{1.5}
\begin{tabular}{|c|c|c|c|c|c|c|c|}
\hline
\multirow{2}{*}{Cities}  &\multirow{2}{*}{Parameters}  & \multicolumn{3}{c|}{IEnKF} & %
\multicolumn{3}{c|}{IF2} \\
\cline{3-8}
& & $\le$ & $\lp$ & $\lb$ & $\le$ & $\lp$ & $\lb$ \\
\hline
2 & 10 & -97124 & -68887  & -70099 & -5133  & \bf{-4368} & -4396\\
4 & 20& -176541 & -128562  & -123847 & \bf{-9540} & \bf{-8314} & -8280\\
6 & 30 & -232078 & -188457  & -177508 &  -13884 &-13044 & -12016\\
8 & 40 & -310994 & -246252  & -228446 & -19075 & -67133 & -15705\\
10 & 50 & -368104 & -305596 & -279035 & -47616 & -127703 & -42846\\
12 & 60 & -435576 & -369515 & -337833 & -57865 & -197042 & -59511\\
\hline
\end{tabular}
\begin{tabular}{|c|c|c|c|c|c|c|c|}
\hline
\multirow{2}{*}{Cities}  &\multirow{2}{*}{Parameters}  & \multicolumn{3}{c|}{IBPF} & %
\multicolumn{3}{c|}{$\theta$} \\
\cline{3-8}
& & $\le$ & $\lp$ & $\lb$ & $\le$ & $\lp$ & $\lb$ \\
\hline
2 & 10 & \bf{-5054} & -4384 & \bf{-4368} &  -4812 &  -4393 & -4393\\
4 & 20& -9633 & -8348 & \bf{-8251} & -9234 &  -8825 & -8290\\
6 & 30 & \bf{-13738} & \bf{-12506} & \bf{-11918} & -13267 & -39422& -11962\\
8 & 40 & \bf{-17600} & \bf{-20904} & \bf{-15276} & -16949& -131278& -15318\\
10 & 50 & \bf{-21295} & \bf{-58647} & \bf{-18570} & -20699 &-225806 &-18953\\
12 & 60 & \bf{-24931} & \bf{-128163}  & \bf{-21814} &  -24294& -311348& -21923\\
\hline
\end{tabular}
\caption{Parameter learning results in terms of log-likelihood for case $3$. Three log-likelihood metrics: $\le$ representing the EnKF metric, $\lp$   representing the PF metric, and $\lb$ representing the BPF metric, were applied to the best parameters learned from IEnKF, IF$2$, and IBPF as well as the true parameter $\theta$. The highest log-likelihood values in each metric are highlighted.}
\label{table:parameter_learning_results3}
\end{table}

\begin{table}[hbt!]
	\centering
	\renewcommand{\arraystretch}{1.5}
	\begin{tabular}{|c|c|c|c|c|c|c|c|}
		\hline
		\multirow{2}{*}{Cities}  &\multirow{2}{*}{Parameters}  & \multicolumn{3}{c|}{IEnKF} & %
		\multicolumn{3}{c|}{IF2} \\
		\cline{3-8}
		& & $\le$ & $\lp$ & $\lb$ & $\le$ & $\lp$ & $\lb$ \\
		\hline
		2 & 14 & Failed\; & Failed\; & Failed\; & \bf{-4904} &\bf{-4358} & -4361\\
		4 & 28& Failed\; & Failed\; & Failed\; & -9979 & -8524 & -8416\\
		6 & 42 & Failed\; & Failed\; & Failed\; & -19795 &-56072 & -13550\\
		8 & 56 & Failed\; & Failed\; & Failed\; & -32960 & -111176 & -19071\\
		10 & 70 & Failed\; & Failed\; & Failed\; &-34320  & -193060 & -23217\\
		\hline
	\end{tabular}
	\begin{tabular}{|c|c|c|c|c|c|c|c|}
		\hline
		\multirow{2}{*}{Cities}  &\multirow{2}{*}{Parameters}  & \multicolumn{3}{c|}{IBPF} & %
		\multicolumn{3}{c|}{$\theta$} \\
		\cline{3-8}
		& & $\le$ & $\lp$ & $\lb$ & $\le$ & $\lp$ & $\lb$ \\
		\hline
		2 & 14 & -4928 & -4374 & \bf{-4358} &  -4812 &  -4393 & -4393\\
		4 & 28& \bf{-9169} & \bf{-8267} & \bf{-8231} & -9234 &  -8825 & -8290\\
		6 & 42 & \bf{-13296} & \bf{-12340} & \bf{-11893} & -13267 & -39422& -11962\\
		8 & 56 & \bf{-17065} & \bf{-17407} & \bf{-15248} & -16949& -131278& -15318\\
		10 & 70 & \bf{-20788} & \bf{-24892} & \bf{-18555} & -20699 &-225806 &-18953\\
		\hline
	\end{tabular}
	\caption{Parameter learning results in terms of log-likelihood for case $4$. Three log-likelihood metrics: $\le$ representing the EnKF metric, $\lp$   representing the PF metric, and $\lb$ representing the BPF metric, were applied to the best parameters learned from IEnKF, IF$2$, and IBPF as well as the true parameter $\theta$. The highest log-likelihood values in each metric are highlighted.}
\label{table:parameter_learning_results4}
\end{table}

%

\acks{This research project was partially supported by NSF grant DMS-$1761603$. Ning Ning's research was also partially supported by the Seed Fund Grant Award at Texas A\&M University.  The authors would like to thank three anonymous reviewers and the Action Editor for their very constructive comments and efforts on this lengthy work, which greatly improved the quality of this paper.}


\vskip 0.2in
\bibliography{sample}

\end{document}